\documentclass[11pt]{article}
\usepackage[utf8]{inputenc}

\usepackage{mystyle}

\newcommand{\bbr}[1]{\left[ {#1} \right]}
\newcommand{\bpa}[1]{\left( {#1} \right)}

\newcommand{\bag}[1]{\left\langle {#1} \right\rangle}

 \newcommand{\cm}{,\;}

\begin{document}

\title {\huge Variational Transport: A Convergent Particle-Based Algorithm for Distributional Optimization}

\author
{Zhuoran Yang\thanks{Princeton University, email:\texttt{zy6@princeton.edu}.} \qquad Yufeng Zhang\thanks{Northwestern  University, email: \texttt{yufengzhang2023@u.northwestern.edu}, \texttt{zhaoran.wang@northwestern.edu}.} \qquad  Yongxin Chen\thanks{Georgia Institute of Technology, email:\texttt{yongchen@gatech.edu}.} \qquad Zhaoran Wang$^\dagger$\thanks{Part of this work was done while the Zhuoran Yang and Zhaoran Wang was visiting the Simons Institute for the Theory of Computing.}  
}

\date{\today}

\maketitle

\begin{abstract}
We consider the optimization problem of minimizing a functional defined over a family of probability distributions, where the objective  functional is assumed to possess a variational form. 
Such a distributional optimization problem arises widely in machine learning and statistics,  with Monte-Carlo sampling, variational inference, policy optimization, and generative adversarial network as examples. 
For this problem, we propose a novel particle-based algorithm, dubbed as  variational transport, which
approximately performs Wasserstein gradient descent over the manifold of  probability distributions via iteratively pushing a set of particles. Specifically, we
prove that 
moving along the geodesic in the direction of 
functional  gradient  with respect to the second-order Wasserstein distance   is equivalent to applying a  pushforward mapping to a probability distribution, which can be approximated accurately by pushing a set of particles. 
Specifically, in each iteration of variational transport, we first solve the variational problem associated with the objective functional using the particles, whose solution yields the Wasserstein gradient direction.  Then we update the current distribution by pushing  each particle  along the direction specified by such a solution. 
By characterizing both the statistical error incurred in  estimating the Wasserstein  gradient and the progress of the  optimization algorithm,
we prove that  when the objective functional satisfies a  functional version of the  Polyak-\L{}ojasiewicz (PL) \citep{polyak1963gradient} and smoothness conditions,    variational  transport   converges linearly to the global minimum of the objective functional up to a certain statistical error, which decays to zero sublinearly as the number of particles goes to infinity.
\end{abstract}



\section{Introduction}\label{eq:intro}

We study 
a class of optimization problems over nonparametric probability distributions, dubbed as distributional optimization, where the goal is to minimize a  functional of probability distribution. 
Specifically, 
a distributional optimization problem is given by $\min_{p \in \cP_2(\cX)} F(p)$, where $\cP_2(\cX)$ is the set of  all probability densities  supported  on $\cX \subseteq \RR^d$ with finite second-order moments, and $F$ is the objective functional of interest. 
Many machine learning problems fall into such a category. 
For instance, in  Bayesian inference \citep{gelman2013bayesian}, the probability  distribution describes the belief based on the observations and the functional of interest is the Kullback-Leibler (KL) divergence. 
In 
distributionally robust optimization (DRO) \citep{rahimian2019distributionally}, the inner problem optimizes a linear functional  to find   the worst-case data distribution. 
Besides, in unsupervised learning models such as generative adversarial network \citep{goodfellow2014generative},
the objective functional captures the proximity between the generative model and the target distribution. 
Whereas   the policy optimization problem in reinforcement learning \citep{sutton2011reinforcement} seeks  a distribution over the state-action space that achieves the   highest expected reward.
All of these  
 instances have been intensively studied separately  with   algorithms proposed in parallel.

 Distributional optimization    belongs to the general family of infinite-dimensional optimization problems \citep{EkeTur83,BurJey05,Fat99} and hence inherits many of the challenges therein. 
 In particular, without further assumptions, it requires an infinite number of parameters to fully represent the optimization variable, namely the distribution $p$. 
 As a result, despite a broad range of applications of distributional optimization, 
 such  infinite-dimensionality makes it significantly more challenging to solve than finite-dimensional optimization. 
 One straightforward approach is to parameterize $p$ by a finite-dimensional parameter $\theta$ as $p_{\theta}$ and 
reduce the problem to solving   $\min_{\theta \in \Theta} F(p_{\theta})$, where $\Theta$ is the parameter space. 
Such a method is a common practice in 
variational inference \citep{gershman2012tutorial, kingma2019introduction}, policy optimization \citep{sutton2000policy, schulman2015trust,haarnoja2018soft}, and GAN \citep{goodfellow2014generative,arjovsky2017wasserstein}, and has achieved tremendous empirical successes. 
However, 
this approach suffers from
the following three  drawbacks. 
First,  
the validity of this approach hinges on the fact that  the parameterized model $\{ p_{\theta} \colon \theta \in \Theta \}$ has sufficient representation power such that it well approximates the global minimizer of $F$ over $\cP_2(\cX)$. 
Otherwise,  the finite-dimensional parameterization  introduces a considerable  approximation bias  such that the global minimizer of $\min_{\theta \in \Theta}  F(p_{\theta})$ is suboptimal for $F$.
Second, the new objective function $F(p_{\theta})$ is 
 generally nonconvex in $\theta$, 
which oftentimes 
results in  undesirable empirical performances and a lack of theoretical guarantees. 
Finally, in many applications of distributional optimization, 
instead of finding the optimal probability  distribution itself,
the goal is to draw  samples from 
it, which is,
e.g., the case for Bayesian inference with the goal of sampling from the posterior. 
In this case, 
there is a tension between optimization and sampling. 
Specifically, 
finite-dimensional  
parametrization  of probability distributions 
are commonly implemented via  either the Boltzmann distribution or the reparameterization trick \citep{kingma2013auto}.
In the former, 
$p_{\theta}$ is written as  $p_{\theta} (\cdot ) =   \exp[- f_{\theta} (\cdot )]/ Z_{\theta}$ for some energy function $f_{\theta}$, where  $Z_{\theta} = \int_{x\in \cX } \exp[- f_{\theta} (x )]~\ud x$ is the normalization factor.  
Although the density $p_{\theta}$ has a closed form which makes the optimization problem convenient, due to the integral in $Z_{\theta}$, it is often challenging to draw samples from $p_{\theta}$. 
In contrast, when using the reparameterization trick, 
sampling from $p_{\theta}$ is conveniently achieved  by sending a random noise  to a 
differentiable mapping with parameter $\theta$. Nevertheless, since the density $p_{\theta}$ is implicitly defined, optimizing $F(p_{\theta})$ can be challenging, especially when the closed form of $p_{\theta}$ is required. 
 
To alleviate these obstacles, we propose to directly solve the infinite-dimensional distributional optimization problem, utilizing the   additional constraint  that the decision variable $p$ is a probability distribution. 
In particular, 
utilizing the 
the optimal transport (OT) framework \citep{otto2001geometry,villani2003topics,villani2008optimal},
it is shown that $\cP_{2} (\cX)$ is a geodesic space where the second-order Wasserstein distance 
plays the role 
of the geodesic distance.
From such a  perspective, in many applications of distributional optimization, the objective functional $F$ is a geodesically convex functional on $\cP_2(\cX)$. 
Thus, we propose to directly optimize the functional $F$ on $\cP_2(\cX)$ by function gradient descent  with respect to the geodesic distance \citep{zhang2016first}, which constructs a sequence of iterates $\{ p_t\}_{t\geq 1}$ in $\cP_2(\cX)$ satisfying 
\#\label{eq:riemann_grad_descent}
p_{t+1} \leftarrow \expm_{p_t} \bigl [  - \alpha_t \cdot \grad F(p_t) \bigr ], 
\#
where $\alpha_t$ is the stepsize, $\grad F $ is the functional gradient with respect to the Wasserstein distance, which is also known as the Wasserstein gradient, and $\expm_p $ denotes the  exponential mapping on $\cP_2(\cX)$, which specifies how to move along any given direction on $\cP_2(\cX)$.  

To implement the update in \eqref{eq:riemann_grad_descent}, we need to (i) obtain the Wasserstein gradient $\grad F(p_t)$ and (ii) calculate the exponential mapping.
Moreover, to make it useful for applications,  it is   desirable to (iii) be able to  draw samples from each $p_t$ and  (iv)  allow  a sufficiently  general class of objective functionals. 
To achieve these goals, 
 we focus on a class of functionals that admit
a variational form
 $F(p) = \sup_{f\in \cF} \{ \EE_p [f(X)] - F^*(f)\} $,   
where $\cF$ is a function class on $\cX$ and $F^*\colon \cF\rightarrow \RR $ is a strongly convex functional.  
Such   a variational form 
generalizes beyond convex functionals,
where $F^*$ corresponds to the convex conjugate when $F$ is convex.      
A key feature of the variational functional objective is that it enables us  implement the Wasserstein  gradient  update in \eqref{eq:riemann_grad_descent} using data sampled from $p_t$. 
In particular, the Wasserstein gradient $\grad F(p_t)$ can be calculated from the solution $f_{p_t}^*$ to  the maximization problem for obtaining $F(p_t)$, which can be estimated using data.
Meanwhile, 
as we will show in \S\ref{sec:algo}, the exponential mapping in \eqref{eq:riemann_grad_descent} is equivalent to a  pushforward mapping induced by $   f_{p_t}^*$, which can be efficiently obtained once $f_{p_t}^*$ is estimated.  
Moreover, to efficiently sample from $p_t$, we approximate it using the 
empirical measure of a set of particles,
which leads to our proposed  algorithm,
which is named as variational transport as it utilizes the  optimal transport framework and a variational representation of the objective. 
Specifically, 
variational transport maintains a set of particles and outputs  their empirical measure 
as the solution  to the distributional optimization problem. 
In each iteration,   variational transport  approximates  the update in \eqref{eq:riemann_grad_descent} by first solving  the dual maximization problem associated with the variational form of the objective and then using the  obtained solution to specify  a direction  to push each particle. 
The variational transport algorithm can be viewed as a  forward discretization of the Wasserstein gradient flow \citep{santambrogio2017euclidean} 
with particle approximation and gradient estimation.

Compared with existing methods, variational transport features a unified algorithmic framework that enjoys the following advantages. First, by considering  functionals with a variational form, the algorithm can be applied to a broad class of objective functionals.
Second, the functional optimization problem associated  with the variational representation  of $F$ can be  solved by any supervised learning methods such as deep learning \citep{lecun2015deep, goodfellow2016deep, fan2019selective} and kernel methods 
\citep{friedman2001elements, shawe2004kernel}, which offers additional flexibility in algorithm design. 
Finally, by considering nonparametric probability distribution, variational transport does not suffer from the approximation bias incurred by finite-dimensional parameterization of the probability distribution, 
and the particle approximation enables convenient sampling from the obtained probability measure.

To showcase these advantages, we consider an instantiation of variational transport where the objective functional $F$ satisfies the Polyak-\L{}ojasiewicz (PL) condition \citep{polyak1963gradient} with respect to the Wasserstein distance and the variational problem associated with $F$ is solved via kernel methods. 
In this case, we prove that variational transport generates a sequence of probability distributions that converges linearly to a  global minimizer of $F$ up to some statistical error. 
Here the statistical error is incurred in estimating the Wasserstein gradient by solving the dual maximization problem using functions in a reproducing kernel Hilbert space (RKHS) with finite data, which   converges sublinearly to zero as the number of particles goes to infinity. 
Therefore, in this scenario, variational transport provably  enjoys both computational efficiency and global optimality. 
 
\vspace{5pt}

{\bf \noindent Our Contribution.} Our contribution is two fold. First, utilizing the optimal transport framework and the variational form of the objective functional, we propose a novel variational transport algorithmic framework  for solving the distributional optimization problem via particle approximation. 
In each iteration, variational transport first solves the variational problem associated with the objective to obtain an estimator of the Wasserstein gradient and then approximately implements Wasserstein gradient descent by pushing the particles. 
Second, when the Wasserstein gradient is approximated using RKHS functions and the objective functional satisfies the  PL condition, we prove that the sequence of probability distributions constructed by  variational transport   converges linearly to the  global minimum of the objective functional,  up to certain statistical error that converges to zero as the number of particles goes to infinity. 
To the best of our knowledge, we seem to propose the first  particle-based algorithm for general distributional optimization problems with both  global convergence and   global optimality guarantees.

\vspace{5pt}
{\bf \noindent Related Works.} 
There is a large body of literature on manifold optimization where the goal is to minimize a functional  defined on  a Riemannian manifold. 
See, e.g.,  \cite{udriste1994convex,ferreira2002proximal,absil2009optimization, ring2012optimization, bonnabel2013stochastic,  zhang2016first,zhang2016riemannian, liu2017accelerated, agarwal2018adaptive, zhang2018r,tripuraneni2018averaging,boumal2018global, becigneul2018riemannian, zhang2018estimate, sato2019riemannian, zhou2019faster, weber2019projection} and the references therein. 
Also see recent reviews \citep {ferreira2020first,hosseini2020recent} 
for  summary.
These works all focus on finite-dimensional manifolds  where each point in the feasible set has a neighborhood that is homeomorphic to the Euclidean space. 
In contrast, the feasible set of distributional optimization is the Wasserstein space on a subset $\cX$ of $\RR^d$, which is an infinite-dimensional  manifold. 
As a result, unlike 
 finite-dimensional manifold optimization, on Wasserstein space, 
 both 
 the   Riemannian gradient of the objective functional and the exponential mapping  cannot be easily obtained, which  makes it infeasible to directly apply   manifold optimization methods.

Moreover, our work is closely related to the vast literature on Bayesian inference. 
Our work is particularly related to the line of research on 
gradient-based MCMC, which 
is a family of   particle-based sampling algorithms that approximate diffusion process whose stationary distribution is the target distribution.  
The finite-time convergence of gradient-based MCMC has been extensively studied. 
 See, e.g.,   \cite {welling2011bayesian, chen2014stochastic, ma2015complete, chen2015convergence,
 dubey2016variance, vollmer2016exploration,chen2016stochastic, dalalyan2017further, chen2017convergence, raginsky2017non, brosse2018promises, xu2018global, cheng2018convergence, chatterji2018theory,wibisono2018sampling,bernton2018langevin,  dalalyan2019user, baker2019control,ma2019there,  ma2019sampling, mou2019improved,vempala2019rapid, salim2019stochastic, durmus2019high, wibisono2019proximal} and the references therein. 
 Among these works, 
 our work is more related to 
 \cite{wibisono2018sampling,bernton2018langevin,  ma2019there,ma2019sampling,cheng2018convergence,vempala2019rapid, wibisono2019proximal, salim2019stochastic}, which establish the finite-time convergence of gradient-based MCMC methods in terms of the KL-divergence.
 These works utilize the property that the diffusion process associated with   Langevin dynamics in $\cX$ corresponds to the  Wasserstein gradient flow  of the KL-divergence in $\cP_2(\cX)$
 \citep{jordan1998variational}, and the  methods proposed in these works  apply various time-discretization techniques.
Besides, \cite{frogner2020approximate} recently 
applies the Wasserstein flow of KL divergence for the inference of diffusion processes.

In addition to gradient-based MCMC,  variational transport also shares similarity with Stein variational gradient descent (SVGD)  \citep{liu2016stein}, which  is a more recent particle-based algorithm for Bayesian inference.
 Variants of SVGD have been subsequently proposed. See, e.g., 
 \cite{detommaso2018stein,han2018stein, chen2018unified,liu2019understanding,gong2019quantile, wang2019stein, zhang2020stochastic, ye2020stein}
  and the references therein. 
 Departing from MCMC where independent stochastic particles are used, it leverages interacting deterministic particles to approximate the probability measure of interest. In the mean-field limit where the number of particles go to infinity, it can be viewed as the  gradient flow of the KL-divergence with respect to a modified Wasserstein metric \citep{liu2017stein}.
 Utilizing the gradient flow interpretation, 
 the convergence of SVGD has been established in the mean-field limit 
\citep{liu2017stein, duncan2019geometry, korba2020non,chewi2020svgd}. 
Meanwhile, it is worth noting that 
\cite{chen2018unified, liu2019understanding} build the connection between MCMC and SVGD 
through the lens of Wasserstein gradient flow of KL divergence.
Comparing to MCMC and SVGD, 
variational transport approximates the Wasserstein gradient descent using particles, which corresponds to a forward discretization of the Wasserstein gradient flow. 
Moreover, utilizing the variational representation of the objective functional, 
variational transport can be applied to functionals beyond the family of $f$-divergences. 

Furthermore, 
our work is also related to 
\cite{arbel2019maximum},
which studies the convergence of Wasserstein gradient flow of the maximum mean discrepancy (MMD) and its discretization.  
Since MMD   is an integral probability metric and thus admits variational representation,
variational transport can also be utilized to minimize MMD. 
Moreover, when restricted to  MMD minimization, our algorithm is similar to the sample-based approximation method proposed in \cite{arbel2019maximum}. 
Besides, \cite {futami2019bayesian} proposes a particle-based algorithm  that minimizes MMD using the Frank-Wolfe method.  
Another related work on 
Bayesian inference is   \cite{dai2016provable},
which proposes a  
distributional optimization algorithm named particle mirror descent (PMD). 
Specifically, PMD performs 
infinite-dimensional mirror descent on the probability density function   using 
functional gradients of the KL-divergence.
In contrast to our work, their functional gradient is with respect to the functional $\ell_1-\ell_{\infty}$ structure whereas variational transport utilize the Wasserstein geometry. 
Moreover, for  computational tractability,  PMD  maintains a set of particles and directly estimates the density function via kernel density estimation in each iteration.  
In comparison, variational transport
does not requires the density functions  of the iterates.
Instead, we 
use the empirical distribution of the particles to approximate the probability measure  and the iterates are updated via pushing the particles in directions specified the solution to    a variational problem.

Finally,  there exists a body of literature on general  
distributional   optimization  problems.   
\cite{gaivoronski1986linearization, molchanov2001variational,molchanov2002steepest, molchanov2004optimisation, pieper2019linear}
study the Frank-Wolfe and steepest descent algorithms on the space of distribution measures. 
The methods proposed in these works utilize functional gradients that might be inefficient to compute in machine learning problems.
Besides, motivated by the idea of representing probability measures using their moments \citep{lasserre2010moments},
for distributional optimization with a linear objective functional that is given by a polynomial function  on $\cX$,  
 \cite{lasserre2001global, lasserre2009moments, lasserre2008semidefinite, henrion2009approximate, jasour2018moment} 
propose convex relaxation methods through the sum-of-squares techniques \citep{parrilo2000structured,lasserre2010moments}.  
These approaches require solving large-scale semidefinite programming problems to the global optima and thus are computationally challenging. 
A more related work is \cite {chu2019probability}, which casts  various  machine learning algorithms  as functional gradient methods  for distributional optimization, where the gradient is not necessarily with respect to the Wasserstein distance. 
In comparison, 
we study  a similar distributional optimization framework 
that comprises these interesting machine learning problems. 
Moreover, 
utilizing the  Wasserstein gradient, 
we propose a novel algorithm that provably finds the global optimum with computational efficiency, which complements the results in \cite{chu2019probability}.

\vspace{5pt}

{\noindent \bf Notation.}  
Throughout this paper, we let $\RR^d$, $\ZZ^d$, $\TT^d$ denote the $d$-dimensional Euclidean space, integer lattice, and torus, respectively. 
For  $\cX$ being  a compact subset of $\RR^d$, let $\cP(\cX)$ and $\cP_2(\cX)$ denote the set of probability distributions over $\cX$ and the set of probability density functions on $\cX$ with finite second-order moments, respectively. 
Here the density functions are with respect to the Lebesgue measure on $\cX$. 
For any $p \in \cP_2(\cX)$, we identify the density function with the probability measure $p(x)~\ud x$ that is  induced by $p$.
For any $p, q\in \cP(\cX)$, let 
$\textrm{KL}(p , q)   $ denote the     Kullback-Leibler divergence	 between $p$ and $q$. 
For any mapping $T \colon \cX \rightarrow \cX$ and any $\mu \in \cP(\cX)$,  
let  $T _{\sharp} \mu $ denote the pushforward measure induced by $T$. 
Meanwhile, let 
 $\trace$, $\dvg$, and $\la \cdot , \cdot \ra$ denote the trace operator, divergence operator, and inner product on $\RR^d$, respectively. 
 For $f\colon \cX\rightarrow \RR$ being a differentiable function, we let $\nabla f$ and $\nabla^2 f$  denote the gradient and Hessian of $f$, respectively. 
 Furthermore, 
 let $\cF$ denote a class of functions on $\cX$ and let 
 $L \colon \cF\rightarrow \RR$ be a functional on $\cF$. 
 We use  
   $\mathcal{D}L, \mathcal{D}^2L$, and $ \mathcal{D}^3L $ to denote the first, second, and  third order   Fr\'echet derivatives of $L$  respectively. 
 Besides, throughout this work, we  let $\cH $ denote an RKHS defined on $\cX$,  and let $\la \cdot , \cdot \ra_{\cH}$ and $ \| \cdot \|_{\cH}$ denote the inner product on $\cH$ and the RKHS norm, respectively. 
Finally, for any Riemannian manifold $\cM$ with  a Riemannian metric $g$, we let $\cT_p \cM$ denote the tangent space at   point $p \in \cM$ and let   $\la \cdot , \cdot \ra_p$ denote the inner product on $\cT_p \cM$ induced by $g$.
For any functional  $F \colon \cM\rightarrow \RR$, we let $\grad F $ denote the functional gradient of $F$ with respect to the Riemannian metric $g$. 


\section{Background}
To study   optimization problems on the space of probability measures, we first introduce the  background knowledge of the Riemannian manifold and the Wasserstein space. In addition, to analyze the statistical estimation problem that arises  in estimating the Wasserstein gradient, we introduce the  reproducing kernel Hilbert space.
 
\subsection{Metric Space and Riemannian Manifold} \label{sec:riemann}
A metric space $(\cX, \| \cdot \|)$ consists of a set $\cX$ and a distance function $\|\cdot \|$ \citep{burago2001course}. Given a continuous curve $\gamma: [0, 1] \rightarrow \cX$, the length of $\gamma$ is defined as $L(\gamma) = \sup \sum_{i=1}^n \| \gamma(t_{i-1}) - \gamma(t_i) \|$, where the supremum is taken over $n \geq 1$ and  all partitions $0=t_0 < t_1  < \ldots < t_n = 1$ of $[0, 1]$. 
It then holds for any curve $\gamma$ that $L(\gamma) \geq \| \gamma(0) - \gamma(1)\|$. If there exists a constant $v \geq 0$ such that $\| \gamma(t_1) - \gamma(t_2) \| = v \cdot |t _1 - t_2|$ for any $t_1, t_2 \in [0,1]$, the curve $\gamma$ is called a geodesic.
 In this case, for any $0\leq t_1< t_2 \leq 1$, the length of $\gamma $ restricted to $[t_1, t_2]$ is equal to $\| \gamma(t_1)- \gamma(t_2) \| $.  Thus, a geodesic is   locally a distance minimizer everywhere. 
Moreover, $(\cX, \|\cdot \|)$ is called a geodesic space if  any two points $x, y\in \cX$ are connected by a geodesic $\gamma$ such that $\gamma(0) = x$ and $\gamma(1) = y$. 

A $d$-dimensional differential manifold $\cM$ is a topological space that is locally homeomorphic to the Euclidean space $\RR^d$ with a globally defined differential structure \citep{chern1999lectures}. 
A tangent vector  at $x\in \cM$ is  an equivalence class of differentiable curves going through $x$ with a prescribed velocity vector at $x$.
 The tangent space at $x$,  denoted by $\cT_x \cM$, consists of all tangent vectors at $x$.
To compare two tangent vectors in a meaningful way, we 
consider a Riemannian manifold   $(\cM, g)$, which is  a  smooth manifold equipped with an inner product $g_x$ on the tangent space $\cT_x\cM$ for any $x \in \cM$ \citep{carmo1992riemannian, petersen2006riemannian}. The inner product of any two tangent vectors $u_1, u_2 \in \cT_x\cM$ is defined as  $\la u_1, u_2 \ra_x = g_x(u_1, u_2)$.
Besides, we  call $g$ the Riemannian metric. On  a Riemannian manifold, the length of a smooth curve $\gamma \colon [0,1]\rightarrow \cM$ is defined as
\#\label{eq:length}
L(\gamma) = \int_0^1\sqrt{\la \gamma ' (t), \gamma'(t) \ra _{\gamma(t)}} ~\ud t.
\#
The distance between any two point $x, y \in \cM$ is defined as $\| x - y\| = \inf_{\gamma} L(\gamma)$, where the infimum  is taken over all smooth curves such that $\gamma (0) = x$ and $\gamma(1) = y$.
 Equipped with such a distance, $(\cM, g)$ is  a metric space and is  further a  
  geodesic space if it is compete and connected 
  \citep[Hopf–Rinow theorem]{burago2001course}. Hereafter, we  assume $(\cM, g)$  is  a geodesic space.

Another  important concept   is the exponential map, which specifies how to move a point along a tangent vector. Specifically,   
the 
exponential mapping  at $x$, denoted by  $\expm_x \colon \cT_x \cM  \rightarrow \cM$, sends any tangent vector  $u \in \cT_x \cM$ to $y = \gamma_u(1) \in \cM$, where    $\gamma_u: [0, 1] \to \cM$ is the unique geodesic determined by $ \gamma_u(0) = x$ and $\gamma_u' (0) = u$. 
Moreover, since $\gamma_u$ is the unique geodesic connecting $x$ and $y$, the exponential mapping is invertible and we have $u = \expm^{-1} _x(y)$. The distance between $x$ and $y$ satisfies $\|x -y \| = [ \la  \expm^{-1} _x(y), \expm^{-1} _x(y) \ra _x ]^{1/2}$, which is also called the geodesic distance.
Note that  tangent vectors of two different points lie in distinct tangent spaces. Thus, to compare these tangent vectors,  for any two points $x, y \in \cM$, 
we define the   parallel transport $ \Gamma _x^y \colon \cT_x\cM \rightarrow \cT_y \cM$, which
specifies   how  a tangent vector  of $x$ is  uniquely identified with an element in $\cT_{y}\cM$.  Moreover, parallel transport  perseveres the  inner product in the sense that $\la u, v \ra _x = \la  \Gamma _x^y u, \Gamma_x^y v \ra _{y} $ for any $u, v \in \cT_x \cM$.

On the Riemannian manifold $(\cM, g)$, for any $x \in \cM$ and any $v \in \cT_x \cM$, the directional derivative of a functional  $f\colon 
  \cM \to \RR$ is defined as 
  $\nabla _v f(x) = \frac{\ud}  {\ud t} f  [ \gamma (t) ] \big  \vert_{t= 0}
 , $ where $\gamma $ is a smooth curve satisfying  $\gamma (0) = x$ and $\gamma '(0) = v$. If there exists $u \in \cT_x \cM$ such that $\nabla _v f(x) = \la u, v \ra _x$ for any $v \in \cT_x \cM$, the functional $f$ is 
 differentiable at $x$ and the tangent vector $u$ is called the gradient of $f$ at $x$, denoted by $\grad f(x)$. 
 With the notions of the gradient and the geodesic distance $\| \cdot \|$, we are able to define convex functionals on $\cM$. Specifically,  a functional $f$ is called geodesically $\mu$-strongly convex if 
 \#\label{eq:define_sc}
 f(y) \geq f(x) + \bigl \la \grad f(x), \expm^{-1}_x(y)  \bigr \ra_x+ \mu /2 \cdot  \| x - y\|^2 ,
 \#
 and we say $f$ is a geodesically convex functional  when $\mu = 0$.
 
\subsection{Wasserstein Space}  \label{sec:wasserstein_space}

Let $\mathcal{P}(\cX)$ denote the set of all Borel probability measures on the   measurable space $(\cX, \mathcal{B}(\cX))$, where $\cX$ is a $d$-dimensional Riemannian manifold without boundary and $\cB(\cX)$ is the Borel $\sigma$-algebra on $\cX$. 
For instance, $\cX$ can be  a convex compact region in $\RR^d$  with zero flux or periodic boundary conditions.
For any $\mu, \nu \in \cP(\cX)$, we let $\Pi(\mu, \nu)$ denote the set of all  couplings of $\mu$ and $\nu$, i.e., $\Pi(\mu, \nu)$ consists of  all probability measures on $\cX \times \cX$ whose two marginal distributions are equal to $\mu$ and $\nu$, respectively. 
Kantorovich's formulation of the optimal transport problem aims to find a coupling 
$\pi$ of  
$\mu$ and $\nu$ such that 
$
   \int_{\cX\times \cX }  \| x - y \| ^2~ \ud   \pi(x,y)  
$  is minimized,
where $\| \cdot    \|$ is the geodesic distance  on $\cX$. 
It is shown that  
there exists a unique minimizer which is called the optimal transport plan 
\citep{villani2008optimal}.  Moreover, the minimal value 
\#\label{eq:w2_def1}
W_2(\mu, \nu)  =  \biggl [ \inf_{\pi \in \Pi(\mu, \nu)}  \int_{\cX\times \cX }  \| x - y \| ^2~ \ud   \pi(x,y)  \biggr ]^{1/2 } 
\#
defines a distance on $\cP(\cX)$, which is known as the  second-order Wasserstein distance. 

To study the optimization problem on the space of probability measures, in 
the following, 
we focus  on a 
 subset of probability distributions $\cP_2(\cX)$, which is defined as 
  \#\label{eq:prob_manifold}
 \cP_2(\cX) = \Big\{p \colon \cX \rightarrow [0, \infty) \colon \int_{\cX} \|x - x_0 \|^2 \cdot p(x) ~\ud x < \infty, \int_{\cX} p(x) ~\ud x = 1 \Big\},
 \#
 where $x_0$ is any fixed point in $\cX$ and we let $\ud x$ denote the Lebesgue measure on $\cX$. 
 Then $\cP_2(\cX)$ consists of probability density functions on $\cX$ with finite second-order moments. 
  It is known that $(\cP_2(\cX), W_2)$  is an infinite-dimensional geodesic space \citep{villani2008optimal}, which is called the Wasserstein space. 
  Specifically, any  curve on $\cP_2(\cX)$ can be written as $\rho \colon [0,1] \times\cX\rightarrow  [0, \infty )$, where $\rho(t,  \cdot)  \in \cP_2(\cX)$ for all $t\in [0,1]$.  A tangent vector at $p \in \cP_2(\cX)$ can be written as $\partial \rho  / \partial _t$ for some curve $\rho$  such that $\rho(0, \cdot ) = p(\cdot)$.   Besides, under certain regularity conditions, the elliptical equation  $- \dvg( \rho \cdot \nabla u) = \partial \rho / \partial t$
  admits a unique solution $u \colon \cX \rightarrow \RR$ \citep{denny2010unique, gilbarg2015elliptic}, where   $\dvg $ is the divergence operator on $\cX$. 
  Notice that $\cX$ is a $d$-dimentional Riemannian manifold. 
Here we let $\nabla u$ denote  the Riemannian gradient 
of function $u$, $\grad u$  to simplify the notation,  which is introduced in \S\ref{sec:riemann}.  
For any $x \in \cX$, $\nabla u(x) \in \cT_x \cX$, which can be identified as a vector in $\RR^d$.  
  Thus, the tangent vector $\partial \rho / \partial t$ is uniquely identified  with the vector-valued mapping $  \nabla u  $.
 Moreover, such a construction endows the infinite-dimensional manifold $\cP_2(\cX)$ with a Riemannian metric  \citep{otto2001geometry, villani2003topics}. Specifically, for any $s_1, s_2 \in \cT_p \cP_2(\cX)$, let $u _1, u_2 \colon \cX \rightarrow\RR$ be the solutions to elliptic equations $- \dvg( p \cdot \nabla u_1) = s_1$ and $- \dvg( p \cdot \nabla u_2) = s_2$, respectively.  The Riemannian metric, which is an inner product structure,  between $s_1$ and $s_2$ is defined  as 
 \#\label{eq:riemann_metric}
 \la s_1, s_2\ra _p =  \int _{\cX} \bigl \la  \nabla u_1 (x) , \nabla u_2(x) \bigr\ra_x \cdot p(x) ~\ud x ,
 \#
 where $\la  \nabla u_1 (x) , \nabla u_2(x) \ra_x$ is the  standard inner product in $\cT_x \cX$, 
 which is  homeomorphic to  the inner product of $\RR^d$. 
 Equipped with such a Riemannian metric,  $(\cP_2(\cX), W_2)$  is a geodesic space and the Wasserstein distance defined in \eqref{eq:w2_def1} can be written~as 
 \#\label{eq:w2_def2} 
 W_2 (\mu, \nu) = \biggl [ \inf _{\gamma \colon [0,1] \rightarrow \cP_2(\cX)  }  \int_0^1 \bigl \la \gamma '(t)  , \gamma '(t)     \bigr \ra _{\gamma(t)} ~\ud t \biggr ] ^{1/2}, 
  \#
  where the infimum is taken over all curves on $\cP_2(\cX)$ such that $\gamma(0) = \mu$ and $\gamma(1) = \nu$, and $\la \cdot   ,  \cdot  \ra _{\gamma(t)}$ in \eqref{eq:w2_def2} is the inner product on the Riemannian manifold defined in \eqref{eq:riemann_metric}. The infimum in \eqref{eq:w2_def2} is attained by the geodesic jointing $\mu$ and $\nu$. Here we slightly abuse the notation by letting $\mu$ and $\nu$ denote the probability measures as well as their densities. 
  In  geodesic space $(\cP_2(\cX), W_2)$, we similarly define the exponential mapping and the parallel transport.
   Furthermore, thanks to the Riemannian metric, we can define the gradient and convexity for functionals  on $\cP_2(\cX)$ in the same way as in \S \ref{sec:riemann} with the geodesic distance $\| \cdot \|$ in \eqref{eq:define_sc} being the Wasserstein distance $W_2 $. 
  In the sequel, for any functional $F \colon \cP_2(\cX) \rightarrow \RR$, we let  $\grad F$ denote the the gradient of $F$ with respect to $W_2$. 
  Such a Riemannian  gradient is also also known as  the Wasserstein gradient.

\subsection{Reproducing Kernel Hilbert Space} \label{bg:rkhs}

Another important tool we need in the following analysis is the reproducing kernel Hilbert space (RKHS). 
Let $(\mathcal{H}, \la \cdot,\cdot \ra_{\mathcal{H}} )$ be a Hilbert space, where $\cH$ is a function class on $\cX$ and  $\la \cdot,\cdot \ra_{\mathcal{H}} \colon \cH \times \cH \rightarrow  \RR $ is the inner product structure on $\cH$. This Hilbert space  is an RKHS if and only if   there exists a feature mapping $K \colon \cX \rightarrow \cH$  such that 
$
f(x)=\la f, K_x\ra_{\mathcal{H}}
$
for all $f\in \cH$ and all $x \in \cX$, where $K$ maps each point $x \in \cX$ to $K_x \in \cH$. Moreover, this enables us to define a  kernel function $K \colon \cX \times \cX \rightarrow \RR $  by letting $K(x, y) = \la K_x, K_y \ra_{\cH } $ for any $x, y \in \cX$ and the feature mapping can be written as $K_x = K(x, \cdot) $.
In addition, an equivalent definition of RKHS is as follows. A Hilbert space $(\mathcal{H}, \la \cdot,\cdot \ra_{\mathcal{H}})$ is an RKHS if and only if for any $x \in \cX$, there exists a constant $M_x >0$ such that for all $f\in \mathcal{H}$,
\$
|f(x)|\le M_x \cdot \| f\|_{\mathcal{H}}.
\$
This definition of RKHS does not require an explicit form of the kernel function $K(\cdot,\cdot)$ and is hence relatively easier to verify  in practice. 

Given an RKHS $\mathcal{H}$ with a bounded kernel $K(\cdot,\cdot)$, let us recall some basic results of the connection between $\mathcal{H}$ and $\cL_{\nu}^2(\cX)$, which is the space of square-integrable functions with respect to a measure $\nu \in \cP(\cX)$. In what follows, we assume that $\nu$ has a positive density function.
As shown in Theorem 4.27 in  \cite{steinwart2008support},   the embedding operator $\mathcal{I}: \mathcal{H}\to \cL_{\nu}^2(\cX)$ is well-defined, Hilbert-Schmidt,  bounded, and its operator norm satisfies
\# \label{kernel_op}
\|\mathcal{I}\|_{\text{op}}=\Bigl[  \int_{\cX} K(x,x)~\ud \nu(x) \Bigr]^{1/2}<\infty.
\# 
For the notational simplicity, we write $[f]_{\nu} \in \cL_{\nu}^2(\cX)$ as the image of any $f\in \cH$ under $\cI$.  
Moreover, the adjoint operator $\mathcal{S}=\mathcal{I}^*: \cL_{\nu}^2(\cX) \to \mathcal{H}$ satisfies that 
$ \cS f =  \int_{\cX}K_{x} \cdot f(x)\ud \nu(x)$ for any $f \in \cL_{\nu}^2(\cX)$. In other words, for any $x \in \cX$ and any $f \in \cL_{\nu}^2(\cX)$, we have 
\$
(\mathcal{S}f)(x) = \la \cS f , K_x \ra _{\cH } = \int_{\cX}\la K_x,K_{x'}\ra _{\cH}\cdot f(x')~\ud \nu(x') =\int_{\cX}K(x,x')\cdot f(x')~\ud \nu(x').
\$
We further define the composition operators   $\cC = \cS \circ \cI$ and $\cT = \cI \circ \cS$, 
which by definition are  self-adjoint and positive semi-definite integral operators on  $\mathcal{H}$ and $\cL_{\nu}^2 (\cX)$, respectively. 
Moreover, by definition, for any $f \in \cH$ and any $h \in \cL_{\nu} ^2(\cX)$, we have 
\#
\label{eq:define_TC}
\cC f = \int_{\cX } K_{x} \cdot [f]_{\nu} (x)~ \ud \nu(x), \qquad \cT h (x) =  \int_{\cX } K(x, x') \cdot h (x')~\ud \nu(x'). 
\#
The operator $\cT \colon \cL_{\nu}^2 (\cX) \rightarrow \cL_\nu^2 (\cX)$ is also known as the integral operator induced by kernel $K$. 
 By definition, for any $f, g \in \cH$, it holds that
\#\label{eq:inner_prod_relation}
\la \cC f , g\ra_{\cH} = \int_{\cX } \la K_{x}, g\ra _{\cH} \cdot  [f]_{\nu} (x) ~\ud \nu(x) =\bigl  \la  [f]_{\nu},  [g]_{\nu} \bigr \ra_{\nu} ,
\#
where $\la \cdot, \cdot \ra _{\nu}$ is the inner product of $\cL^2_\nu (\cX)$, and $[f]_\nu$ and $[g]_{\nu}$ are the images of $f$ and $g$ under $\cI$, respectively.
 Note that $\cI$ is injective, since $\nu$ has a positive density function.
By Mercer's Theorem \citep{steinwart2012mercer}, when the embedding $\cI$ is injective, the integral operator $\mathcal{T}$ has countable and positive eigenvalues $\{\mu_i\}_{i\ge 1}$ and corresponding eigenfunctions $\{ \psi_i\}_{i\ge 1}$, which form  an orthogonal system of $\cL_{\nu}^2(\cX)  $. Note that the embedding $\cI$ is injective when $\nu$ has a positive probability density function.
Moreover, the RKHS $\cH$ can be written as a subset of $\cL_{\nu}^2 (\cX)$ as  
\$
\cH = \biggl \{ f \in \cL_{\nu}^2 (\cX) \colon \sum_{i=1}^\infty    \frac{ \la f ,\psi_i \ra _{\nu} ^2 }{\mu_i } < \infty  \biggr \},
\$
which is equipped with an  RKHS inner product $\la \cdot, \cdot \ra_{\cH}$ defined by 
\$
\la f , g\ra _{\cH } = \sum_{i=1}^{\infty} 1/  \mu_i \cdot \la f ,\psi_i \ra _{\nu} \cdot \la g ,\psi_i \ra _{\nu}. 
\$
By such a construction, 
 the scaled eigenfunctions $\{\sqrt{\mu_i}\psi_i \}_{i\ge 1} $  form  an orthogonal system of the RKHS $ \mathcal{H}$ and the feature mapping $K_x \in \cH$ is given by 
 $
 K_x = \sum_{i=1}^\infty \mu_i\psi_i \cdot \psi_i (x)   
 $
 for any $x\in \cX$.
 Thus,  the integration operator $\mathcal{C}$ admits the following spectral decomposition 
\$
\mathcal{C} f =\int_{\cX } K_{x} \cdot f (x)~ \ud \nu(x) =  \sum_{i=1}^{\infty} \mu_i \la f, \psi_i \ra_{\nu }\cdot   \psi_i  = \sum_{i=1}^{\infty} \mu_i \la f , \sqrt{\mu_i } \psi_i \ra_{\cH  }\cdot  \sqrt{\mu_i } \psi_i , \qquad  \forall f\in \mathcal{H}.
\$
Finally, the spaces  $\mathcal{H}$ and $\cL_{\nu}^2(\cX)$  are unified via the notion of $\alpha$-power space, which is defined by
\#\label{eq:power_space}
\mathcal{H}^{\alpha}:= \Bigl\{\sum_{i=1}^{\infty} u_i \mu_i^{\alpha/2} \psi_i: \sum_{i=1}^{\infty} u_i^2<\infty \Bigr\} = \Bigl\{f \in \cL_{\nu}^2(\cX) : \sum_{i=1}^\infty    \frac{  \la f, \psi_i \ra_{\nu}  ^2 }{\mu_i^\alpha }  \Bigr\} , \qquad   \alpha\ge 0, 
\#
and is equipped with an inner product $
\la f , g \ra _{\cH^{\alpha }} = \sum_{i=1}^\infty \mu_i ^{-\alpha}\cdot  \la f, \psi_i \ra_{\nu} \cdot \la g, \psi_i \ra_{\nu} 
$.
By definition,  $\mathcal{H}^{\alpha}$ is a Hilbert space for any $\alpha \ge 0$.  
 Two special cases  are  $\alpha = 0$ and $\alpha = 1$, where  $\cH^{\alpha}$ corresponds to  $L_{\nu}^2(\cX)$ and   $\mathcal{H}$, respectively.  
 Intuitively, as $\alpha$ increases from zero,  the $\alpha$-power spaces specifies a cascade of more and more restricted function classes.




\section{The Variational  Transport Algorithm} \label{sec:algo}
 
Let $\cX$ be a compact $d$-dimensional  Riemannian manifold without boundary and let $(\cP_2(\cX), W_2)$ be the Wasserstein space   on $\cX$, where $\cP_2(\cX)\subseteq \cP(\cX)$ is defined in \eqref{eq:prob_manifold} and $W_2 $ is the Wasserstein distance defined in \eqref{eq:w2_def1}. 
Note that each element in $\cP_2(\cX)$ is a probability density over $\cX$. 
For intuitive understanding,   $\cX$ can be regarded as   a  compact subset of the  Euclidean space  $\RR^d$ with a  periodic boundary condition, e.g., the torus $\TT^d$. 
Besides, 
to differentiate functionals on $\cX$ and $\cP_2(\cX)$, we refer to functional on $\cX$ as functions hereafter.   

Let $F\colon \cP_2(\cX)\rightarrow \RR$ be a differentiable functional on $\cP_2(\cX)$ with a variational form, 
\#\label{eq:var_func} 
F(p) = \sup_{f \in \mathcal{F}}  \biggl \{ \int _{\cX} f(x) \cdot  p(x)  ~\ud x - F^*(f) \biggr\},
\#
where $\cF$  is a class of square-integrable functions on $ \cX$ with respect to the Lebesgue measure, and $F^* \colon \cF \rightarrow \RR $ is a strongly convex functional.
Such a variational representation 
generalizes convex conjugacy  
to $\cP_2(\cX)$. 
In the sequel, 
 we consider the following distributional optimization problem,
\#
\label{eq:var_opt}
	\minimize_{p \in \cP_2(\cX)}~ F(p)  
	\#
over the Wasserstein space $\cP_2(\cX)$. Such an optimization problem  
  incorporates many prominent examples  in statistics and machine learning, which are listed as follows.

\subsection{Examples of Distributional Optimization}\label{sec:examples of distributional}

\begin{example} \label{example_nonconvex} [Nonconvex Optimization]
	Let $g \in \cF $ be a possibly nonconvex function on $\cX$. 
	For the  problem  of minimizing $g$ over $\cX$,  we 
	consider the following variational reformulation known as the lifting trick,
\#\label{eq:nonconvex_func_opt}
\minimize_{x \in \cX } ~g(x) \qquad \Longrightarrow \qquad \minimize_{p \in \cP_2(\cX)}~ F_{g} (p) =  \int_{\cX} g(x) \cdot p(x)  ~\ud x,
\#
which bridges a $d$-dimensional possibly nonconvex optimization problem over $\cX$ and  an infinite-dimensional linear optimization problem over $\cP_2(\cX)$. 
Moreover, when the global minima of $g$ is a discrete set, 
for any global minimum $x^\star$ of $g$, 
the Dirac $\delta$-measure $\delta_{x^\star}$ at $x^\star$ is a global optimum of the distributional optimization problem in \eqref{eq:nonconvex_func_opt}.
To see that the linear functional $F_g$ admits a variational form as in \eqref{eq:var_func}, we define $F_g^* \colon \cF \rightarrow \RR $ by letting $F_g^*(g) = 0$ and $F_g^* (f) = +\infty$ for any $f \neq g$. 
Then it holds that 
\$  
F_g(p) = \sup_{f \in \cF } \biggl \{ 
	\int _{\cX} f(x) \cdot  p(x)  ~\ud x - F_g^*(f) \biggr\} = \int_{\cX} g(x) \cdot p(x)  ~\ud x,
\$
which implies that the lifted   problem in \eqref{eq:nonconvex_func_opt} is a degenerate example of the distributional optimization problem in \eqref{eq:var_func} and \eqref{eq:var_opt}. 

Moreover, 
 following the principle of maximum entropy \citep{GuiShe85,ShoJoh80}, we can further incorporate a Kullback-Leibler (KL) divergence regularizer into \eqref{eq:nonconvex_func_opt}. Specifically, letting $p_0 \in \cP_2(\cX)$ be a prior  distribution, for any $p \in \cP_2(\cX)$, 
$\textrm{KL}(p, p_0)$ can be written as  a variational form 
\citep{nguyen2010estimating} as  
\#\label{eq:KL_form}
\textrm{KL} (p, p_0) = \sup_{h\in \cF}  \Bigl \{  1 + \EE_{X\sim p}  [ h(X) ] - \EE_{X\sim p_0} \big \{ \exp[ h (X)   ]\big  \}     \Bigr \} ,
\#
where the optimal solution is given by $h^\star  (x) = \log [ p(x) / p_0(x)]$. 
Then we consider a 
regularized distributional optimization problem 
\#\label{eq:regularized_nonconvx_opt}
\minimize_{p \in \cP_2(\cX)} ~ F(p) &= \biggl \{ \int_{\cX} g(x) \cdot p(x) ~\ud x + \tau \cdot \mathrm{KL} (p, p_0)\bigg \}  \notag \\
&   = \sup_{f \in \cF}  \bigg\{  \int_{\cX} f(x)     \cdot p(x) ~\ud x - \int_{\cX} \tau \cdot \exp \bigl \{ [f(x)  - g (x) ] / \tau \bigr \}  \cdot p_0(x) ~\ud x + \tau    \biggr \}, 
\#
where in the second equality we utilize \eqref{eq:KL_form} and let $f = g+ \tau \cdot h$. 
Thus, \eqref{eq:regularized_nonconvx_opt} also falls in the framework of \eqref{eq:var_func} and  \eqref{eq:var_opt} with $F^*(f) = \tau \cdot  \EE_{X\sim p_0}( \exp\{ [ f(X)  - g  (X)   ] / \tau \}) - \tau $. 
  Furthermore, the Gibbs variational principle \citep{PetRagVer89} implies that the global optimum of \eqref{eq:regularized_nonconvx_opt} is given by 
\#\label{eq:gibbs_global_optim}
  p_\tau^\star(x)   \propto \exp\bigl[ -g(x)/\tau\bigr] \cdot p_0(x),\qquad  \forall x\in \cX.
\#
Assuming that $p_0$ is an uninformative and improper prior, the posterior  $p_\tau^\star$ in \eqref{eq:gibbs_global_optim}  is a Gibbs measure that concentrates to the global optima of $g$  as $\tau\rightarrow 0$.  
\end{example}

\begin{example}[Distributionally Robust Optimization (DRO)]
	Let $\Theta$ be a compact subset of $\RR^{\bar d}$ and let $\ell (\cdot; \cdot ) \colon \cX \times \Theta \rightarrow \RR$ be a bivariate function, where $\bar d$ is a fixed integer. 
	We consider the following DRO problem \citep{rahimian2019distributionally},
\#\label{eq:dro}
\min_{\theta\in \Theta }\max_{p \in \cM   }  \int_{\cX} \ell(x; \theta)\cdot p(x)~ \ud x,
\#
where 
$\cM \subseteq \cP_2(\cX)$ is called the ambiguity set.
One commonly used ambiguity set is the level set of KL divergence, i.e., 
$\cM : = \{p: \textrm{KL} (p, p_0) \leq \epsilon\}$ for some $\epsilon > 0$, where $p_0$ is a given probability measure, namely the nominal distribution. In the context of supervised learning, $\ell$ in \eqref{eq:dro} is the loss function and the the goal of DRO is to minimize the population risk, which is the expected loss  under an adversarial perturbation of the data generating distribution $p_0$ within the ambiguity set $\cM$.  
When $\cM$ is specified by the level set of KL divergence, for any fixed $\theta$,  using Lagrangian duality, we can transform the inner problem in \eqref{eq:dro} into   a KL divergence regularized distributional optimization problem as in \eqref{eq:regularized_nonconvx_opt} with $g $ is replaced by $\ell(\cdot ; \theta)$. 
As a result, the inner problem of DRO can be formulated as an instance of the distributional optimization problem given by \eqref{eq:var_func} and \eqref{eq:var_opt}. 
\end{example}

\begin{example}[Variational  Inference] In  Bayesian statistics \citep{gelman2013bayesian},  it is critical to estimate the posterior distribution  of the latent variable $Z\in \cZ$ given data $X$.  The density of  the  posterior distribution  is given by 
	\#\label{eq:posterior}
	p (z \given x) = \frac{ p(x\given z) \cdot p_{0} (z)} { \int_{\cZ} p(x\given z) \cdot p_{0} (z ) ~\ud z},
	\#
	where $x$ is the observed data,  $ p(x\given z) $ is the likelihood   function, and $ p_{0}  \in \cP(\cZ)$ is the prior distribution of the latent variable. 
	Here $p(x\given z)$ and $p_0$ are given in closed form. 
	In many statistical  models, the computation of the integral in \eqref{eq:posterior}  is intractable. 
	To circumvent such intractability, variational inference turns to minimize the KL divergence between a variational posterior $p$ and the true   posterior $p(z\given x)$ in   
	\eqref{eq:posterior} \citep{WaiJor08,BleKucMca17}, yielding the following distributional optimization problem,
 \#\label{eq:variational_inference_solution}
p^\star(z)  = \argmin_{p\in \cP_2(\cZ) }\mathrm{KL} \bigl[ p(z),  p(z\given x)\bigr ]  = \argmax_{p\in \cP_2(\cZ) } \biggl \{ -\int_\cZ \log\biggl[\frac{p(z)}{p(x\given z) \cdot p_0(z)}\biggr] \cdot p(z) ~ \ud z \biggr \} .
 \#
 By the variational form of KL divergence in \eqref{eq:KL_form}, 
 we observe that such a problem belongs to the framework given by \eqref{eq:var_func} and \eqref{eq:var_opt}.
 Meanwhile, the right-hand side of \eqref{eq:variational_inference_solution} is known as the evidence lower bound, which avoids the integral in \eqref{eq:posterior}. The common practice of variational inference is to parameterize $p$ by a finite-dimensional parameter and optimize such a parameter in \eqref{eq:variational_inference_solution}, which, however,  potentially leads to a  bias in approximating the true posterior $p(z\given x)$. As is shown subsequently, our proposed framework avoids such bias by representing $p$ with particles. 
	\end{example}

\begin{example}[Markov-Chain Monte-Carlo (MCMC)]\label{eg:mcmc} 
	Following the previous example,  consider
	sampling from the posterior distribution $p(z\given x)$ via MCMC. 
	One example is the  (discrete-time) Langevin MCMC algorithm \citep{GilRicSpi95,BroGelJonMen11, welling2011bayesian}, which constructs a Markov chain $\{ z_k\}_{k\geq 0}$  given by  
 \$ 
 z_{k+1} \leftarrow z_k - \gamma_k \cdot \nabla \log \bigl [ p(y\given z_k) \cdot p_{0} (z_k)\bigr] + \sqrt{2 \gamma_k} \cdot  \epsilon_k, 
 \$
 where $\gamma _k > 0$  is the stepsize and $\{ \epsilon_k\}_{k\geq 0}$ are independently sampled from $N(0, I_d)$. 
 Such a Markov chain   
 iteratively attains the limiting probability measure $p^\star  $ in \eqref{eq:variational_inference_solution} while  avoiding the intractable   integral in \eqref{eq:posterior}.  
 See, e.g., 	\cite{cheng2017underdamped, cheng2018convergence, xu2018global, durmus2019high} and the references therein  for the analysis of  the  Langevin MCMC algorithm. 
  Besides, it is shown that (discrete-time) Langevin MCMC can be viewed as (a discretization of) the Wasserstein gradient flow of $\textrm{KL} [ p(z),  p(z\given x))$ over $\cP_2(\cZ)$ \citep{jordan1998variational, wibisono2018sampling}. 
  In other words, posterior sampling with Langevin MCMC can be posed as a  distributional optimization method. 
  Furthermore,  in addition to the KL divergence, $F(p)$ in \eqref{eq:var_func} also incorporates other $f$-divergences  \citep{Csi67}. 
  For instance, 
   the  $\chi^2$-divergence between $p$ and $q$ can be  written as 
\$
\chi^2 (p, q) = \sup_{f\in \cF}\Bigl \{  2 \EE_{X\sim p} [ f(X) ] - \EE_{X\sim q} [ f^2 (X) ] - 1\Bigr \}.
\$
As a result,  when applied to posterior sampling, our proposed framework generalizes Langevin MCMC beyond the minimization of the KL divergence and further allow   more general objective functionals such as   $f$-divergences.
\end{example}

\begin{example}
	[Contextual Bandit] We consider the following interaction protocol between an agent and the environment. At the $t$-th round, the agent observes a context $c_t \in \cC$, which can be chosen adversarially. Then the agent takes an action $a_t \in \cA$ and receives the reward $r_t \in \RR$, which is a random variable with mean $R(c_t, a_t)$, where $R \colon \cC \times \cA\rightarrow \RR$ is known as the reward function.  Here $\cC$ and $\cA$ are the context and action spaces, respectively, and $R$ is unknown. The agent aims to maximize its expected total reward. To this end, at the $(t+1)$-th around, the agent estimates the reward function $R$ within a function class $\cF := \{f_z: z\in \cZ\}$ based on the observed evidence $\cE_t := \{(c_t, a_t, r_t)\}_{t'=1}^{t}$, where $z$ denotes the parameter of $f_z$. The two most common algorithms for contextual bandit, namely exponential weighting for exploration and exploitation with experts (Exp4) \citep{FarMeg05} and Thompson sampling (TS) \citep{AgrGoy12,AgrGoy13}, both maintain a probability measure over $\cF$, which is updated by 
\#\label{eq:w101}
 p_{t+1}(z) \leftarrow \argmin_{p \in \cP_2 (\cZ)} \biggl \{ \int_{\cZ} G_t(z) \cdot p(z) ~\ud z + \frac{1}{\gamma_t}  \cdot \textrm{KL} (p ,  p_t) \biggl \} \propto \exp\bigl[ -\gamma_t \cdot G_t(z)\bigr]  \cdot p_t(z),
\#
where $G_t$ is a function specified by the algorithm and $\gamma _t > 0$ 
 is the stepsize. 
 In particular, Exp4 sets $G_t(z)$ as  the estimated reward of the greedy policy with respect to $f_z$, while TS sets $G_t(z)$ to be the log-likelihood of the observed evidence $\cE_t$ so that $p_{t+1}(z)$ is the posterior. 
 We observe that the KL divergence regularized distributional optimization problem in \eqref{eq:w101} falls in our proposed framework. 
 Moreover, 
when $\cZ$ is complicated, existing Exp4-based algorithms mostly rely on the discretization of $\cZ$ to maintain $p(z)$ as a probability vector \citep{AueCesFis02}. 
In contrast, as we will see 
subsequently, our proposed algorithm  bypasses the discretization of $\cZ$ in Exp4, and meanwhile, eases the computational overhead of TS. 
\end{example}

\begin{example}[Policy Optimization] In reinforcement learning \citep{sutton2011reinforcement}, the model is commonly formulated as a Markov decision process $( \cS, \cA, R, P, \gamma )$ with a state space $\cS$, an action space $\cA$, a reward function $R$,   a  Markov transition kernel $P$, and a discount factor $\gamma \in (0,1)$. A  policy $\pi$ maps each state $s\in \cS$ to  a distribution $\pi(\cdot \given s)$  over the action space $\cA$. Any policy $\pi$ induces a visitation measure $\rho_{\pi}$ over $\cS \times \cA$, which is defined as 
	\#\label{eq:visitation}
	\rho_{\pi} (s,a) = \sum_{t \geq 0}^{\infty} \gamma ^t \cdot  \PP ( s_t= s, a_t = a), \qquad  \forall (s,a) \in \cS\times \cA.
	\#
	Here $s_t$ and $a_t$ are the $t$-th state and action, respectively, and  the trajectory $\{ s_t, a_t \}_{t \geq 0}$ is generated following policy $\pi$.  
	Let $\rho_0 \in \cP(\cS\times \cA)$ be a fixed distribution.
	We consider the 
	regularized policy optimization problem  \citep{ziebart2010modeling,schulman2017equivalence}
	\$
	\maximize_{\pi} \bigl \{   \EE_{(s,a) \sim \rho_{\pi} } [ R(s, a)] - \lambda \cdot  \textrm{KL}(  \rho_\pi, \rho_0 ) \bigr \} ,
	\$ where $\lambda > 0$ 
	is the regularization parameter  that balances 
	between 
	exploration and exploitation.   
	This problem can be viewed as a distributional optimization problem  with an extra constraint that the distribution of interest  is a visitation measure induced by a policy. 
\end{example}

\begin{example}[Generative Adversarial Network (GAN)] 
	The goal of GAN 	\citep{goodfellow2014generative} is to learn a generative model $p$ that is close to a target distribution $q$, where   $p$ is defined by transforming a low dimensional noise via a neural network. Since the objective in \eqref{eq:var_func} includes $f$-divergences as special cases, our distributional optimization problem incorporates the family of $f$-GANs 	\citep{nowozin2016f}  by letting the function class $\cF$ in \eqref{eq:var_func} be the class of discriminators. 
Besides, by letting $\cF$ be the family of $1$-Lipschitz continuous functions defined on $\cX$, 
the first-order Wasserstein distance between $p$ and $q$ can be written as 
\#\label{eq:wasserstein_gan}
W_1(p, q) = \sup_{f\in \cF} \bigl  \{    \EE_{X\sim p} [f(X)]    - \EE_{X \sim q} [ f(X) ] \bigr \}, 
\#
which is of the same variational form as in \eqref{eq:var_func}. Thus, the  distributional optimization  in \eqref{eq:var_opt} includes  Wasserstein GAN  \citep{arjovsky2017wasserstein} as a special case. 
  Moreover, 
  by letting $\cF$ in \eqref{eq:wasserstein_gan} be various function classes,
  we further recover other integral probability metrics   \citep{muller1997integral} between probability distributions, which induce  a variety of GANs \citep{mroueh2017fisher,li2017mmd,mroueh2018sobolev, uppal2019nonparametric}. 

\end{example}

 \subsection{Variational Transport} \label{sec:sub:algo}
In what follows, we 
 introduce the variational transport algorithm 
 for the distributional optimization problem specified in \eqref{eq:var_func} and \eqref{eq:var_opt}.  
 Recall that $\cX$ is a compact Riemannian manifold without a boundary and   $(\cP_2(\cX), W_2) $ defined in \eqref{eq:prob_manifold}
is a geodesic space equipped with a Riemannian metric. 
To simplify the presentation, in the sequel, we specialize to the case where $\cX$ is a compact subset of $\RR^d$ with a periodic boundary condition. 
 For instance, $\cX$ can be a torus $\TT ^d$, which can be viewed as the $d$-dimensional hypercube $[0,1)^d$ 
where the opposite $(d-1)$-cells are  identified,  
We specialize to  such a structure only for  rigorous  theoretical analysis, which also appears in other works involving the Wasserstein space \citep{graf2015fast}.
Our results can be readily generalized to a general $\cX$ with extra technical care. 
 We will provide a brief introduction of the Wasserstein space defined on $\TT^d$ in \S\ref{sec:torus}.

Moreover, note that the objective function $F(p)$ in \eqref{eq:var_func} admits a variational form as the  supremum over a function class $\cF$. Since  maximization over an unrestricted function class is computationally intractable,    for the variational problem in \eqref{eq:var_func} to be well-posed, in the sequel, we consider  $\cF$ to be a subset of the square-integrable functions   such that maximization over $\cF$ can be   solved numerically, e.g., 
 a class of deep neural networks or an RKHS. 

Motivated by the first-order methods in the Euclidean space, we propose a first-order iterative optimization algorithm over   $\cP_2(\cX)$. Specifically, suppose $p  \in \cP_2(\cX) $ is the current iterate of our algorithm, we  
  would like to update $p $ in the direction of the Wasserstein gradient $\grad F(p )$. Moreover, to ensure that $p $ lies in $\cP_2(\cX)$,  we update $p$ by moving in the direction of $\grad F(p )$ along the geodesic. 
 Thus, in the ideal case,   the Wasserstein gradient update   is given by 
\#\label{eq:grad_update}
p \leftarrow \expm_{p } \bigl [- \alpha   \cdot \grad F(p ) \bigr ], 
\#
where $\expm_{p }$ is the exponential mapping at $p $ and $\alpha > 0$ is the stepsize. 

To obtain an implementable algorithm,  it remains to characterize the Wasserstein  gradient $\grad F(p)$  and  the exponential mapping $\expm_p$  at any $p \in \cP_2(\cX)$.  The following proposition  establishes the connection between the Wasserstein  gradient
and functional derivative 
 of a  differentiable  functional on $\cP_2(\cX)$.

\begin{proposition}[Functional Gradient] \label{prop:func_grad}
	We assume $F: \cP_2(\cX) \to \RR$ is a differentiable functional with its functional derivative 
	with respect to the $\ell_2$-structure 
	denoted by $\delta F(p)/\delta p$. Then, it holds that  
\#\label{eq:riem_grad}
		\grad F(p) = -\dvg \bbr{p\cdot \nabla \bpa{\frac{\delta F}{\delta p}}},
\#
where $\dvg$ is the divergence operator on $\cX$.
Furthermore, for the functional $F $ defined in \eqref{eq:var_func}, 
we assume that  $F^*$ is strongly convex on $\cF$ 
and that   the supremum is attained at $f^*_p \in \mathcal{F}$ for a fixed $p\in \cP_2(\cX)$, i.e., $F(p) = \EE_{X \sim p}   [f_p^*(X)] - F^* (f_p^*)$. Then, we have  $  \delta F/ \delta p = f^*_p$ and $\grad F (p) = - \dvg ( p \cdot \nabla f^*_p )$. 
\end{proposition}
\begin{proof}
	See \S\ref{proof:prop:func_grad} for a detailed proof.
	\end{proof}


By Proposition \ref{prop:func_grad},
to obtain $\grad F(p)$, we need to first solve the variational problem in \eqref{eq:var_func} to obtain $f_p^* \in \cF$ and then compute the divergence in \eqref{eq:riem_grad}. However, even if we are given $f_p^*$, 
 \eqref{eq:riem_grad} further requires the 
 analytic form of density function $p$, 
 which is oftentimes challenging to obtain.  
Furthermore, 
in addition to the challenge of computing $\grad F(p)$, 
to implement  the Wasserstein gradient  update in \eqref{eq:grad_update}, we also need to characterize the exponential mapping $\expm_{p}$. 
 Fortunately, in the following proposition, we  show  that exponential mapping along   any tangent vector in $\cT_p \cP_2(\cX)$ is  equivalent to a pushforward mapping of $p$, which enables us to perform the Wasserstein gradient update via pushing particles.

\begin{proposition}[Pushing Particles as Exponential Mapping]
\label{prop:push_par}
For any $p \in \cP_2(\cX)$ and any $s \in \cT_p \cP_2(\cX)$, suppose the elliptic equation $-\dvg (p \cdot \nabla u)  = s$ admits a unique solution $u \colon \cX \rightarrow \RR$  that is twice continuously differentiable. We further assume that $\nabla u \colon \cX\rightarrow \RR^d$ is $H$-Lipschitz continuous. Then  for any $t \in [0, 1/ H )$,
 we have 
\#\label{eq:push_eq}
\bigl [\expm_{\cX} ( t\cdot \nabla u) \bigr ] _{\sharp} p = \expm_p(t \cdot s), 
\#
 where $\expm_{\cX}$ denotes the exponential mapping  on   $\cX$  and   $\expm_p$  is the exponential mapping on  $\cP_2(\cX)$ at  point  $p$.
 Moreover,  for any $x \in \cX$, the density function of $   [\expm_{\cX} ( t\cdot \nabla u)   ] _{\sharp} p$ at $x$ is given by   
 \$
\Bigl \{ \bigl [\expm_{\cX} ( t\cdot \nabla u) \bigr ] _{\sharp} p \Big\} (x) = p(y) \cdot \biggl  | \frac{\ud }{\ud y} \expm_{y}  [ t \cdot \nabla u(y)]   \biggr  | ^{-1} ,
 \$
 where   $ x = \expm_{y}  [ t \cdot \nabla u(y)]$ and $ | \frac{\ud }{\ud y} \expm_{y}  [  t \cdot \nabla u(y)]     | $ is the determinant of the Jacobian. 
\end{proposition}
\begin{proof}
	See \S\ref{proof:prop:push_par} for a detailed proof.
\end{proof}


Note that the elliptic equation in Proposition \ref{prop:push_par} holds in the weak sense.
Combining Propositions \ref{prop:func_grad} and \ref{prop:push_par}, if $\nabla f_p^*$ is $H$-Lipschitz,  it holds for any $t \in [0, 1/H)$ that
 \#\label{eq:push_result}
 \expm_p [ -t \cdot  \grad F(p)] =\bigl [  \expm_{\cX} \bigl (-t\cdot \nabla  f_p^*) \bigr ]_{\sharp} p, 
 \#
 which reduces to $(\id + t \cdot   \nabla  f_p^*)  _{\sharp} p$ when $\cX = \RR^d$. 
 Here $\id \colon \RR^d \rightarrow \RR^d$ is the identity mapping. 
 Thus, when we have access to $f_{p}^*$, \eqref{eq:push_result} enables us to perform the Wasserstein gradient update   approximately via  pushing particles. Specifically,  assume that probability measure 
 $p$ 
  can be approximated by  an empirical measure  $\tilde p$ of $N$ particles $\{ x_i \}_{i\in [N]}$, that is, 
 \#\label{eq:empirical}
  p \approx \tilde p =   \frac{1}{N} \sum_{i=1}^N \delta _{x_i},
  \#
  where $\delta_x$ is the  Dirac $\delta$-measure for $x\in \cX$. 
  For each $i\in [N]$, we define 
  $
  y_i = \expm_{x_i} \bigl  [ - t \cdot \nabla f_p^* (x_i)\bigr ] . 
  $
 Then  we have   
 $
  [  \expm_{\cX} \bigl (- t\cdot \nabla  f_p^*)   ]_{\sharp} \tilde  p = N^{-1} \sum_{i=1}^N \delta _{y_i} $  by direct computation. 
By \eqref{eq:push_result} and \eqref{eq:empirical},   $  \expm_p [ -t \cdot  \grad F(p)] $ can be approximated by  the empirical measure of $\{ y_i\}_{i\in [N]}$.

In addition, since $f_p^*$ is unknown, we estimate $f_p^*$ 
 via the following maximization problem, 
 \#\label{eq:dual_prob}
 \maximize_{f \in  \mathcal{F} } ~   \int_{\cX}  f(x) ~\ud \tilde p(x) - F^*(f)  =   \frac{1}{N}  \sum_{i=1}^N f(x_i) - F^* (f) ,
 \#
 where we replace the probability distribution $p$ in \eqref{eq:var_func} by the empirical measure  $\tilde p$ in \eqref{eq:empirical}. When $N$ is sufficiently large,  since $\tilde p$ is close to $p$, we expect that the solution to \eqref{eq:dual_prob}, denoted by $\tilde f_p^*\in   \cF $,   is close to  $f_p^*$. 
This enables us to  approximate the  Wasserstein  gradient update $ \expm_{p } \bigl [- \alpha   \cdot \grad F(p ) \bigr ]$  by $[ \expm_{\cX} ( - \alpha \cdot \nabla \tilde f_p^*)] _{\sharp} \tilde p$, which is the empirical measure of 
 $\{ \expm _{x_i}   [ - \alpha \cdot \nabla \tilde f_p^* (x_i)  ]  \}_{i\in [N]}$. Therefore, we obtain the  variational transport  algorithm for solving the  desired distributional optimization problem in \eqref{eq:var_opt}, whose details are presented in Algorithm \ref{algo:main}. 
 
 \begin{algorithm} [htbp]
 	\caption{ Variational Transport Algorithm  for Distributional Optimization}  
 	\label{algo:main} 
 	\begin{algorithmic}[1] 
 		\STATE{{\textbf{Input:} Functional $F \colon \cP_2(\cX) \rightarrow \RR $ defined in \eqref{eq:var_func},  initial point $p_0 \in \cP_2(\cX)$, number of particles $N$, number of iterations $K$, and stepsizes $\{ \alpha_k\}_{k=0}^K$.}}  
 		\STATE{Initialize the  $N$ particles  $\{ x_i \}_{i\in [N]}$  by drawing  $N$ i.i.d.  observations from $p_0$.}
 		\FOR{$k = 0, 1,2, \ldots, K$}
 		\STATE{Compute $\tilde f^*_k \in   \cF$ via \# 
 			\label{eq:def_tilde_fk}
 			\tilde {f}^*_k \leftarrow   \underset{f \in \mathcal{F} }{\text{argmax}}~\biggl \{ \frac{1}{N} \sum_{i=1}^{N  } f(x _ i) - F^*(f) \biggr \}.
 			\#} 
 		\STATE{Push the particles by letting $x_i\leftarrow  \expm _{x_i}   [ -  \alpha_k \cdot \nabla \tilde f_k^* (x_i)  ] $ for all $i\in [N]$. } \label{line:5}
 		\ENDFOR
 		\STATE{{\textbf{Output:}} The empirical measure $N^{-1} \sum_{i\in [N] } \delta_{x_i}$.}
 	\end{algorithmic}
 \end{algorithm}
 
In this algorithm, we  maintains $N$ particles $\{x_i\}_{i\in [N]}$ and output their empirical measure as the solution to \eqref{eq:var_opt}.
These particles are updated in 
  Line \ref{line:5} of  Algorithm \ref{algo:main}. In the $k$-th iteration, the update  is equivalent to updating the empirical measure $\tilde p  = N^{-1} \sum_{i\in [N] } \delta_{x_i} $ by the pushforward measure  $[ \expm_{\cX} ( - \alpha_k \cdot \nabla \tilde f_k^*)] _{\sharp} \tilde p$, which approximates the Wasserstein   gradient update in \eqref{eq:grad_update} with stepsize $\alpha_k$. 
  Here $\tilde f_k^*$ is obtained in \eqref{eq:def_tilde_fk} by solving an empirical maximization problem over a function class $\cF$, which can be  chosen to be an RKHS or the class of deep neural networks in practice. 
 After getting $\tilde f_k^*$, we push each particle $x_i$  along the direction   $ \nabla \tilde f_k(x_i)$ with stepsize $\alpha_k$, which  leads  to   $[ \expm_{\cX} ( - \alpha_k \cdot \nabla \tilde f_k^*)] _{\sharp} \tilde p$ and completes one iteration of the algorithm. 
 
To gain some intuition of variational transport,  as a concrete example, consider the regularized nonconvex optimization problem introduced Example \ref{example_nonconvex} with $\cX = \RR^d$.
 Let $g \colon \RR^d \rightarrow \RR$ be a differentiable function on $\RR^d$. 
 Consider the distributional optimization  where the objective functional is defined as 
 $$F(p) = \int _{\cX} g(x)\cdot p(x) ~\ud x + \tau \cdot \mathrm{Ent}(p)   = \int _{\cX} g(x)\cdot p(x) ~\ud x + \tau \cdot \int_{\cX} p(x) \cdot \log p (x) ~\ud x
 . $$ 
 Here $\tau > 0$ is the regularization parameter and $\textrm{Ent}$ is the entropy functional. 
 For such an objective  functional $F$, the Wasserstein gradient is given by 
 \$
 \grad F(p)   =   - \dvg \bigl [  p \cdot \big ( \nabla g + \tau \cdot \nabla \log  p   \bigr ) \bigr ] .
 \$  
Then, suppose the $N$ particles $\{ x_i\}_{i\in [N]}$ are sampled from  density $p$, variational transport proposes to push each particle $x_i$ by letting 
\#\label{eq:particle_gd_ent}
x_i \leftarrow x_i - \alpha_k \cdot \bigl [ 
 \nabla g(x_i) + \tau \cdot \nabla \log p(x_i) \bigr ],
\#
which can be viewed as a regularized gradient descent step for objective function $g$ at $x_i$. 
  When $\tau = 0$, the update in \eqref{eq:particle_gd_ent} reduces to a standard gradient descent step. However, due to the nonconvexity of $g$, gradient descent might be trapped in a stationary point of $g$. 
By  introducing the entropy regularization,   variational transport aims to find   the  Gibbs distribution 
$p_{\tau}^* (x) \propto \exp[ -g(x) / \tau ]$,  which is the global minimum of $F$ and is supported on the global minima of $g$ as $\tau$ goes to zero.

Meanwhile, we remark that our variational transport algorithm is deterministic in essence: the only randomness comes from the initialization of the particles. 
  Moreover, this algorithm   can be viewed 
  as   Wasserstein  gradient descent  with biased  gradient estimates. 
  Specifically,
  to characterize the bias, we define  
  a   sequence of
  transportation maps $\{ T_k \colon \cX \rightarrow \cX \}_{k=0}^{K+1}$   as follows, 
  \#\label{eq:transport_maps}
  T_0  = \id , \qquad T_{k+1} =     [ \expm_{\cX}( - \alpha_k \cdot \nabla \tilde f_k^* )] \circ T_{k}, ~~ \forall k \in \{0, \ldots, K\},
  \#  
  where $\tilde f_k^*$ is obtained in \eqref{eq:def_tilde_fk} of Algorithm \ref{algo:main}. 
  For each $k\geq 1$, 
  we define 
  $\hat p_k = (T_k)_{\sharp} p_0$. 
Note that the population version of the optimization problem in \eqref{eq:def_tilde_fk} of Algorithm \ref{algo:main} and its solution are given by   
  \begin{align}
  	\label{eq:def_hat_fk}
  	\widehat{f}^*_k \leftarrow   \underset{f \in   {\mathcal{F}}}{\text{argmax}}~ \biggl \{ \int_{\cX} f(x) \cdot  \hat{p}_k(x) ~ \ud x  - F^*(f)  \biggr \}. 
  \end{align}
  Thus,   $\{ \hat p_k\}_{k \geq 0}$ are the iterates constructed by the ideal version of Algorithm \ref{algo:main} where the number of particles goes to infinity. 
  For such an ideal sequence, at each $\hat p_k$, 
  by Proposition \ref{prop:func_grad},  the desired Wasserstein gradient direction is 
  $\grad F (\hat p_k )= - \dvg (\hat p_k \cdot \nabla \hat f_k^* )$. 
   However, with only  finite data, we can only obtain an estimator $\tilde f_k^*$ of $\hat f_k^*$ by solving \eqref{eq:def_tilde_fk}. Hence,  the estimated Wasserstein  gradient at $\hat p_k$  is $- \dvg ( \hat{p}_k\cdot  \nabla \tilde {f}^*_k ) $, whose error is denoted by 
  $\delta _k = -\dvg [ \hat p_k \cdot ( \nabla \tilde f_k^* - \nabla \hat f_k^* )]$. Note that $\delta_k$ is a tangent vector at $\hat p_k$, i.e., $\delta_k \in \cT_{\hat p_k } \cP_2(\cX)$.  Using the Riemmanian metric on $\cP_2(\cX)$, we define 
  \begin{align}
  	\label{eq:def_varepsilon_k}
  	\varepsilon_k = \bag{\delta_k, \delta_k}_{\hat{p}_k} = 
  	\int_{\cX} \bigl \|    \nabla \tilde {f}^*_k (x)  - \nabla \widehat{f}^*_k (x) \bigr  \|_2 ^2 \cdot \hat p_k(x) ~\ud x. 
  \end{align}
  Therefore, Algorithm \ref{algo:main} can be viewed as a biased  Wasserstein  gradient algorithm on $\cP_2(\cX)$, where the squared norm of the bias term in the $k$-th iteration is  $ \varepsilon$, which decays to zero as $N$ goes to infinity. The algorithm constructs explicitly a sequence of transportation maps $\{ T_k\}_{k=0}^{K+1}$ and implicitly a sequence of probability  distributions $\{ \hat p_k \}_{k=1}^{K+1}$,  which is  approximated by  a sequence of  empirical distributions of $N$ particles.

Furthermore, for computational efficiency, in  practical implementations of the variational transport algorithm, we can push the particles without solving  \eqref{eq:def_tilde_fk} completely. 
For example,  when we employ  a deep neural network  $f_{\theta} \colon \cX \rightarrow \RR$ to solve \eqref{eq:def_tilde_fk}, where $\theta$ is the network weights,   we can replace \eqref{eq:def_tilde_fk}  by updating $\theta$ with a few gradient ascent updates. 
  Then, the particle $x_i$ is updated by $ \expm _{x_i}   [ -  \alpha_k \cdot  \nabla  f_\theta (x_i)  ] $. 
  Moreover, when $N$ is very large, we can also randomly sample a subset  of the $N$ particles to   update $f_{\theta}$ via mini-batch stochastic gradient ascent.  


\section{Theoretical Results}

In this section, we provide theoretical analysis for the variational transport algorithm. 
For mathematical rigor, 
in the sequel, 
we let $\cF$ in  \eqref{eq:def_tilde_fk} 
be an RKHS $(\cH, \la \cdot , \cdot \ra _{\cH})$ defined on $\cX$ with a  reproducing kernel $K (\cdot, \cdot)$.
In this case, the 
  the distributional optimization problem  specified in \eqref{eq:var_func} and \eqref{eq:var_opt} can be written as 
  \#\label{eq:new_problem}
 \minimize_{p\in \cP_2(\cX)} F(p) , \qquad  F(p) = \sup_{f \in \cH }  \biggl \{ \int _{\cX} f(x) \cdot  p(x)  ~\ud x - F^*(f) \biggr\}.
 \#
 In \S\ref{sec:stat_thoery}, we probe  statistical error incurred in solving  the inner maximization problem in \eqref{eq:new_problem},
 which provides an upper bound for the 
 bias of Wasserstein gradient estimates constructed by variational transport. 
 Then  
  in \S\ref{sec:opt_theory}, we establish the rate of  convergence of   variational transport   for a class of objective functionals satisfying the Polyak-\L{}ojasiewicz (PL) condition \citep{polyak1963gradient} with respect to the Wasserstein distance.

 \subsection{Statistical Analysis} \label{sec:stat_thoery}
 To characterize the bias 
 of   the Wasserstein  gradient estimates, 
 in this subsection, we provide a  statistical analysis of the inner maximization problem in \eqref{eq:new_problem}. 
 Without loss of generality, 
it suffices to consider the following nonparametric estimation problem. 
Let 
 $\PP$ be a probability measure over $\cX$ and $f^*$ be the minimizer of the population risk function over the  RKHS    $\cH $, i.e., 
 \#\label{eq:population_prob_stat}
 f^*=\argmin_{f \in \cH }  L(f) , \qquad \text{where} \quad L(f)  =  - \EE _{X\sim \PP} \bigl [ f(X) \bigr ]     + F^*(f) .
 \#
 Our goal is to estimate $f^*$  in \eqref{eq:population_prob_stat} based on
 $n$ i.i.d. observations 
 $X_1,\cdots X_n$   sampled from   $\PP$.   
 A natural estimator of $f^*$ is the minimizer of the empirical risk with an  RKHS norm regularizer, which is defined as 
 \# \label{eq:estimator}
 f_{n,\lambda}=\argmin_{f \in \mathcal{H}}\bigl \{  L_n(f) +\lambda /2 \cdot \|f\|^2_{\mathcal{H}}\bigr\}, \qquad \text{where} \quad L_n(f)=- \frac{1}{n} \sum_{i=1}^n f(X_i)+F^*(f).
 \#
 Here $ \lambda>0 $ is the  regularization parameter and    $L_n$ is  known as  the  empirical risk function.

 The population problem defined in \eqref{eq:population_prob_stat} reduces to  \eqref{eq:def_hat_fk} when we set  the density of $\PP$  to be $\hat p_k$ and the observations $\{ X_i\}_{i\in [n]}$ to be the particles $\{x_i\}_{i\in [N]}$ sampled from  $\hat p_k$.  Then, the estimator $f_{n, \lambda}$ in \eqref{eq:estimator} is a regularized version of $\tilde f_k^*$ defined in \eqref{eq:def_tilde_fk}.  
 In the sequel,  we 
 upper bound  the statistical error   $f_{n, \lambda} - f^*$ in terms of  the RKHS norm, which is further used to obtain an upper bound of $\varepsilon_k$ defined in \eqref{eq:def_varepsilon_k}.

 Before presenting the theoretical results, we first introduce a few   technical assumptions.
Let $\nu$ denote the Lebesgue measure on $\cX$.
  As introduced in \S\ref{bg:rkhs}, $\cH$ can be viewed as a subset of $L_{\nu}$. The following assumption specifies  a regularity condition for   functional $F^*$ in \eqref{eq:population_prob_stat}.
  
 \begin{assumption} \label{assume:convexity}
 	We assume that the functional $F^* \colon \cH \rightarrow \RR$ is smooth enough such that $F^*$ is  Fr\'echet differentiable up to the third order. For $j \in \{ 1,2,3\}$, we denote by $\cD^j F^*$ the $j$-th order Fr\'echet derivative   of $F^*$ in the RKHS norm. 
 	In addition, we assume that there exists a constant $\kappa > 0$ such that,    for any $f, f'\in \cH$,  we have  
 	\#\label{eq:strongconvex}
 	F^* (f)- F^*(f') -  \bigl\la \cD F^*(f'), f-f' \bigr\ra_{\cH}  \ge\frac{\kappa }{2}\int_{\cX} \bigl|f (x)-f' (x) \bigr|^2 ~\ud x.
 	\#
  
 \end{assumption}

Assumption \ref{assume:convexity} is analogous to the strongly convexity condition in parametric statistic problems. 
It lower bounds the  Hessian of the population risk $L$ defined in  \eqref{eq:population_prob_stat}, which enables us to 
characterize the distance between any function and the global minimizer $f^*$ based on the  population risk.  It is worth noting that  the right-hand side of \eqref{eq:strongconvex} is an integral with respect to the Lebesgue measure. 
 
We introduce a regularity condition for the global minimizer of the population risk $L$. 
 \begin{assumption} \label{assume:high-order}
 	There exists a constant $\beta \in (1, 2)$ such that $f^* $ defined in \eqref{eq:population_prob_stat} belongs to   $\mathcal{H}^{\beta}$, which is  $\beta$-power space of order $\beta$ defined in \eqref{eq:power_space}. 
 \end{assumption}

Assumption \ref{assume:high-order} requires that  the global minimizer $f^*$  of the population risk  $L$    belongs to a higher order power space of the RKHS $\cH$. In other words,  $f^*$  has greater smoothness than   a general function in   $\cH$. This assumption is made solely  for technical reasons and is  adopted from \cite{fischer2017sobolev}, which is used to  control the bias caused by the RKHS norm regularization. 
 
Before we lay out our next assumption,  we define 
 \#\label{eq:def_hat_f_lam}
 \hat{f}_{\lambda}= f ^* - \lambda \cdot \bigl  (\mathcal{D}^2 F^* (f^*)+\lambda \cdot  \texttt{id}  \bigr  )^{-1} \cdot f^* =  \bigl[  \bigl (\mathcal{D}^2 F^* (f^*)+\lambda \cdot \texttt{id} \bigr )^{-1}\circ \mathcal{D}^2 F^* (f^*) \bigr] \cdot  f^*,
 \#
 which is the solution to the   linearized equation 
 \#\label{eq:linearized_eq}
 \cD L_{\lambda} (f^*)  +   \bigl [ \cD^2 L_{\lambda} (f ^* )  \bigr ] ( f - f^* )  = 0.
 \#
 Here   $L_{\lambda} (f) = L(f) + \lambda /2 \cdot \| f \|_{\cH}^2$ is the RKHS norm  regularized population risk, $\texttt{id} \colon \cH \rightarrow \cH $ is the identity mapping on $\cH$, and $\cD^2 F^* \colon \cH \rightarrow \cH$ is the second-order Fr\'echet derivative of $F^*$ with respect to the RKHS norm. 
 Note that we have  $\cD L_\lambda(f^* ) = \lambda \cdot f^* $ since $\cD L (f^* ) = 0. $  
 Besides, note that the left-hand side of \eqref{eq:linearized_eq} is the first-order expansion of $\cD L_{\lambda} (f)$ at $f^*$. Thus, $\hat f_{\lambda} $ defined in \eqref{eq:def_hat_f_lam} can be viewed as an approximation of the minimizer of $L_\lambda$. Since $L_\lambda$ is the population version of the objective in \eqref{eq:estimator}, it is reasonable to expect that $  f_{n, \lambda} $ is close to $\hat f _{\lambda}$ when the sample size is sufficiently large. It remains to characterize the difference between $\hat f_{\lambda} $ and $f^*$. 
 In the following assumption, we postulate that $\hat f_{\lambda }$ converges to  $f^*$ when the regularization parameter $\lambda$  goes to zero.  
 
 \begin{assumption} \label{assume:linearize} 
 	We define $\Theta_{\lambda, 3}  > 0 $ as 
 	\# \label{theta3}
 	\Theta_{\lambda,3}=\sup_{\phi,\psi \in \cH} \sup_{u,v\in \cH }\bigl\| (\cD^2 F^* (\phi)+\lambda \cdot \texttt{id} )^{-1}\circ \bigl( \cD^3 F^* (\psi) uv \bigr)\bigr\|_{\cH}\big / \bigl(\|u\|_{\cH}\cdot \|v\|_{\cH} \bigr).
 	\# 
 	Then, we assume that $\hat f_{\lambda} $ defined in \eqref{eq:def_hat_f_lam} satisfies  $\|\hat{f}_{\lambda}-f^* \|_{\cH}\to 0$ and $\Theta_{\lambda,3}\cdot \|\hat{f}_{\lambda}-f^* \|_{\cH} \to 0$ as $\lambda $ goes to zero.
 \end{assumption}
 Compared with Assumptions \ref{assume:convexity} and \ref{assume:high-order}, Assumption \ref{assume:linearize} is much more technical and obscure. As we will see in the  proof details, $\Theta_{\lambda,3} $ naturally appears  when bounding   the difference between $\hat f_\lambda$ and the minimizer of $L_\lambda$. 
 Roughly speaking, $\hat{f}_{\lambda}$ and $\Theta_{\lambda,3}$ together characterize the bias induced  by the   regularization term $\lambda / 2 \cdot \| f \|_{\cH}^2 $ in \eqref{eq:estimator}. Hence, intuitively, Assumption \ref{assume:linearize} states that the effect of regularization shrinks to zero as $\lambda$ goes to zero and the system is stable under small perturbations. 
 A similar assumption also appears in 
 \cite{cox1990asymptotic} for the analysis of regularized maximum likelihood estimation in RKHS. 

 In addition to the above assumptions,   we  impose the following regularity condition on the RKHS kernel $K$.

 \begin{assumption} \label{ass4}
 	Recall that the kernel function $K$ of $\cH$ satisfies $K(x,y) =\la  K(x, \cdot ), K(y, \cdot ) \ra  _{\cH }$ for any $x,y \in \cX$.
 
 	For any $i, j \in \{1, \ldots, d\}$, we denote by $\partial_j K(x, \cdot)$ and $\partial^2_{ij}K(x, \cdot )$ the derivatives $\partial K(x, \cdot) / \partial x_j$ and $ \partial^2K(x, \cdot)/(\partial x_i \partial x_j)$, respectively. 
 	We assume that there exists absolute  constants $C_{K,1}  $, $C_{K,2}$, and $ C_{K, 3}$  such that  
 	\$ \sup_{x\in \cX }   \bigl\| K(x, \cdot ) \bigr\|_{\cH}  &\leq C _{K,1},\\ 
 	\sup_{x\in \cX}  \bigl\| \partial _j K(x, \cdot ) \bigr\|_{\cH}  &\leq C_{K,2} , \qquad \forall j \in \{1, \ldots, d\},\\
 	\sup_{x\in \cX}  \bigl\| \partial^2 _{ij} K(x, \cdot ) \bigr\|_{\cH}  &\leq C_{K,3} , \qquad \forall i,j \in \{1, \ldots, d\}.
 	\$   
 \end{assumption}


 Assumption \ref{ass4}  characterizes the regularity of  $\cH$. Specifically, it postulates that the kernel $K$  and its derivatives are upper bounded, which is satisfied by many popular kernels, including the Gaussian RBF kernel and the Sobolev kernel \citep{smale2003estimating, rosasco2013nonparametric, yang2016model}.
Moreover, Assumption \ref{ass4}   implies that the  embedding mapping from RKHS $\cH$ to space $\cL^{\infty}_{\nu}(\cX)$ is continuous. Indeed, for any $f \in \cH$ and any $x\in \cX$, it holds that 
 \$ 
 f(x) = \bigl\la K(x, \cdot), f \bigr\ra_{\cH} \leq \bigl\| K(x, \cdot ) \bigr\|_{\cH} \cdot \| f \|_{\cH} =  \sqrt{K(x, x)} \cdot \| f \|_{\cH} \leq C_{K,1} \cdot \| f\|_{\cH},
 \$
 which further implies that $\| f \|_{\cL_{\nu}^{\infty}} \le C_{K,1} \cdot  \|f\|_{\cH}  $ for all $f \in \cH$.  
 
 
 With Assumptions \ref{assume:convexity}-\ref{ass4}, we are ready to present the following theorem, which characterizes the statistical error of estimating $f^*$ by $f_{n, \lambda}$. 
 \begin{theorem} \label{thm:stat}
 	Under Assumptions \ref{assume:convexity}-\ref{ass4}, we set  the regularization parameter  $\lambda$ in \eqref{eq:estimator} to be
  	\$
 	\lambda  =  \cO\bigl [ C_{K,1}^{1 - \alpha _{\beta} }  \cdot \kappa^{\alpha_\beta}   \cdot \| f^* \|_{\cH ^\beta} ^{\alpha_\beta - 1 } \cdot  (\log n / n )^{(1 - \alpha _{\beta}) /2 }    \bigr ],
 	\$
 	where $\alpha_{\beta} =  (\beta-1) / (\beta+1)$ and $\cO(\cdot)$ omits absolute constants that does not depend on $\kappa$,   $C_{K,1}$, or  $\| f^* \|_{\cH^\beta}$.  
 	Then, with probability at least  $1 - n^{-2}$, we have that
 	\# \label{eq:error_bound_stat}
 	\begin{split}
 		\| f_{n,\lambda}-f^*\|_{\cH} &= \cO\bigl [  C_{K,1}^{\alpha_\beta  }   \cdot \kappa ^{- \alpha_\beta   } \cdot \| f^* \|_{\cH^\beta} ^{ 1- \alpha_\beta }  \cdot  ( \log n / n  )^{ \alpha_\beta /2 }    \bigr ], \\
 		\int_{\cX} \bigl \| \nabla  f_{n, \lambda }(x)  - \nabla f^* (x) \bigr  \|_2^2  ~\ud \PP(x)  &= \cO \bigl [ C_{K,2}^2 \cdot d \cdot C_{K,1}^{2 \alpha_\beta  } \cdot \kappa ^{- 2 \alpha_\beta   } \cdot \| f^* \|_{\cH^\beta} ^{2-2 \alpha_\beta }   \cdot  ( \log n / n  )^{ \alpha_\beta }   \bigr ] .
 	\end{split}
 	\#
 \end{theorem}
 \begin{proof}
 	See \S \ref{sec:proof_thm_stat} for a detailed proof. 
 \end{proof}
 
 As shown in Theorem \ref{thm:stat},
 when the regularization parameter $\lambda$ is properly chosen,  
 the statistical error measured by $\| f_{n, \lambda} - f^* \| _{\cH}$ converges to zero as the sample size $n$ goes to infinity with high probability. 
 Moreover, when regarding $\kappa$ and $\| f^*\|_{\cH^\beta}$ as constants, we have $\| f_{n, \lambda} - f^* \|_{\cH } = \tilde \cO(  n^{ - \alpha_{\beta}/ 2} )$ where $\tilde O(\cdot )$ omits absolute constants and logarithmic factors. 
 This implies that the statistical error converges to zero at a sublinear $n^{ - \alpha_{\beta}/ 2}$ rate. Here $\alpha_{\beta} $ increases with $\beta$, which quantifies the smoothness of $f^*$. 
 In other words, when the target function $f^*$ is smoother, we obtain a faster statistical rate of convergence.
  
 Furthermore, we remark that our statistical rate is with respect to the RKHS norm, which is  more challenging to handle than   $\| f_{n, \lambda } -f ^* \|_{\cL_{\PP}^2 } .$ 
 Under the regularity condition that  the  kernel satisfies Assumption \ref{ass4}, 
 the RKHS norm statistical rate further implies an upper bound on the estimation error of the gradient function $\nabla f^*$. 
 This result is critical for establishing the biases of the Wasserstein gradient estimates constructed by variational transport. 
 Finally,   as shown in \S\ref{sec:proof_thm_stat}, 
 by integrating the tail probabilities, 
 we can also establish in-expectation statistical rates that are similar to the 
  high-probability bounds given  in  \eqref{eq:error_bound_stat}. 
 

\subsection{Optimization Theory}
 \label{sec:opt_theory}
 In what follows, we establish the rate of convergence for the variational transport algorithm where  
  $\cF$ is set to be an  RKHS $( \cH, \la \cdot ,\cdot \ra_{\cH} ) $ with kernel $K (\cdot, \cdot)$.
  As introduced in \S\ref{sec:algo},   variational transport   creates a sequence of probability measures via  Wasserstein  gradient  descent with biased Wasserstein gradient estimates. 
  Specifically, starting from $\tilde p_0 = p_0$, we define 
  \#\label{eq:final_iters}
  \tilde p_{k+1} =\bigl [  \expm_{\cX} ( - \alpha_k \cdot  \nabla \tilde f_k^*  ) \bigr  ]_{\sharp} \tilde p_k, 
  \#
  where $p_0$ is the input of the algorithm and $\tilde f_k^*$ is defined similarly as in   \eqref{eq:def_tilde_fk}  with an RKHS norm regularization, 
  \#\label{eq:new_stat}
  \tilde {f}^*_k \leftarrow   \underset{f \in \cH }{\text{argmax}}~\biggl \{ \frac{1}{N} \sum_{i=1}^{N  } f(x _ i) - F^*(f) - \lambda \cdot \|  f \|_{\cH}^2  \biggr \} .
  \#
  Here    $\lambda> 0$ is the regularization parameter, and $\{ x_i \}_{i \in [N]}$ are i.i.d.~observations drawn from distribution $\tilde p_k$.

We remark that sampling  i.i.d.~observations and utilizing RKHS in   \eqref{eq:new_stat}  are both 
 artifacts adopted only for theoretical analysis. We present the details of such a modified algorithm in  Algorithm \ref{algo:main2} in \S\ref{sec:algos}. 
 Without these modifications, Algorithm \ref{algo:main2} reduces to the general method proposed in Algorithm \ref{algo:main},  a deterministic particle-based algorithm, which is more advisable for 
  practical implementations.

Such a  modified version of variational transport   can also be viewed as Wasserstein  gradient descent method for minimizing the functional $F$ in \eqref{eq:new_problem}. Here the bias incurred in the estimation of the Wasserstein gradient stems from the statistical error of  $\tilde f_k^*$. 
Specifically, by   Propositions \ref{prop:func_grad} and  Proposition \ref{prop:push_par}, at each $\tilde p_k$, the desired   Wasserstein  gradient descent    updates $\tilde p_k$ via 
\#\label{eq:F_gd_step}
\expm_{\tilde p_k } \bigl [  - \alpha_k \cdot  \grad F(\tilde p_k) \bigr ]   =  \bigl [  \expm_{\cX} (   - \alpha_k  \cdot  \nabla  f_k^*  ) \bigr  ]_{\sharp}    \tilde p_k,
\#
where  $f_k^*$ is the solution to the maximization problem in \eqref{eq:new_problem} with $p$ replaced by  $\tilde p_k$, and $\alpha_k$ is the stepsize. 
Comparing \eqref{eq:final_iters} with \eqref{eq:F_gd_step}, it is evident that 
the estimation error $\nabla \tilde {f}^*_k  (x)  -\nabla   f_k^*(x) $ contributes to the bias, which can be handled by the statistical analysis presented in \S\ref{sec:stat_thoery}. 

Our goal is to show that $\{ \tilde  p_k \}_{k\geq 0} $ defined in \eqref{eq:final_iters}  converges to the minimizer of $F$. We first present the a regularity condition on $F$.

\begin{assumption} \label{assume:objective}
We assume that   $F$ in \eqref{eq:new_problem} is gradient dominated in the sense that 
\#\label{eq:F_grad_dom}
\mu \cdot \bigl [ F (p) - {\textstyle \inf_{p \in \cP_2(\cX) } }F(p) \bigr  ]  \leq \la \grad F(p), \grad F(p) \ra _p 
\#
for any $p \in \cP_2(\cX)$, where $\mu>0$ is an absolute constant.
Here $\la \cdot ,\cdot \ra_p$ is the inner product on $\cT_p \cP_2(\cX)$, which is defined in \eqref{eq:riemann_metric}. 
In addition, we assume that there exist an absolute constant $L$ such that  $F$ is $L$-smooth  with respect to the Wasserstein distance. Specifically, for any $p$, $q$ in $\cP_2(\cX)$, we have 
\# 
	F(q) & \leq F(p) + \big \la  \grad F(p), \expm^{-1}_p(q) \big\ra_{p} + L/ 2 \cdot W_2^2(p, q),\label{eq:F_smooth}
	\#
where $\expm^{-1}_p (q) \in \cT_p \cP_2(\cX) $ is  the tangent vector  that specifies the geodesic connecting $p$ and $q$. 
\end{assumption} 
Assumption \ref{assume:objective} assumes that  $F$ admits the gradient  dominance and the smoothness  conditions. 
Specifically, \eqref{eq:F_grad_dom} and \eqref{eq:F_smooth} extends the classical  Polyak-\L{}ojasiewicz (PL) \citep{polyak1963gradient} and smoothness conditions    to distributional optimization, respectively. 
This  assumption holds for the KL divergence 
$F(\cdot) = \textrm{KL}(\cdot, \bar q)$ where $\bar q \in \cP_2 (\cX)$ is a fixed distribution. 
For such an objective functional, 
we have  $\grad F(p) = - \dvg [  p \cdot \nabla \log( p / q)   ] $ and \eqref{eq:F_grad_dom} reduces to 
\#\label{eq:lsi}
\mu \cdot \textrm{KL}  (p , \bar q) \leq \int _{\cX} \bigl \| \nabla \log [ p(x) / \bar q(x)] \bigr \|_2^2 \cdot p(x) ~\ud x = \EE_{p} \bigl [ \bigl \| \nabla \log (p / \bar q)\bigr \| _2^2 \bigr ]  = \mathrm{I} (p , \bar q),
\# 
where $\mathrm{I} (p , \bar q)$ is known as the relative Fisher information between $p$ and $\bar q$. 
Inequality \eqref{eq:lsi} corresponds to the 
  logarithmic Sobolev inequality \citep{gross1975logarithmic, otto2000generalization, cryan2019modified},
  which holds for   $\bar q$ being the density of a   large class of distributions, e.g., Gaussian and log-concave distributions.
   In addition, 
   the smoothness condition in \eqref{eq:F_smooth} is a natural extension  of that for Euclidean space  to  $\cP_2(\cX)$, 
   where the Euclidean distance is replaced by $W_2$. 

The following assumption characterizes the regularity of the solution to the inner optimization problem associated with the variational representation of the objective functional.
 \begin{assumption} \label{assume:f_star} 
 	For any $p \in \cP_2(\cX)$, let $f_p^* \in \cH$ denote  
 	the solution to the variational   problem in \eqref{eq:new_problem}  used to define $F$.
 	We assume that 
 there exists a constant $R > 0$ and $\beta \in (1,2 )$ such that  $f_p^*$ 
 belongs to $\{f \in \cH^\beta \given \|f\|_{\cH^\beta} \le R\}$ for all $p \in \cP_2(\cX)$,  where $\cH^ \beta$ is the $\beta$-power space.  
 \end{assumption} 
 
 This assumption postulates that the target function $f_p^*$ belongs to an RKHS norm ball of the $\beta$-power space $\cH^\beta$. 
 Such an assumption, similar to Assumption \ref{assume:high-order} in \S\ref{sec:stat_thoery},  ensures that the variational problem for defining $F$ is well-conditioned in the sense that its solution lies in the $\beta$-power space with uniformly bounded RKHS norm in $\cH^\beta$. 
 
 We are now ready to present our main result.
\begin{theorem}[Convergence of Variational Transport]
	\label{thm:comp}
	Let $\tilde p_k$ be the iterates defined in \eqref{eq:final_iters}.
	 Suppose that Assumptions \ref{assume:convexity}, \ref{assume:linearize}, \ref{ass4},  \ref{assume:objective}, and \ref{assume:f_star} hold, where constants $\mu$ and $L$ in Assumption \ref{assume:objective} satisfy  $0 < \mu \leq L$.
	In variational transport, we set   stepsize $\alpha_t$ to be  a constant $\alpha$, which satisfies 
	\#\label{eq:set_stepsize}
	\alpha \in \Bigl (0, ~\min \bigl \{ 1/(2L), 1/ H_0 \bigr \} \Bigr ],
	\#
	where
	we define $H_0 = 2d  \cdot C_{K, 3} \cdot \mu_{\max}^{(\beta - 1)/2} \cdot R$. 
	Here
	 $\mu_{\max} = \max_{i \geq 1}  \mu_i$ is the largest eigenvalue of the integral operator $\cT$ of the RKHS. 
	Moreover, 
	we set the number of particles $N$ to be sufficiently large such that $N > K$, where $K$ is the number of iterations. 
  Then, with probability at least $1 - N^{-1}$, it holds for any $k \le K$ that 
	\#\label{eq:convergence_rate}
	  F(\tilde p_k)     -   {\textstyle \inf_{p \in \cP_2(\cX) } }F(p)	   \leq \rho^k \cdot [ F(\tilde p_0)  - {\textstyle \inf_{p \in \cP_2(\cX) } }F(p)  ]  + (1-\rho)^{-1} \cdot {\rm Err}
	\# 
	where we define  
	$\rho = 1 - \alpha \cdot \mu /2  $ and 
	\begin{align}\label{eq:def-err}
	{\rm Err} =   \cO \Bigl [\alpha\cdot C_{K,2}^2 \cdot d \cdot C_{K,1}^{2 \alpha_\beta  } \cdot \kappa ^{- 2 \alpha_\beta   }\cdot R^{2-2 \alpha_\beta }   \cdot  ( \log N / N  )^{ \alpha_\beta }   \Bigr ].
	\end{align}
	Here $\alpha_\beta = (\beta - 1) / (\beta + 1)$ and $\cO(\cdot )$ omits absolute constants.

\end{theorem}

\begin{proof}
See \S\ref{sec:proof:convergence} for a detailed proof.
\end{proof} 
Theorem \ref{thm:comp} proves that, with high probability,  variational transport converges linearly to the global minimum of the objective functional $F$ 
up to an error   term ${\rm Err}$ defined in \eqref{eq:def-err}. 
In particular, \eqref{eq:convergence_rate} shows that the suboptimality of $\tilde p_k$ is bounded by the sum of two terms, which correspond  to the computational and statistical errors, respectively. 
Specifically, 
the first term  converges to zero at a linear rate as $k$ increases, which characterizes the convergence rate of the population version of Wasserstein gradient descent. Meanwhile, the second term  is independent of $k$ and  characterizes the statistical error incurred in estimating the Wasserstein gradient direction using $N$ particles. 
When $N$ and $k$ are sufficiently large, the right-hand size on  \eqref{eq:convergence_rate} is  dominated by the statistical error $(1- \rho)^{-1} \cdot {\rm Err}$, which decays to zero as $N$ goes to infinity. 
In other words, as the number of particles and the number of iterations both go to infinity, variational transport finds the global minimum   of $F$. 

Furthermore, the constant  stepsize $\alpha$ in \eqref{eq:set_stepsize} is determined by  (i) the smoothness parameter $L$ of the objective functional and (ii)   $H_0 =2d  \cdot C_{K, 3} \cdot \mu_{\max}^{(\beta - 1)/2} \cdot R$. 
Specifically, the requirement that $\alpha < 1/ (2L)$ is standard in the optimization literature, which guarantees that each step of Wasserstein gradient descent  decreases the objective functional sufficiently. 
Whereas  $H_0$ serves as an upper bound on the Lipschitz constant of $\nabla \tilde f_k^*$, 
which also enforces an upper bound on the stepsize due to  Proposition \ref{prop:push_par}. 
Under the assumptions made in Theorem \ref{thm:comp}, both $L$ and $H_0$ are regarded as constants. 
As a result, with a constant stepsize $\alpha$, $\rho$ is a constant in $(0, 1)$, which further implies that  variational transport converges linearly up to a certain statistical error. 

Finally,  neglecting constants such as $d$, $\{C_{k,i}\}_{i \in [3]}$, $\kappa$, $R$, $\mu$, and $L$,  when both $N$  and $k$ are sufficiently large, variational transport attains an $\tilde \cO(N^{-\alpha_{\beta} }  )$ error, which reflects the statistical error incurred in solving the inner variational problem in \eqref{eq:new_problem}. 
Here $\tilde \cO(\cdot )$ hides absolute constants and logarithmic terms. 
The parameter $\alpha_\beta$ increases with  $\beta$. 
That is, when the target functions under  Assumption \ref{assume:f_star} belong  to a   smoother RKHS class, variational transport attains a smaller statistical error.

\section{Conclusion}

We study the distributional optimization problem where the objective functional admits a variational form. 
For such a problem, we propose a particle-based algorithm, dubbed as variational  transport, 
which approximately performs   Wasserstein gradient descent  
using  a set of particles. 
Specifically, in each iteration of variational transport, 
we first solve the inner variational problem associated with the objective functional using the particles. 
The solution to such a problem yields the Wasserstein gradient direction approximately. 
Then we update the current distribution by pushing each particle along the direction specified by the solution to the inner variational problem. 
Furthermore, to analyze the convergence rate of variational transport, 
we  consider the setting where the objective functional satisfies a functional version of the of the  Polyak-\L{}ojasiewicz (PL) \citep{polyak1963gradient} and smoothness conditions,    
and the inner variational problem is solved within an RKHS. 
For such an instantiation, under proper assumptions, 
we prove that variational transport constructs a sequence of probability distributions that converges linearly to the global minimizer of the objective functional up  to a  statistical error due to estimating the Wasserstein gradient with finite particles. Moreover, such a statistical error converges to zero as the number of particles goes to infinity.

\section{Acknowledgement}
Zhaoran Wang acknowledges National Science Foundation (Awards 2048075, 2008827, 2015568, 1934931), Simons Institute (Theory of Reinforcement Learning), Amazon, J.P. Morgan, and Two Sigma for their supports.
 
\bibliographystyle{ims}
\bibliography{particle}
\clearpage

\appendix{}

\section{Variational Transport  with Random Sampling}\label{sec:algos}

In this section, we present a version of the variational  transport algorithm where we sample a batch of $N$ i.i.d. particles in each iteration. 
Specifically, this version of variational transport maintains a sequence of transportation plans $\{ T_k\}_{k\geq 0}$ such that $\tilde p_k = (T_k)_{\sharp} p_0$ for all $k \geq 1$, where $\tilde p_k$ is given in \eqref{eq:final_iters}.
Then, in the $k$-th iteration, by drawing $N$ i.i.d. observations from $p_0$ and applying transformation $T_k$ to these observations, we obtain $N$ i.i.d. observations from $\tilde p_k$. 
This enables us to establish the statistical error incurred in solving the inner variational problem in \eqref{eq_def_of_widetilde_f^*_k}, which further yields the 
an upper bound on the 
estimation error of the Wasserstein gradient. 
We remark that Algorithm \ref{algo:main2} is considered only for the sake of theoretical analysis; random sampling is unnecessary in practical implementation.   
 

\begin{algorithm} [htbp]
	\caption{Variational Transport Algorithm with Random Sampling}  
	\label{algo:main2} 
	\begin{algorithmic}[1] 
		\STATE{{\textbf{Input:} Functional $F \colon \cP_2(\cX) \rightarrow \RR $ defined in \eqref{eq:var_func},  initial point $p_0 \in \cP_2(\cX)$, number of particles $N$, number of iterations $K$, and stepsizes $\{ \alpha_k\}_{k=0}^K$.}}  
		\STATE{Initialize the transportation plan $T_0\leftarrow \id $.}
		\FOR{$k = 0, 1,2, \ldots, K$}
		\STATE{Generate
			$N$ particles  $\{ x_i ^{(k)}\}_{i\in [N]}$  by drawing  $N$ i.i.d.  observations from $p_0$.}
		\STATE{Push the particles
			$N$ particles by letting  $x_i\leftarrow  T_k(x_i^{(k)}) $ for all $i\in [N]$.}
		\STATE{Compute $f^*_k \in \cF$ via \# 
			\label{eq_def_of_widetilde_f^*_k}
			\tilde {f}^*_k \leftarrow   \underset{f \in   {\mathcal{F}}}{\text{argmax}}~\biggl \{ \frac{1}{N} \sum_{i=1}^{N  } f(x _ i) - F^*(f) \biggr \}.
			\#} \label{line:1}
		\STATE{Update the transportation plan by letting $T_{k+1} = [ \expm_{\cX}( - \alpha_k \cdot \nabla \tilde f_k^* )] \circ T_{k}$.}
		\ENDFOR
		\STATE{{\textbf{Output:}} The final transportation plan  $T_{K+1}$.}
	\end{algorithmic}
\end{algorithm}

%

\section{Additional Background}

In this section, we introduce some additional background knowledge that is related to the theory of variational transport. 
\subsection{Wasserstein Torus $\cP_2(\TT^d)$} \label{sec:torus} 
In this section, we introduce the Wasserstein torus $\cP(\TT^d)$. We first give a characterization of the torus $\TT^d$. We define a equivalence relation $\sim$ on $\RR^d$ as follows
\begin{align*}
x \sim x',\quad \text{if and only if } x - x' \in \ZZ^d.
\end{align*}
Then, the torus $\TT^d = \RR^d / \sim$ is the quotient space. For each $x \in \RR^d$, there is a unique $\bar x \in [0, 1)^d$ that is equivalent to $x$, which is defined by $\bar x = (x_1 - \lfloor x_1 \rfloor, \dots, x_d - \lfloor x_d \rfloor )^\top $. To simplify the notation, we  denote by $[x] = \{x'\given x'\sim x\} \in \TT^d$ the equivalence class of $x$ and use $\bar x \in [0,1)^d$ to identify $[x]$. We endow $\TT^d$ with the distance $\|[x]- [y]\|_{\TT^d}$ for any $[x], [y] \in \TT^d$, which is defined as follows,
\begin{align*}
\bigl\|[x]- [y] \bigr\|_{\TT^d} = \min_{a\in [x], b \in [y]} \|a - b \|.
\end{align*}
Note that $(\TT^d, \|\cdot  \|_{\TT^d})$ is a compact space.
We say a function $f: \RR^d \rightarrow \RR$ is on $\TT^d$ if and only if $f (x) = f(x'), \forall x \sim x'$, namely, $f $ is a periodic function. 
On the other hand, any function $f \colon \TT^d \rightarrow \RR$ 
can be extended periodically to $\RR^d$ 
by letting $f(x) = f( [x])$ for any $x\in \RR^d$. 
To characterize the probability measure space $\cP_2(\TT^d)$, we define a equivalence relation on $\cP_2(\RR^d)$ as follows 
\begin{align}\label{eq:dist_equivalence} 
\mu \sim \nu \quad \text{if and only if } \int_{\RR^d}  f ~\ud \mu = \int_{\RR^d} f ~ \ud \nu, \qquad \forall f \in \cC(\TT^d),
\end{align}
where $\cC(\TT^d)$ is the continuous function class on $\TT^d$. 
As a concrete example,   any two multivariate Gaussian distributions 
$N(\mu_1, \Sigma_1)$ and $N(\mu_2, \Sigma_2) $ are equivalent if and only if $\mu_1 - \mu_2 \in \ZZ^d$ and $\Sigma_1 = \Sigma_2$.
Furthermore, let $[\mu]$ denote the    equivalence class
of $\mu \in \cP_2(\RR^d)$.  
We define $\cP_2(\TT^d)$ as  the collection of all equivalence classes.   
In other words, we have $\cP_2(\TT^d) = \cP_2(\RR^d) / \sim$, which  is a quotient space $\cP_2(\RR^d)$.
For any $\mu \in \cP_2(\RR^d)$,  
there is a measure $\bar \mu$ supported on $[0,1)^d$ that lies in the equivalence class $[\mu]$. 
Specifically, let $p_{\mu}$ be the density of $\mu$, then the density of $\bar \mu$ is given by 
\#\label{eq:wrap_dist}
p_{\bar \mu} (x) = \sum_{k \in \ZZ^d} p_{\mu} (x + k).
\#
In particular, when $\mu$ is a Gaussian distribution, $\bar \mu$ defined in \eqref{eq:wrap_dist} is known as the wrapped Gaussian distribution on the torus.

Furthermore, if a functional $F \colon \cP_2(\RR^d) \rightarrow \RR $ is invariant to the equivalence relation in  \eqref{eq:dist_equivalence} in the sense that $F(\mu) = F(\nu)$ whenever $\mu \sim \nu$, then we regard $F$ as a functional on $\cP_2(\TT^d)$.
Similarly to \S\ref{sec:wasserstein_space}, we define the second-order  Wasserstein distance for $\mu, \nu \in \cP_2(\TT^d)$ as follows,
\begin{align*}
W_2(\mu, \nu) = \biggl [ \inf_{\pi \in \Pi(\mu, \nu)}  \int_{\TT^{d}\times \TT^{d} }  \bigl\|x- y \bigr\|_{\TT^d}~ \ud   \pi(x,y)  \biggr ]^{1/2 },
\end{align*}
where $\Pi(\mu, \nu)$ consists of all probability measures on $\TT^d \times \TT^d$ with marginals $\mu$ and $\nu$. Note that $(\cP(\TT^d), W_2)$ is compact \citep{gangbo2014weak}.

\subsection{Sobolev Space and Reproducing Kernel Hilbert Space} \label{sb}
To complement the brief introduction of RKHS in 
\S\ref{bg:rkhs}, in the following, 
we introduce the Sobolev space, which is  commonly utilized  in nonparametric statistics, and introduce its connection to RKHS. 
For a differentiable  function $f\colon \cX\rightarrow \RR$  defined on  $\cX \subseteq \RR^n$ and a multi-index $\balpha = (\alpha_1, \ldots, \alpha_n)$, the mixed partial derivative of $f$ is defined as
\$
D^{\balpha}f=\frac{\partial^{\|\balpha\|_1}f }{\partial^{\alpha_1}{x_1}\cdots\partial^{\alpha_n}{x_n}}.
\$  
The $(k,p)$-th order Sobolev norm of a function $f$ is defined as follows,
\$
\| f\|_{\cW^{k,p}}= \biggl[\sum_{\balpha:\|\balpha\|_1\le k }  \int_{\cX} |D^{\balpha}f(x)|^p~ \ud x \biggr]^{1/p},
\$ 
where we let $\ud x$ denote the    Lebesgue measure. Sometimes, what we need is integration with respect to a general probability measure $\nu$. In this case, we define the weighted $(k,p)$-th order Sobolev norm by replacing the integration with respect to $\nu$, namely,  
\$
\| f\|_{\cW_\nu ^{k,p} }= \bigg[\sum_{\balpha:\|\balpha\|_1\le k }  \int_{\cX} |D^{\balpha}f(x)|^p~\ud \nu(x) \bigg]^{1/p}.
\$ 
All functions with finite (weighted) $(k,p)$-th order Sobolev norm consists the (weighted) $(k,p)$-th order Sobolev space, which is denoted by $\cW^{k,p}(\cX)$ ($\cW_{\nu}^{k,p}(\cX)$).
Note that the (weighted) Sobolev space is a Banach space and  when $p=2$, it is also a Hilbert space with the following inner product 
\$
\la f,g\ra_{\cW_{\nu} ^{k,p}}=\sum_{\balpha:\|\balpha\|_1\le k }  \int_{\cX} D^{\balpha}f(x)\cdot D^{\balpha}g(x)~\ud \nu(x) .
\$ 

In the sequel, we aim to show that $\cW^{k,2}(\cX)$ is also an RKHS when $k > d/2$.  To  this end, we first introduce an equivalent definition of RKHS without using an explicit reproducing kernel.

\begin{definition}[Implicit Definition of RKHS] \label{def:rkhs2} 
	A Hilbert space $(\mathcal{H}, \la \cdot,\cdot \ra_{\mathcal{H}})$ is an RKHS if the evaluation  functionals is continuous on $\cH$. That is, for any $x \in \cX$,   there exists a constant $M_x>0$ such that   
	$
	|f(x)|\le M_x \cdot  \| f\|_{\mathcal{H}}
	$ for all $f\in \mathcal{H}$.
\end{definition}

The above definition of RKHS  does not require an explicit construction of the  kernel function   and hence is easier to check in practice. Moreover, as shown in Theorem 1 in 
\cite{berlinet2011reproducing},  $(\mathcal{H}, \la \cdot,\cdot \ra_{\mathcal{H}})$ has a reproducing kernel if and only if the evaluation functionals are  continuous on $\cH$. Thus, Definition \ref{def:rkhs2} is equivalent to the one   given in \S\ref{bg:rkhs}. 
Utilizing the equivalent definition, we have  the following theorem. 

\begin{theorem} \label{thm0} 
	For any $k > d/2$, the  $(k,2)$-th order Sobolev space $\cW^{k,2}(\cX)$ defined on a bounded domain $\cX \subseteq \RR^d $  is an RKHS.  
\end{theorem}
\begin{proof}
	To prove that $\cW^{k,2}(\cX)$ is an RKHS, by Definition \ref{def:rkhs2},   it suffices to show that there exists a constant $C > 0$ such that 
	$
	\sup_{x\in \cX} | f(x) |  \leq C \cdot \| f \|_{ \cW^{k,2}}.  
	$
	To this end, we first  introduce the notion of the H\"older continuous function and the H\"older  space.  See, e.g.,  Chapter 5 of \cite{evans2010partial} for details.

	\begin{definition}[H\"older space] 
		\label{def:holder_space}
		Let $\cX$ be a bounded domain in $\RR^d$ and $\alpha>0$ be a  constant.  A function $f\colon \cX \rightarrow \RR$ is called $\alpha$-H\"older continuous, if there exists a constant $C>0$ such that 
		$
		|	f(x) - f(y) |  \leq C \cdot \| x - y \|_2 ^\alpha 
		$
		for all $x, y \in \cX$. 
		Moreover, we define 
		\#\label{eq:holder_norm1}
		\| f\|_{\cC^{0, \alpha}}  = \sup_{x \neq y \in \cX}  \frac{ |f(x) - f(y) | }{ \| x - y \|_2^\alpha}
		\# 
		as the H\"older coefficient of $f$.
		
		For any integer $k \geq 0$, the H\"older space $\cC^{k, \alpha} ( \cX)  $ contains all functions with  continuous derivatives up to order $k$ and the $k$-th partial derivatives are $\alpha$-H\"older continuous. Specifically, we define the H\"older norm of $f$ as 
		\#\label{eq:holder_norm}
		\| f\|_{\cC^{k, \alpha}}  = \max _{\balpha  \colon \| \balpha\|_1 \leq k } \sup_{x \in \cX }  \bigl | D^{\balpha } f(x) \bigr | + \max_{ \balpha  \colon \| \balpha\|_1 =  k} \bigl \| D^{\balpha } f \bigr \|_{\cC^{0, \alpha}},
		\#
		where the maximization in \eqref{eq:holder_norm} is taken over all multi-indices and $\| \cdot \|_{\cC^{0, \alpha}}$ is defined in \eqref{eq:holder_norm1}.  Then, we define the H\"older space  $\cC^{k, \alpha} ( \cX)  $ as 
		\$
		\cC^{k, \alpha} ( \cX)   = 	 \bigl \{ f \colon \cX \rightarrow  \RR  \colon  \| f\|_{\cC^{k, \alpha}} < \infty  \bigr \}. 
		\$
		Besides, it is known that $\cC^{k, \alpha} (\cX)$ equipped with norm $\| \cdot \|_{\cC^{k, \alpha}}$ is a Banach space.
	\end{definition}
	
	Our proof is based on the   Sobolev embedding inequality, which specifies a sufficient condition for the   Sobolev space $\cW^{k, p}(\cX)$ to be contained in a H\"older space. Formally, we have the following lemma.

	\begin{lemma}[Sobolev Embedding Inequality] \label{sei}
		 Let $\cX$ be a bounded domain in $\RR^d$. 
		Consider the Sobolev space $\cW^{k,p}(\cX)$. We let    $\gamma = \lfloor d/p \rfloor +1-d/p$ if $d/  p $ is not an integer, and let $\gamma$ be an arbitrary  number in $(0,1)
		$ if $ d/p$ is an integer. 
		Then, if $k  > d / p$, we have $ \cW^{k,p}(\cX) \subseteq \cC^{m, \gamma} (\cX)$, where $m = k -  \lfloor d/ p \rfloor  - 1$. More specifically, there exists a constant $C$ such that for all $u \in \cW^{k,p}(\cX)$,  it holds that 
		\#\label{eq:sei}
		\| u\|_{\cC^{m,\gamma}} \le  C\cdot \|u\|_{\cW^{k,p}},
		\#
		where $C$ depends only on $k, d, p, \gamma$ and $\cX$, and $\| \cdot \|_{\cC^{m, \gamma}}$ and $\| \cdot \|_{\cW^{k, p}}$ are the H\"older  and Sobolev norms, respectively.
	\end{lemma}

	Lemma \ref{sei} is a standard result in literature of partial differential equations. See Theorem 6 in Chapter 5 of  \cite{evans2010partial} for a  detailed proof. 
	Applying 
	Lemma \ref{sei}
	with $ p =2$ and $ k > d/2$,  we obtain that 
	\# \label{eq0}
	\sup_{x\in \cX} |f(x)|  \leq C\cdot \|u\|_{\cW^{k,2}} 
	\#   
	for any $f \in \cW^{k,2}(\cX)$, where $C$ does not depends on  the choice of $f$.
	Finally, recall that the
	Sobolev space $\cW^{k,2}(\cX)$ is also a Hilbert space. By Definition \ref{def:rkhs2}, $  \cW^{k,2}(\cX) $ is an RKHS, which completes the proof of Theorem \ref{thm0}. 
\end{proof}
Similar result holds for the weighted Sobolev space, which is established in the following corollary. 
\begin{corollary} \label{cor0}For the $(k,2)$-th order weighted Sobolev space $\cW_{\nu}^{k,2}(\cX)$ defined on a bounded domain $\cX \subseteq \RR^d$. We assume that  the measure $ \nu $ has a density function $\nu(x)$ with respect to the Lebesgue measure and $\nu(x)$ is lower bounded by some constant $\kappa>0$. Then    $\cW_{\nu}^{k,2}(\cX)$ is  an RKHS if   $k>n/2$.
	\begin{proof}
		By the definition of weighted Sobolev space, $\cW_{\nu}^{k,2}(\cX)$ is a Hilbert space. Similar to the proof of Theorem \ref{thm0}, since $\nu(x) $ is lower bounded by $\kappa $,  for any $f\in \cW^{k, 2}_{\nu}$, we have 
		\$
		\sup_{x\in \cX} |f(x)| \le C\cdot \|u\|_{\cW^{k,2}}\le C/\kappa \cdot \|u\|_{\cW^{k,2}_{\nu}},
		\$
		where the second inequality follows from the definition of weighted Sobolev norm $  \|\cdot\|_{\cW^{k,2}_{\nu}}$. Therefore,  by Definition \ref{def:rkhs2},  $ \cW_{\nu}^{k,2}(\cX) $ is an RKHS.
	\end{proof}
\end{corollary}

As an  RKHS, the kernel function of Sobolev space is bounded \citep{novak2018reproducing}, which satisfies Assumption \ref{ass4}.

\subsection{Stochastic Gradient Descent on Riemannian Manifold}

In this section, we examine  the convergence of  stochastic (Riemannian) gradient descent on a  general Riemannian manifold  $(\cM, g)$,
which serves as the foundation of the analysis of variational transport. 
Throughout this subsection,  we assume that any two points on $\cM$ uniquely determine a geodesic.
Let $\| \cdot   \|$ be the geodesic distance on $\cM$.     Before we present the convergence result,    we first extend the classical optimization concepts   to Riemannian manifolds.  The definitions in this section  can also be found in the literature  on geodesically convex optimization. See, e.g., \cite{zhang2016first,zhang2016riemannian,liu2017accelerated,zhang2018estimate} and the references therein. 

\begin{definition} [Geodesical Convexity]
A function $f \colon  \cM \to \RR$ is called  geodesically convex  if for any $x, y \in \cM$ and a  geodesic $\gamma \colon [0,1] \rightarrow \cM$ such that $\gamma (0) = x$ and $\gamma(1) = y$, we have 
\#\label{eq:gconv1}
	f [ \gamma(t) ] \leq t \cdot f [ \gamma(0) ]  + (1-t) \cdot f [ \gamma(1) ] , \qquad \forall t \in [0, 1].
\#
\end{definition}

The following lemma characterizes the geodesical convexity 
based on the gradient of $f$.

\begin{lemma} \label{lemma:convex_riemann}
A differentiable function $f: \cM \to \RR$ is geodesically convex if and only if 
\#\label{eq:gconv2}
	f(y) \geq f(x) + \bigl \la \grad f(x), \expm^{-1}_x(y)\bigr \ra_x , \qquad  \forall x, y \in \cM.
\#
\end{lemma}
\begin{proof}
For any $x, y \in \cM$,  
  let $\gamma: [0, 1] \to \cM$ be the unique geodesic satisfying 
$ \gamma(0)   = x$ and $
	\gamma(1) = y$. 
By the definition of the exponential mapping, we have $\expm _x [ \gamma' (0)] = y$ and
   $ \gamma '(0) = \expm^{-1}_x(y)$. 
  In addition, \eqref{eq:gconv1} shows that 
 $f[ \gamma(t)] $ is a convex and differentiable function with respect to $t$, which further  implies  
 \#\label{eq:gconv3}
	f \bigl[ \gamma(1) \bigr]  \geq f \bigl[ \gamma(0) \bigr] +  \frac{\ud }{\ud t}  f\bigl[ \gamma(t) \bigr]  \biggr |_{t=0}.
  \#
By the definition of the directional derivative, we have  
\#\label{eq:gconv4}
\frac{\ud }{\ud t}  f\bigl[ \gamma(t) \bigr]  \biggr |_{t=0}  = \nabla_{\gamma '(0)} f (x) = \bigl \la \grad  f(x),  \gamma' (0) \bigr \ra _x = \bigl \la \grad f(x), \expm^{-1}_x(y) \bigr \ra _x.
\#
Thus, combining \eqref{eq:gconv3} and \eqref{eq:gconv4}, we obtain \eqref{eq:gconv2} of Lemma \ref{lemma:convex_riemann}.

It remains to show that  \eqref{eq:gconv2} implies \eqref{eq:gconv1}. 
For any geodesic $\gamma: [0, 1] \to \cM$, we aim to show that $f[ \gamma(t)] $ is  a convex function of    $t$. To see this, for any  $0 \leq t_1 \leq t_2 \leq 1$, let  $x = \gamma(t_1)$ and $y = \gamma(t_2)$.
Note that we can reparametrize $\gamma$ to obtain a new geodesic $\widehat{\gamma}$ with $\widehat{\gamma}(0) = x$ and $\widehat{\gamma}(1) = y$ by letting 
\$
	\widehat{\gamma}(t) = \gamma\bigl( t_1 + (t_2 - t_1)\cdot t \bigr)
\$
for any $t\in [0,1]$. Since $\hat \gamma$ is a geodesic, by the definition of the exponential map, we have 
\$	\expm^{-1}_x(y) = \hat{\gamma}'(0) = (t_2 - t_1)\cdot \gamma' (t_1) .
\$
Thus, by 
  \eqref{eq:gconv2}, we have
\$
  &	f\bigl[ \gamma(t_2) \bigr]  =  f(y)  \geq f(x) + \bigl \la \grad f(x), \expm^{-1}_x(y) \bigr \ra _x \\
 & \qquad =   f \bigl[  \gamma(t_1) \bigr]  + \Bigl \la \grad f \bigl[ \gamma(t_1) \bigr] , (t_2 - t_1)\cdot \gamma' (t_1)\Bigr \ra _{\gamma(t_1)}   =  f \bigl[  \gamma(t_1) \bigr] + (t_2 - t_1)  \cdot \frac{\ud}{\ud t}  f \bigl[  \gamma(t) \bigr] \bigg|_{t = t_1},
\$
which implies that $f [ \gamma(t) ] $ is a convex function of $t$. Thus, \eqref{eq:gconv1} holds, which completes the proof.
\end{proof}

In the following, we extend the concepts of strong convexity and smoothness to Riemannian manifolds.

\begin{definition} [Geodesically Strong Convexity and Smoothness]
	For any $\mu >0$, 
a differentiable   function $f \colon \cM \to \RR$ is called  geodesically $\mu$-strongly convex  if
\begin{align*}
	f(y) \geq f(x) + \bigl \la \grad f(x), \expm^{-1}_x(y)  \bigr \ra_x+ \mu /2 \cdot  \|x - y \|^2.
\end{align*}
Function $f$ is geodesically $L$-smooth if $\grad f$ is $L$-Lipschitz continuous. That is, for any $x, y \in \cM$, we have 
	\# \label{eq:grad_lip}
	\Bigl \la \grad f(x) - \Gamma_y^x \bigl[ \grad f(y) \bigr] , \grad f(x) - \Gamma_y^x \bigl[ \grad f(y) \bigr] \Bigr \ra _x \leq L^2 \cdot \| x - y \|^2  ,
\#
where $\Gamma_y^x: \cT_y\cM \to \cT_x\cM$ is the parallel transport.
\end{definition}
 Note that we apply the parallel transport in \eqref{eq:grad_lip} to compare $\grad f(x)$ and $\grad f(y)$, which belong to two different tangent spaces. In the following, we introduce the  notion of gradient dominated function. 
 
\begin{definition} [Gradient Dominance]
	Let  $\mu> 0$ and $f \colon \cM   \rightarrow \RR$ be 
	a  differentiable function with $f^* = \min_{x \in \cM} f(x)$. 
	Function  $f$   
	 is   called $\mu$-gradient dominated if 
	 \#
	 	\label{eq:grad_dom}
	 2 \mu \cdot  [ f(x) - f^* ]  \leq \bigl \la \grad f(x), \grad f(x) \bigr \ra _x , \qquad  \forall x \in \cM.
	 \#
	\end{definition}

In the following lemma,  similarly to functions  on  the Euclidean space, we show that $\grad f$ being  Lipschitz smooth  implies that $f$ can be upper bounded by the distance function. Meanwhile, we prove that gradient dominance is implied by geodesically strong convexity  and  is thus  a weaker condition.

\begin{lemma} 	\label{lem_descent_lemma}
If $f: \cM \to \RR$  is geodesically  $\mu$-strongly convex, then $f$ is also $\mu$-gradient dominated. 
In addition, if $f$ has $L$-Lipschitz continuous gradient, then we have 
 \#\label{eq:smoothness}
 	f(y) \leq f(x) + \big \la  \grad f(x), \expm^{-1}_x(y) \big\ra _x + L / 2 \cdot \| x - y\|^2 , \qquad  \forall x,y \in \cM.
 \# 
\end{lemma}
\begin{proof} 
For the first part, let $f$ be a geodesically $\mu$-strongly convex function. Since $(\cM, g)$ is a geodesic space, we have $ 	\| x - y\|^2 = \la \expm^{-1}_x(y), \expm^{-1}_x(y) \ra_x$. Thus, by direct computation, we have 
\#\label{eq:grad11}
	f(y) & \geq  f(x) + \bigl \la  \grad f(x), \expm^{-1}_x(y) \bigr \ra_x +  \mu /2 \cdot \bigl  \la \expm^{-1}_x(y), \expm^{-1}_x(y) \bigr \ra_x  \notag \\
	 & =    f(x) + \mu /2 \cdot  \Bigl\la  \expm^{-1}_x(y) +1 / \mu \cdot  \grad f(x), \expm^{-1}_x(y) + 1 / \mu \cdot  \grad f(x) \Bigr \ra _x \notag \\
	 &\qquad - 1 / (2\mu) \cdot  \bigl \la  \grad f(x),  \grad f(x) \bigr \ra_x \notag  \\
	 & \geq    f(x) - 1 / (2\mu) \cdot  \bigl \la  \grad f(x),  \grad f(x) \bigr \ra_x. 
\#
By setting $y = x^*$ such that $f(x^*) = f^*$ in \eqref{eq:grad11}, we establish \eqref{eq:grad_dom}. 
 
To establish  the second part of Lemma \ref{lem_descent_lemma}, 
 for any $x, y \in \cM$, let $\gamma$ be the unique geodesic satisfying $\gamma(0) = x$ and $\gamma(1) = y$. Then, we have $\expm_x^{-1} (y) = \gamma' (0)$.
 Moreover, for any $t\in [0, 1]$, note that $\gamma'(t)\in \cT_{\gamma(t) } \cM$. By the definition of the parallel transport,   we have 
 $\Gamma _{\gamma(t)}^x \gamma'(t) = \gamma'(0) = \expm_x^{-1} (y).$
   Thus, by \eqref{eq:gconv4} it holds that  
 \#\label{eq:grad13}
 & f(y) - f(x) - \bigl \la  \grad f(x), \expm^{-1}_x(y) \bigr \ra_x = f[ \gamma(1) ] - f[\gamma(0)] -\frac{\ud 	}{\ud t} f[ \gamma(t)] \biggr | _{t= 0} \notag \\
 & \qquad  = \int_0 ^1   \bigl \{ \bigl \la \grad f[ \gamma(t) ], \gamma' (t) \bigr \ra _{\gamma(t) } - \bigl \la \grad f(x), \gamma'(0) \bigr \ra _{x} \bigr \} \ud t  \notag \\
 & \qquad = \int_0 ^1   \Bigl (   \bigl \la \Gamma_{\gamma(t)}^x \bigl\{ \grad f[ \gamma(t) ]\bigr \}  - \grad f(x) , \expm_x^{-1} (y)   \bigr \ra _{x} \Bigr ) \ud t,
 \# 
 where in the last equality we transport the tangent vectors  in $\cT_{\gamma(t) } \cM$ to $\cT_x \cM$.  Besides,  
 by \eqref{eq:grad_lip} and the Cauchy-Schwarz inequality, for any $z \in \cM$, we have 
 \#\label{eq:grad12}
&  \Bigl | \bigl \la \grad f(x) - \Gamma_z^x [ \grad f(z) ] , \expm_{x}^{-1}(y) \bigr \ra _{x} \Bigr | \notag \\
 & \qquad \leq   \Bigl \{\bigl \la \grad f(x) - \Gamma_z^x [ \grad f(z) ] ,\grad f(x) - \Gamma_z^x [ \grad f(z) ]  \bigr \ra_x  \Bigr \}^{1/2}  \cdot \Bigl [\bigl \la \expm_{x}^{-1}(y) ,  \expm_{x}^{-1}(y) \bigr \ra _{x} \Bigr]^{1/2}\notag \\
 &\qquad \leq 
  L  \cdot \| x- y\| \cdot \| z - x \|. 
 \#
 Finally, combining \eqref{eq:grad13} and \eqref{eq:grad12}, we have 
 \$
 f(y) - f(x) - \bigl \la  \grad f(x), \expm^{-1}_x(y) \bigr \ra_x \leq L \cdot \int_0 ^1 \|x - y \| \cdot \bigl\|\gamma(t)- x  \bigr\| \ud t = L/ 2\cdot \| x- y \|^2,
 \$
 where the last equality follows from the fact that $\| \gamma(t) - x  \| = \|\gamma(t)- \gamma(0)\|  = t \cdot \| x - y\|$. Therefore, we establish \eqref{eq:smoothness} and complete the proof of Lemma \ref{lem_descent_lemma}.
\end{proof}


Now we are ready to study the convergence of stochastic gradient descent 
for optimization on $(\cM, g
)$. Let $f\colon \cM \rightarrow \RR$ be a differentiable function   and our goal is 
minimize such an objective function on $\cM$, i.e., 
$
\min_{x\in \cM} f(x) . 
$
 Initialized  from   $x_0 \in \cM$, stochastic (Riemannian)  gradient descent   constructs a sequence of iterates $\{ x_{t} \}_{t\geq 0} \subseteq \cM$  according to 
\#\label{eq:sgd_manifold}
x_{t+1} = \expm_{x_t} \bigl(- \alpha_t \cdot g_t \bigr ),
\#
where $g_t \in \cT_{x_t} \cM$ is the descent direction and $\alpha_t > 0$ is the stepsize. Here we assume  $ g_t  = \grad f(x_t) + \delta _t$ is a stochastic perturbation of $\grad f(x_t)$ with error $\delta_t$ which might be biased.  The following theorem establishes a finite-time convergence  guarantee for  such a method.

\begin{theorem}[Stochastic Riemannian Gradient Descent]
	\label{thm_conv_sgd_with_err}
	Let $f: \cM \to \RR$ be a differentiable function on a Riemannian manifold $(\cM, g)$ that is complete and connected. We assume $f$ is $\mu$-gradient dominated and   has  $L$-Lipschitz continuous  gradient, where parameters $\mu$ and $L$ satisfy  $0 < \mu \leq L$.
	 Let $\{ x_t\}_{t\geq 0}$ be generated  according to  \eqref{eq:sgd_manifold},
	where $g_t =  \grad f(x_t) + \delta_t$ for some $\delta_t \in \cT_{x_t}\cM$. 
	For each $t\geq 0$, we denote by $\cF_t$ the $\sigma$-algebra generate by $\{ x_k\}_{k\leq t}$.  In addition, we assume that there is a sequence  $\{ \varepsilon_t\}_{t\geq 0} \subseteq \RR$ such that   $\EE [ \la \delta_t, \delta_t \ra _{x_t} \given \cF_t ] = \varepsilon_t$ for all $t\geq 0$. Moreover, the stepsize $\alpha_t$ in \eqref{eq:sgd_manifold}  is set to be a constant $\alpha \in (0, 1/(4L)]$.  Then, for any $t \geq 1$, it holds that 
\#\label{eq:convergence_rate-r}
\EE[ f(x_t)]  - f^* 	   \leq \rho^t \cdot [f(x_0) - f^* ]  + \alpha \cdot \rho^t \cdot  \sum_{\ell =0}^{t-1}   \varepsilon_{\ell} \cdot \rho^{-(\ell+1)},
\# 
where we define  
	$\rho = 1 - \alpha \cdot \mu /2  $. 
 Furthermore, if sequence $\{ \varepsilon _t\}_{t\geq 0}$ converges to zero as $t$ goes to infinity, then    \eqref{eq:convergence_rate-r} implies that $ 
		\lim_{t \rightarrow  \infty} ~ \EE[ f(x_t)]  - f^* = 0
 $, i.e, the stochastic gradient updates in \eqref{eq:sgd_manifold} converge to the global minimum of $f$.
\end{theorem}

This theorem proves that, when $\alpha$ is properly chosen,  the expected  suboptimality $\EE [ f(x_t)] - f^*$ is upper bounded by a sum of two terms. The first term decays to zero at a linear rate, which corresponds to the convergence rate of  Riemannian gradient descent 
with an exact gradient.  
Meanwhile, the second term exhibits the effect of the gradient errors $\{ \delta_t\}_{ t\geq 0}$ and converges to zero if the second-order moments of $\{ \delta_t\}_{ t\geq 0}$, i.e., $\{ \varepsilon_t\}_{t \geq 0}$,  converge to zero.

\begin{proof} We first analyze the performance of a single step of stochastic Riemannian  gradient descent. Since $\grad f$ is  $L$-Lipschitz, by \eqref{eq:smoothness} in Lemma \ref{lem_descent_lemma},  we have for any $k\geq 0$ that
\begin{align} \label{eq:sgd1}
	f(x_{k+1}) &  \leq   f(x_k) - \alpha_k \cdot \bigl\la \grad f(x_k),  \grad f(x_k) \bigr \ra_{x_k} - \alpha_k \cdot  \bigl \la  \grad f(x_k), \delta_k \bigr\ra_{x_k} \notag \\
	& \qquad  +  \alpha_k^2 \cdot L/ 2 \cdot  \bigl \la  \grad f(x_k) + \delta_k,  \grad f(x_k) + \delta_k \bigr\ra_{x_k} . 
\end{align}
Here   $\la \cdot , \cdot \ra _{x_k}$ denotes the inner product on $\cT_{x_t}\cM$. 
By the basic inequality $2 ab \leq a^2 + b^2$,   we have
\begin{align}
	 - 2 \bigl \la \grad f(x_k), \delta_k \big\ra _{x_k} & \leq  \bigl \la  \grad f(x_k),  \grad f(x_k) \bigr\ra_{x_k} + \bigl \la \delta_k, \delta_k\bigr\ra_{x_k}, \label{eq:sgd2}\\
	  \bigl \la  \grad f(x_k) + \delta_k,  \grad f(x_k) + \delta_k \bigr\ra_{x_k} &\leq 2 \bigl \la  \grad f(x_k),  \grad f(x_k)\bigr\ra_{x_k} + 2 \bigl\la \delta_k, \delta_k\bigr\ra_{x_k} .  \label{eq:sgd3}
\end{align}
 Thus, combining \eqref{eq:sgd1}, \eqref{eq:sgd2}, and \eqref{eq:sgd3},  we have 
 \#\label{eq:sgd4} 
 f(x_{k+1}) &  \leq   f(x_k) - \alpha_k \cdot \bigl\la \grad f(x_k),  \grad f(x_k) \bigr \ra_{x_k}   \\
 & \qquad + ( \alpha_k /2 +  \alpha_k^2 \cdot L) \cdot \bigl [ \la  \grad f(x_k),  \grad f(x_k) \bigr\ra_{x_k} + \bigl \la \delta_k, \delta_k\bigr\ra_{x_k}\bigr ]  \notag \\
 & =    f(x_k) - \alpha_k \cdot ( 1- 2\alpha_k L) / 2 \cdot \bigl\la \grad f(x_k),  \grad f(x_k) \bigr \ra_{x_k} + \alpha_k ( 1+2  \alpha_k L ) /2 \cdot \bigl\la \delta_k, \delta_k\bigr\ra_{x_k}. \notag
  \#
  Moreover, since $f$ is $\mu$-gradient dominated, combining \eqref{eq:grad_dom} and \eqref{eq:sgd4} we have 
  \#\label{eq:sgd5} 
  f(x_{k+1})  \leq   f(x_k) - \mu \cdot \alpha_k \cdot ( 1- 2\alpha_k L)  \cdot [ f(x_k) - f^*] +\alpha_k ( 1+2  \alpha_k L ) /2 \cdot \bigl\la \delta_k, \delta_k\bigr\ra_{x_k}.
  \#
 Taking the conditional expectation   given $\cF_k $ on both sides of \eqref{eq:sgd5}, we obtain that 
 \#\label{eq:sgd6}
 \EE [ f(x_{k+1} ) \given \cF_k ] - f^* \leq  [ 1-  \mu \cdot \alpha_k \cdot ( 1- 2\alpha_k L)  ]  \cdot [ f(x_k) -f^* ] + \alpha_k ( 1+2  \alpha_k L ) /2 \cdot  \varepsilon_k.
 \#
  Thus, in \eqref{eq:sgd6} we establish the performance of a single step of stochastic gradient descent. 
  
  Under the assumption that $\alpha_k = \alpha  \leq 1 / (4L) $, we have $1-  \mu \cdot \alpha_k \cdot ( 1- 2\alpha_k L) \leq 1 - \alpha \cdot \mu /2 \in (0,1)$ and $ ( 1+2  \alpha_k L ) /2 \leq 1$. Recall that we define $\rho =1 - \alpha \cdot \mu /2 $.  Thus, by  taking the total expectation  on both ends of \eqref{eq:sgd6}, we have 
  \#\label{eq:sgd7} 
  \EE [ f(x_{k+1}) ] - f^*   \leq \rho \cdot \bigl \{  \EE[ f(x_k) ] - f^* \bigr \} + \alpha \cdot \varepsilon_k.
  \#
 The  recursive relationship in \eqref{eq:sgd7} implies that 
 \begin{align}
 \label{eq:sgd8}
 \rho^{-(k+1)} \cdot \bigl \{  \EE [ f(x_{k+1}) ] - f^*  \bigr \} \leq \rho^{-k} \cdot \bigl \{  \EE[ f(x_k) ] - f^* \bigr \}  + \rho^{-(k+1)} \cdot \alpha \cdot \varepsilon_k.
 \end{align}
 Thus, $\{ \rho^{-k} \cdot   \EE [ f(x_{k})   - f^* ]   \}_{k \geq 0  }$  form a telescoping series. From  \eqref{eq:sgd8}, by summing from $k = 0$ to any $t \geq 1$,   we have
 \$
  \rho^{-t} \cdot \bigl \{  \EE [ f(x_{t}) ] - f^*  \bigr \}  \leq \sum_{\ell = 0}^{t-1} \alpha\cdot  \varepsilon_\ell \cdot \rho^{-(\ell+1) }  + \bigl [f(x_{0})  - f^* \bigr  ]
 \$ 
 for all $t\geq 1$, which implies \eqref{eq:convergence_rate-r}.

   Finally, if the series  $\{ \varepsilon_t\}_{t\geq 0}$  converges to zero as $t$ goes to infinity, for any $\epsilon > 0$, there exists an integer $N_\epsilon >0$ depending on $\epsilon $ such that, $\varepsilon _t < \epsilon $ for all $t\geq N_{\epsilon}$.
   Thus, for any $t \geq N_{\epsilon} + 2$, we have 
 \#\label{eq:sgd9} 
   &\rho^t \cdot  \sum_{\ell =0}^{t-1}   \varepsilon_{\ell} \cdot \rho^{-(\ell+1)}   = \rho^t \cdot \biggl [ \sum_{\ell =0}^{N_{\epsilon} }   \varepsilon_{\ell} \cdot \rho^{-(\ell+1)} \biggr ] + \rho^t \cdot  \sum_{\ell =N_{\epsilon} +1}^{t-1 }   \varepsilon_{\ell} \cdot \rho^{-(\ell+1)}  \notag \\
  & \qquad  \leq \rho^t \cdot \biggl [ \sum_{\ell =0}^{N_{\epsilon} }   \varepsilon_{\ell} \cdot \rho^{-(\ell+1)} \biggr ]  +  \epsilon \cdot   \sum_{\ell =N_{\epsilon} +1}^{t-1 }     \rho^{t-(\ell+1)}   \leq \rho^t \cdot \biggl [ \sum_{\ell =0}^{N_{\epsilon} }   \varepsilon_{\ell} \cdot \rho^{-(\ell+1)} \biggr ]   + \epsilon / ( 1- \rho).
 \#  
 Note that $N_{\epsilon}$ in \eqref{eq:sgd9} does not depend on $t$.  Thus, combining \eqref{eq:convergence_rate-r} and \eqref{eq:sgd9}  and letting $t$ goes to infinity, we obtain that 
 \$
\limsup_{t\rightarrow \infty} \bigl \{  \EE[ f(x_t)]  - f^* 	 \bigr \} \leq  \epsilon / ( 1- \rho),
 \$
 where $\epsilon $ can be set arbitrarily small. Therefore, we have $\lim_{t\rightarrow \infty} \EE[ f(x_t) ] - f^* = 0$, which completes the proof of Theorem \ref{thm_conv_sgd_with_err}.
   \end{proof}

\section{Proof of Main Result}
In this section, we lay out the proofs of the main results, namely, Theorems \ref{thm:stat} and \ref{thm:comp}. 

\subsection{Proof of Theorem \ref{thm:stat}} \label{sec:proof_thm_stat}  

\begin{proof}
	The proof of Theorem \ref{thm:stat}  
	consists of three steps. Specifically, in the first and second steps, we upper bound the bias  caused by the RKHS norm regularization and the variance of the estimator, respectively.  Then, in the last step we balance  the   bias and variance  by choosing a  suitable regularization parameter $\lambda$ and obtain the final the estimation error $\| f_{n,\lambda} -f^* \|_{\mathcal{H}}$, which  further yields an upper bound on $\| \nabla f_{n,\lambda} - \nabla f^* \|_{\cL_{\PP}^2 }$.

	\vspace{4pt}
	
	{\noindent {\bf  Step 1.}} We first consider the bias incurred by the regularization.  To this end, we define  
	\#\label{eq:op_reg}
	f_{\lambda}=\argmin_{f \in \mathcal{H}}  L_{\lambda}(f), \qquad \text{where}~~  L_{\lambda}(f) = L(f)+ \lambda / 2\cdot \|f\|_{\mathcal{H}}^2 .
	\#
	Note that the optimization problem in \eqref{eq:op_reg} is the population version of the one in \eqref{eq:estimator}. 
	Since $f^*$ is equal to $f_{\lambda}$ in \eqref{eq:op_reg} with $\lambda =0 $,
	$\|f_{\lambda}-f^*\|_{\cH}$ reflects the  bias due to the  regularization $\lambda / 2\cdot \|f\|_{\mathcal{H}}^2$.

	The basic idea of bounding this term is as follows. By the definitions of $f_{\lambda}$ and $f^*$,  we have that
	\# \label{eq:optimality}
	\mathcal{D}L(f_{\lambda})+\lambda\cdot  f_{\lambda} = \mathcal{D} L_{\lambda} (f_{\lambda} )=0, \qquad 
	\mathcal{D}L(f^*)&=0.  
	\# 
	We aim to apply   spectral decomposition to   $ f_{\lambda} $ and $f^*$ using  the orthonormal  basis  of RKHS $\mathcal{H}$ and upper bound the discrepancy between the corresponding   coefficients. However, the challenge is that, when the population loss $L$ is not  quadratic, the Fr\'echet derivative $ \mathcal{D}L $ is not a linear functional, which makes it intractable to obtain $f_{\lambda}$  in closed form.   
	 To overcome such a  difficulty, we  take a detour by linearizing  the optimality condition of   $f_{\lambda}$   in  \eqref{eq:optimality}
	via the Taylor expansion and obtain  an approximate solution $\hat f_{\lambda}$ to the linearized equation. 
	Then, we control the bias by upper bounding  $\| \hat f_{\lambda} - f^*  \|_{\cH} $  and $\|\hat{f}_{\lambda}-f_{\lambda}\|_{\mathcal{H}}  $ separately.  
	Since  $\hat{f}_{\lambda}$ is the solution to  a linear equation, we can obtain it in a closed form, which enables us to upper bound $\|\hat{f}_{\lambda}-f^*\|_{\mathcal{H}}$ via spectral decomposition. Besides, to handle the linearization error $\|\hat{f}_{\lambda}-f_{\lambda}\|_{\mathcal{H}}  $, we resort to  
  the   contractive mapping theorem \citep{rudin1976principles} to control it  using  $\|\hat{f}_{\lambda}-f^*\|_{\mathcal{H}}$.

	Specifically, by applying the Taylor expansion to $\mathcal{D} L_{\lambda}$, 
	for any $f$  that is close to $f^*$, we have 
	\# \label{linear}
	\mathcal{D}L_{\lambda} (f)\approx \mathcal{D}L_{\lambda} (f^*)+\mathcal{D}^2L_{\lambda} (f^*)(f-f^*). 
	\#
	By setting the right-hand-side of   \eqref{linear} to zero,  the solution  $\hat f_{\lambda}$  is given by  
\#\label{eq:def_hat_f_lam2}
\hat{f}_{\lambda}= f ^* - \lambda \cdot  \bigl(\mathcal{D}^2 F^* (f^*)+\lambda \cdot   \id  \bigr)^{-1} \cdot f^* =  \bigl[  \bigl (\mathcal{D}^2 F^* (f^*)+\lambda \cdot  \id  \bigr )^{-1}\circ \mathcal{D}^2 F^* (f^*) \bigr] f^*,
\#
 where $ \mathcal{D}^2 F^* (f^*) \colon \cH \rightarrow \cH $ is the second-order Fr\'echet derivative with respect to the RKHS norm  and $\id \colon \cH  \rightarrow  \cH  $ is the identity mapping on $\cH$.
By the triangle inequality, we upper bound the bias term $\|f_{\lambda}-f^*\|_{\cH}$ by 
\#\label{eq:bound_bias11}
 \|f_{\lambda}-f^*\|_{\cH} \leq   \| \hat f_{\lambda } - f^* \|_{\cH} + \| f_{\lambda } - \hat f_{\lambda } \|_{\cH} . 
\#

 We first bound $ \| \hat f_{\lambda } - f^* \|_{\cH}$ via spectral decomposition. 
 As introduced in \S\ref{bg:rkhs}, $\cH$ can be viewed as a subset of $\cL_{\nu} ^2 (\cX) $, where $\nu$ denotes the Lebesgue measure on $\cX$. 
 Consider the connection between $\cH$ and $\cL_{\nu}^2(\cX)$.
Recall that we introduce the 
  operator $\cC \colon \cH \rightarrow \cH$ in  \eqref{eq:define_TC}, which admits a spectral decomposition, where the eigenfunctions  $\{\psi_i\}_{i \ge 1}$ form an orthogonal system of $\cL_{\nu}^2 (\cX)$. 
	Moreover, for any $f, g\in \cH$, by the definition of $\cC$, we have 
	\#\label{eq:cc_inner}
	\la f, g \ra _{\cL_{\nu}^2 } & = \int_{\cX} f(x) \cdot g(x ) ~\ud  x  = \int_{\cX} f(x) \cdot \bigl\la g(\cdot ), K(x, \cdot) \bigr\ra _{\cH} ~\ud   x  \notag \\
	& =  \biggl \la g (\cdot ), \int_{\cX} f(x) \cdot    K(x, \cdot)   ~\ud x  \biggr \ra_{\cH} = \la g, \cC f \ra _{\cH}. 
	\#
	Then, using the orthogonal system  $\{\psi_i\}_{i \ge 1}$, 
	we can write $f^* $ using  
	$\{\psi_i\}_{i \ge 1}$ as $f^*(x)=\sum_{i=1}^{\infty} a_i \cdot \psi_i(x)$, where $\{ a_i \}_{i\geq 1} \subseteq \RR$. Besides,  let $\{\zeta_i\}_{i\ge 1}$ be the eigenvalues of the operator  $\mathcal{D}^2 F ^*(f^*) \colon \cH \rightarrow \cH$. 
	Since $F^*$ is strongly  convex under Assumption \ref{assume:convexity}, for any $h \in \cH$, by \eqref{eq:strongconvex} in  Assumption \ref{assume:convexity} and \eqref{eq:cc_inner}, we have 
	\$
\bigl \la h, 	\cD^2 F^*(f^*) h \bigr \ra_{\cH}   \geq  \kappa \cdot  \la h, h  \ra _{\cL_{\nu}^2 } = \kappa \cdot 
	\la h, \cC h  \ra_{\cH}, 
	\$
	which 
  implies that  $ \mathcal{D}^2 F^* (f^*) \succeq \kappa \cdot \mathcal{C} $. This further implies that $\kappa \cdot \mu_i \le \zeta_i$ for all $i\ge1$, where  
	$\{\mu_i \}_{i\ge 1} $ are the eigenvalues of operator $ \mathcal{C}$.  
	Thus, combining the 
	spectral decomposition  of $f^*$ and 
	the 
	definition of 
	$\hat{f}_{\lambda}$ in   \eqref{eq:def_hat_f_lam2}, we obtain
	\#\label{eq:decom_diff}
	f^* -   \hat{f}_{\lambda}
	=  \lambda \cdot  (\mathcal{D}^2 F^* (f^*)+\lambda \cdot  \id )^{-1} \cdot f^* 
	= 
	\sum_{i=1}^{\infty} \frac{   \lambda\cdot  a_i}{\lambda+\zeta_i} \cdot  \psi_i
	=\sum_{i=1}^{\infty} \frac{ \lambda \cdot  a_i \cdot     \mu^{-1/2}_i}{\lambda+\zeta_i} \cdot \sqrt{\mu_i}\psi_i.
	\#
	Since $\{\sqrt{\mu_i}\psi\}_{i\ge1}$ forms an orthogonal basis of  $\mathcal{H}$, by Parseval's equation and \eqref{eq:decom_diff}, we have
	\#\label{eq:decom_diff2}
	\|\hat{f}_{\lambda}-f^*  \|^2_{\mathcal{H}}&=\sum_{i=1}^{\infty} \biggl( \frac{\lambda \cdot  a_i \cdot  \mu^{-1/2}_i}{\lambda+\zeta_i}\biggr)^2=\lambda^2 \cdot \sum_{i=1}^{\infty}\biggl( \frac{\mu_i^{(\beta-1)/2}}{\lambda+\zeta_i} \biggr)^2\cdot \mu_i^{-\beta}a_i^2 \notag  \\
	&\le (\lambda /\kappa  )^2 \cdot \sum_{i=1}^{\infty}\biggl( \frac{\mu_i^{(\beta-1)/2}}{\lambda/\kappa +\mu_i} \biggr)^2\cdot \mu_i^{-\beta}\cdot a_i^2 , 
	\#
	where $\beta \in (1,2)$ and the  inequality follows from  the fact that $\kappa \cdot \mu_i \leq \zeta _i$. 
	Meanwhile, for any $a\in (0, 1)$ and any  $t, c>0$, it holds that ${t^{\alpha}}/{(c+t)}\le c^{\alpha-1}$. Note that we have $1 < \beta <2$. Applying this inequality with  $\alpha =  (3-\beta) /2 $, $t = \lambda / \kappa$ and $c = \mu_i$, we have 
	\#\label{eq:decom_diff3}
	(\lambda /\kappa  )^{3-   \beta} \cdot (\lambda /\kappa  + \mu_i )^{-2} \leq \mu_i ^{ 1 -  \beta   }.
	\#  
	Thus, combining \eqref{eq:decom_diff2} and  \eqref{eq:decom_diff3}, we have 
	\#\label{eq:decom_diff4}
	\|\hat{f}_{\lambda}-f^*  \|^2_{\mathcal{H}} &  \le
	(\lambda /\kappa  )^{  \beta -1 }  \cdot \sum_{i=1}^{\infty} \bigl (\mu_i^{(\beta-1)/2} \bigr )^2 \cdot \mu_i^{1- \beta }  \cdot \mu_i^{-\beta}\cdot a_i^2  \leq  (\lambda/\kappa )^{\beta-1}\cdot \sum_{i=1}^{\infty}\mu_i^{-\beta} a_i^2\notag \\
	& =  (\lambda /\kappa  )^{  \beta -1 } \cdot \| f^*\|^2_{\mathcal{H}^{\beta}},
	\#
	where $\cH^\beta$ is the power space of order $\beta$, whose definition is given in 
	\S\ref{sb}.
	
	It remains to upper bound the 
	linearization error $\|\hat{f}_{\lambda}-f_{\lambda}\|_{\mathcal{H}}$. 
	In the following, for notational simplicity, we denote $\cD L_{\lambda}$  by $\cZ_{\lambda}$ and define  an 
	operator $\cF_{\lambda} \colon \cH \rightarrow \cH$ by 
	\# \label{eq:def_F_lambda}
	\mathcal{F}_{\lambda}(\phi)=\phi- \bigl [  \mathcal{D}\cZ_{\lambda} (f^*) \bigr ] ^{-1}\circ \mathcal{Z}_{\lambda}(f^*+\phi), \qquad  \forall \phi \in \cH.
	\# 
	Intuitively,  $\mathcal{F}_{\lambda}$ characterizes the linearization error of $\mathcal{Z}_{\lambda}$. 
	To see this, recall that the first equation in  \eqref{eq:optimality} is equivalent to $\cZ_{\lambda}(f_{\lambda})=0 $.   
	As we see later, 
	$ \mathcal{F}_{\lambda} $ is a contraction mapping in a neighborhood of the target function $f^*$ and thus  has a  unique  fixed point $\phi_{\lambda}$  locally, which satisfies $\mathcal{F}_{\lambda}(\phi_{\lambda})=\phi_{\lambda}$.  Hence,  by the definition of $\mathcal{F}_{\lambda}$  in  \eqref{eq:def_F_lambda},     we have $ \mathcal{Z}_{\lambda}(f^*+\phi_{\lambda})=0$, which implies that $f_{\lambda}=f^*+\phi_{\lambda}$ by strong convexity. 
	Besides, since $\cD L(f^*) = 0$, we have $\cZ_{\lambda} (f^*) = \lambda \cdot f^*$.  
	Thus, 
	$\hat f_{\lambda} $ defined in \eqref{eq:def_hat_f_lam}   can be written as $\hat f_{\lambda} = f^* + \mathcal{F}_{\lambda}(0)  $. 
	Then, we can write the difference between $\hat f_{\lambda}$ and $ f_{\lambda}$   as 
	\# \label{contraction}
	f_{\lambda}-\hat{f}_{\lambda}=({f}_{\lambda}-f^*)-(\hat{f}_{\lambda}-f^*)=\mathcal{F}_{\lambda}(\phi_{\lambda})-\mathcal{F}_{\lambda}(0),
	\# 
	which reduces our problem to upper  bounding the discrepancy between  $ \phi_{\lambda} $ and zero.

	In the following, we first show that $\cF_{\lambda}$ defined in \eqref{eq:def_F_lambda} admits a unique fixed point locally. Specifically, let  $
	\cB_{2d_{\lambda}}(0):=\{ f \in \cH: \|f\|_{\cH}\le 2 d_{\lambda} \} 
	$ 
	be the RKHS ball  centered at zero with radius $ 2 d_\lambda$, where we define   $ d_{\lambda} = \|\hat{f}_{\lambda}-f^*\|_{\cH}$.
	For any $\phi \in \cB_{2d_{\lambda}} (0)$, by triangle inequality  we have 
	\#\label{eq:bound_hnorm1}
	\| \cF_{\lambda} (\phi)  \|_{\cH} &
	\leq \| \hat f_{\lambda} - f^* \|_{\cH }  
 +  \bigl \| \phi -\bigl [  \mathcal{D}\cZ_{\lambda} (f^*) \bigr ] ^{-1}\circ \mathcal{Z}_{\lambda}(f^*+\phi)  - ( \hat f_{\lambda}  - f^* ) \bigr \|_{\cH }   \notag \\
	&\leq d_{\lambda} +   \| \cF_{\lambda}(\phi) -  \cF_{\lambda}(0)   \|_{\cH}.   
	\#  
	Applying the second-order  Taylor expansion to  $\mathcal{Z}_{\lambda}$, we obtain that 
	\#\label{eq:taylor11}
	& \mathcal{Z}_{\lambda}(f^*+\phi_2) =\mathcal{Z}_{\lambda}(f^*+\phi_1)+\bigl [ \mathcal{D} \cZ_{\lambda}(f^*)  \cdot (\phi_2-\phi_1) \bigr ] + \\
	&\qquad \qquad \int_0^1\int_0^1 \cD^2\cZ_{\lambda}\bigl \{ f^*+s' \cdot [\phi_1+s \cdot (\phi_2-\phi_1) ] \bigr \}  \cdot (\phi_2-\phi_1)  \cdot   [ \phi_1+s\cdot(\phi_2-\phi_1)] ~ \ud s \ud s', \notag 
	\#
	where $\phi_1$ and $\phi_2$ are any two elements in $\cB_{2d_{\lambda}}(0)$.
	By rearranging the terms in \eqref{eq:taylor11}.
	we have 
	\#\label{eq:taylor12}
	& \bigl [  \mathcal{D}\cZ_{\lambda} (f^*) \bigr ] ^{-1}  \circ \bigl [  \mathcal{Z}_{\lambda}(f^*+\phi_1 )  - \mathcal{Z}_{\lambda}(f^*+\phi_2 ) \bigr ] =   ( \phi_1 - \phi_2)  - \bigl [  \mathcal{D}\cZ_{\lambda} (f^*) \bigr ] ^{-1}  \\
	& \qquad \qquad 
	\circ\int_0^1\int_0^1 \cD^2\cZ_{\lambda}\bigl \{ f^*+s' \cdot [\phi_1+s \cdot (\phi_2-\phi_1) ] \bigr \}  \cdot (\phi_2-\phi_1)  \cdot   [ \phi_1+s\cdot(\phi_2-\phi_1)] ~ \ud s \ud s'.\notag  
	\#
	Besides, recall that we have $\cD \cZ_{\lambda}(f  )=\cD^2  F^* (f ) + \lambda \cdot \id $ and $\cD^2 \cZ_{\lambda} (f) = \cD^3 F^*(f )$ for any $f \in \cH$. 
	Hence, by \eqref{eq:taylor12} and the definition of $\cF_{\lambda}$ in \eqref{eq:def_F_lambda}, we have 
	\$
	& \cF_{\lambda}(\phi_1)-\cF_{\lambda}(\phi_2) =\bigl (\cD^2 F^*(f^*)+\lambda \cdot  \id  \big )^{-1} \\
	&\qquad \qquad \circ \int_0^1\int_0^1 \cD^3 F^* \bigl\{f^*+s'\cdot \bigl[ \phi_1+s \cdot (\phi_2-\phi_1 ) \bigr] \bigr\} \cdot  (\phi_2-\phi_1)\cdot [\phi_1+s\cdot (\phi_2-\phi_1) ]  ~\ud s \ud s'.
	\$
	In addition, 
	by convexity, we have  $\phi_1+s\cdot (\phi_2-\phi_1)\in \cB_{2d_{\lambda}}(0)$ for all $s \in [0,1]$. Thus,  by the definition of $\Theta_{\lambda,3}$ in   \eqref{theta3} and the fact that $ \|\phi_1+s \cdot (\phi_2-\phi_1)\|_{\mathcal{H}}\le 2 d_{\lambda}$, we have 
	\#\label{eq:bound_hnorm2}
	\|\cF_{\lambda}(\phi_1)-\cF_{\lambda}(\phi_2)\|_{\mathcal{H}}\le  2 \Theta_{\lambda,3}\cdot  d_{\lambda} \cdot \| \phi_1-\phi_2\|_{\mathcal{H}}.
	\# 
	Moreover, under Assumption \ref{assume:linearize}, 
	both  $d_{\lambda} = \|\hat{f}_{\lambda}-f^*\|_{\cH}$ and $\Theta_{\lambda, 3}\cdot d_{\lambda}$ converge to  zero as $\lambda$ goes to zero. 
	Therefore,  
	there exists a constant $\lambda_0>0$ such that, for all $\lambda<\lambda_0$, we have $\Theta_{\lambda,3}\cdot  d_{\lambda}<1/4  $.  As a result, \eqref{eq:bound_hnorm2}  implies that $\cF_{\lambda}  $ is a $1/2$-contractive mapping on the RKHS ball $\cB_{2 d_{\lambda}}(0)$.
	Furthermore,  combining \eqref{eq:bound_hnorm1} and \eqref{eq:bound_hnorm2} we obtain  for any $\phi \in \cB_{2 d_{\lambda}}(0)$ that
	\$
	\| \cF_{\lambda} (\phi)  \|_{\cH} \leq d_{\lambda} + \| \cF_{\lambda} (\phi) - \cF_{\lambda} (0) \|_{\cH } \leq d_{\lambda} + 1/2 \cdot \| \phi \|_{\cH} \leq 2 d_{\lambda}, 
	\$  
	which implies that $\cF_{\lambda} (\phi) \in \cB_{2 d_{\lambda}}(0)$ for all $\phi \in \cB_{2 d_{\lambda}}(0)$. 
	Thus, by the contractive map theorem \citep{rudin1976principles}, 
	$\cF_{\lambda}$ admits a unique fixed point $\phi_{\lambda}$ in $\cB_{2 d_{\lambda}}(0)$, which satisfies $\phi_{\lambda} = \cF_{\lambda}(\phi_{\lambda}
	)$. 
	
	Finally, using the fixed point $\phi_{\lambda}$, we can write $f_\lambda $ as $f^* + \phi_{\lambda} $. Thus,  
	by \eqref{contraction} we have 
	\#\label{eq:bound_linear_final}
	\|f_{\lambda}-\hat{f}_{\lambda}\|_{\cH}=\| \mathcal{F}_{\lambda}(\phi_{\lambda})-\mathcal{F}_{\lambda}(0)\|_{\cH}\le 1/2\cdot \|\phi_{\lambda}\|_{\cH} \le d_{\lambda}=\|\hat{f}_{\lambda}-f^*\|_{\cH}.
	\# 
	Then we combine \eqref{eq:decom_diff4} and \eqref{eq:bound_linear_final} to obtain 
	\#\label{eq:bias_final}
	\|f^*-{f}_{\lambda} \|_{\cH}\le \|f_{\lambda}-\hat{f}_{\lambda} \|_{\cH}+ \|f^*-\hat{f}_{\lambda} \|_{\cH}\le 2 (\lambda /\kappa  )^{(\beta-1)/2} \cdot \| f^*\|_{\mathcal{H}^{\beta}},
	\#
	which establishes the bias of our estimator.  Thus, we conclude the first step of the proof.
	
	\vspace{4pt}     
	{\noindent {\bf  Step 2.}} In the second step, we control the variance of the  estimator $f_{n, \lambda}$, which is characterized by $\|f_{\lambda}-{f}_{n,\lambda}\|_{\cH}  $. 
Recall that we define $f_{n, \lambda}$ defined in \eqref{eq:estimator} is the solution to the regularized empirical risk minimization problem. By the optimality condition of $f_{n, \lambda}$, we have 
	\# \label{eq2}
	L_n( f _{n,\lambda})+  \lambda / 2 \cdot \|  f _{n,\lambda}\|^2_{\cH} \le L_n(f_{\lambda})+\lambda / 2 \cdot\| f_{\lambda} \|^2_{\cH}. 
	\# 
	Besides, since $L_n \colon \cH \rightarrow \RR$ is a convex functional on $\cH$,   
  $L_n(\cdot)+\lambda/2\cdot \| \cdot\|^2_{\cH}  $ is $\lambda/2$-strongly convex with respect to the RKHS norm,  which implies 
	\# \label{eq3}
	&   L_n(f _{n,\lambda})+\lambda / 2 \cdot \| f _{n,\lambda}\|^2_{\cH}  \notag \\
	& \qquad \ge L_n(f_{\lambda})+\lambda / 2 \cdot   
	\| f_{\lambda} \|^2_{\cH}+\bigl \la \cD L_n(f_{\lambda})+\lambda \cdot  f_{\lambda}, f  _{n,\lambda}-f_{\lambda} \bigr \ra_{\cH}+\lambda / 2 \cdot\| f _{n,\lambda}-f_{\lambda}\|^2_{\cH}. 
	\#
	Hence, 
	by \eqref{eq2}, \eqref{eq3}, and the Cauchy-Schwartz inequality, we have 
	\# \label{eq4}
	\| f _{n,\lambda}-f_{\lambda}\|_{\cH}\le 2 / \lambda \cdot   \| \cD L_n(f_{\lambda})+\lambda  \cdot f_{\lambda}\|_{\cH}.
	\#  
	Thus, to control the variance of $f_{n, \lambda}$, 
	it remains to upper bound the term $ \| \cD L_n(f_{\lambda})+\lambda f_{\lambda}\|_{\cH}  $. Note  that $\EE[L_n(f) ]=L(f) $ for any $f\in \cH$ and that  $f_{\lambda}$ is the minimizer of the  regularized population risk $L_{\lambda}(f) = L(f)+\lambda/2\cdot \| f\|^2_{\cH}  $. We have 
	\$
	\EE\bigl[\cD L_n(f_{\lambda})+\lambda \cdot f_{\lambda}\bigr] = \cD L_{\lambda} (f_{\lambda} )=0, 
	\$ 
	which implies that
	\begin{align}
	\label{eq:bound_h_norm1}
	&  \bigl \| \cD L_n(f_{\lambda})+\lambda \cdot  f_{\lambda}\bigr \|_{\cH} \notag \\
	&\qquad =  \Bigl \| \cD L_n(f_{\lambda})+\lambda \cdot  f_{\lambda}- \EE\bigl[\cD L_n(f_{\lambda})+\lambda \cdot  f_{\lambda}\bigr] \Bigr\|_{\cH}=\Bigl \| \cD L_n(f_{\lambda})- \EE \bigl[\cD L_n(f_{\lambda})\bigr] \Bigr\|_{\cH}.
	\end{align}	 
	By the definition of $L_n$ in \eqref{eq:estimator} and the reproducing   property of RKHS $\cH$, we obtain
	\#\label{eq:bound_h_norm2}
	\cD L_n(f_{\lambda})- \EE[\cD L_n(f_{\lambda})]=\frac{1}{n}\sum_{i=1}^{n}K(X_i,\cdot)-\int_{\cX}K(x ,\cdot)~\PP(\ud x ),
	\# 
	where $K(\cdot,\cdot)$ is the reproducing  kernel of   $\cH$. 
	Thus, combining \eqref{eq:bound_h_norm1} and \eqref{eq:bound_h_norm2}, to  upper bound $    \| \cD L_n(f_{\lambda})+\lambda \cdot  f_{\lambda}\|_{\cH}$, 
	it suffices to characterize the concentration error of i.i.d.  random variables  $\{K(X_i, \cdot )\}_{ i \in [n]}$, measured in terms of the RKHS norm. 
	We achieve such a goal via the Bernstein-type inequality for random variables taking values in a Hilbert space, which is stated as follows.

	\begin{lemma} \label{Bern} (Bernstein-Type  Inequality for Hilbert Space Valued Random Variables) Let $(\Omega, \cE, \PP)$ be a probability space and let  $(\cH,\la \cdot ,\cdot \ra_{\cH})$ be a separable Hilbert space. We assume  that $\xi \colon \Omega\rightarrow \cH $ is a stochastic element taking values in $\cH$. Suppose  that there exist  constants $L, \sigma>0$ such that $\|\xi(\omega)\|_{\cH } \le L$ almost surely over $\PP$ and that $\EE_{\PP} [\| \xi (\omega)  \|^2_{\cH}] \leq \sigma^2$. 
		Let $\omega_1,\ldots,\omega_n$ be 
		$n$ i.i.d.~random variables with distribution $\PP$.  Then, for any   $\tau\ge 1$, we have 
		\$
		\PP\biggl(   \Bigl \| \frac{1}{n}\sum_{i=1}^n\xi(\omega_i) - \EE_{\PP} \bigl[ \xi(\omega)\bigr]   \Bigr\|_{\cH}\ge    \sqrt{\frac{16\tau   \sigma^2}{n}} +\frac{4 L\tau }{3n}    \biggr) \le 2e^{-\tau}.
		\$
	\end{lemma}
	Lemma \ref{Bern}, obtained from \cite{steinwart2008support},    is an extension of the classic Bernstein's  inequality \citep{bernstein1946theory} to  random variables taking values  in a  Hilbert space $\cH$. See Theorem 6.14 in \cite{steinwart2008support}    for  more  details. 
	
Now we 
 apply Lemma \ref{Bern} to random variables $\{ K(X_i, \cdot)\}_{i=1}^n$. Recall that under Assumption \ref{ass4}, we have $ \| K(x, \cdot ) \|_{\cH}  \leq C_{K,1}$ for all $x\in \cX$. 
  Hence, applying Lemma \ref{Bern} with $L = \sigma = C_{K,1}$,    we have for all $\tau\ge1$ that
 \#\label{eq:apply_bern}
 \PP\Bigl(
 \cD L_n(f_{\lambda})- \EE[\cD L_n(f_{\lambda})]\ge 4 C_{K,1} \cdot \bigl ( \sqrt{ \tau / n} + \tau /  n   \bigr )    \Bigr) \le 2e^{-\tau}.
 \#


	Finally, combining \eqref{eq4}, \eqref{eq:bound_h_norm1}, and \eqref{eq:apply_bern}, we have for any $\tau \geq 1$ that
	\#\label{eq:step2_final}
	\PP \Big[ \| f_{n,\lambda}-f_{\lambda}\|_{\cH}>  8 C_{K,1}  /\lambda \cdot   \bigl(  \sqrt{ \tau  / n } +   \tau / n \bigr ) \Bigr ] \leq 2 e^{-\tau}.  
	\# 
	with probability at least $1-2e^{-\tau}$.  
	Equivalently, with probability at least $1 - 2 e^{\tau}$, we have 
	\#\label{eq:var_term_final}
	\| f_{n,\lambda}-f_{\lambda}\|_{\cH} \leq 8 C_{K,1}  /\lambda \cdot   \bigl(  \sqrt{  \tau  / n } +   \tau  / n  \bigr )  \leq  16   C_{K,1}  /\lambda\cdot \sqrt{ \tau / n  }  . 
	\# 
For the notational simplicity, in the sequel, we let $\alpha_K$ denote $16 C_{K,1} / \lambda$. 	
	Recall that the expected value of a nonnegative random variable $U$ can be written as $\EE(U) = \int_{0 }^\infty \PP( U > t)~ \ud t$.  By letting $\tau \geq 1$ in \eqref{eq:var_term_final}, it holds that 
\#\label{eq:var_expectation1}
  \EE \big( \| f_{n,\lambda}-f_{\lambda}\|_{\cH}^2  \bigl 
  )  & \leq   \alpha_K^2 / n     + \int_{   \alpha_K^2 / n     } ^{  \infty }  \PP \bigl(  \| f_{n,\lambda}-f_{\lambda}\|_{\cH}^2   > t  \bigr )~ \ud t .
\# 	
For any $t \geq  \alpha_K^2 / n$, we define $\tau_t =    n    t \cdot \alpha_K^{-2}  $. Then, since $\tau_t  \geq  1$,  it holds that 
\#\label{eq:expectation_term1}
& \int_{       \alpha_K^2 / n     } ^{  \infty }   \PP \bigl( \| f_{n,\lambda}-f_{\lambda}\|_{\cH} ^2  > t  \bigr )~\ud t    \leq \int_{       \alpha_K^2 / n     } ^{  \infty } \PP \bigl(\| f_{n,\lambda}-f_{\lambda}\|_{\cH}  > \alpha_K  \cdot   \sqrt{  \tau_t / n }    \bigr) \ud t   \notag \\
& \qquad \leq  \int_{       \alpha_K^2 / n     } ^{  \infty }   2 \exp( - \tau_t )~\ud t \leq   \int_{ 1 }^{\infty} 2 \alpha_K ^2 /  n    \cdot  \exp( - u   ) ~ \ud u \leq   2 \alpha_K ^2 /  n  ,
\#
where the last inequality follows from the change of variable $u = \tau_t$.
Hence,  combining \eqref{eq:var_expectation1} and  \eqref{eq:expectation_term1},  we have that 
\#\label{eq:final_expectation}
 \EE \big( \| f_{n,\lambda}-f_{\lambda}\|_{\cH}^2  \bigl 
)   & \leq  3 \alpha_K^2  / n = 768\cdot  C_{K,1}^2 / ( \lambda^{2} \cdot n)  . 
\#
Therefore, in \eqref{eq:var_term_final} and \eqref{eq:final_expectation}, we establish high-probability and in-expectation bounds  for variance term, which completes the second step. 

\vspace{4pt}   
{\noindent {\bf  Step 3.}} Finally, in the last step, we choose a proper $\lambda$ to balance    the bias and variance terms  in \eqref{eq:bias_final}, \eqref{eq:var_term_final}, and \eqref{eq:final_expectation}. 
To begin with, by \eqref{eq:bias_final}, \eqref{eq:step2_final}, and the triangle inequality,  with probability at least $1-2e^{-\tau}$,  we have 
\#\label{eq:final_bound1}
\| f_{n,\lambda}-f^*\|_{\cH}&  \leq 2  (\lambda /\kappa  )^{(\beta-1)/2} \cdot \| f^*\|_{\mathcal{H}^{\beta}}  + 16   C_{K,1}  /\lambda\cdot \sqrt{ \tau / n  }
\#
for all $\tau \geq 1$.  
Note that both $\kappa$ and $C_{K,1} $ are constants. 
Setting $\tau =  3 \log n $ in \eqref{eq:final_bound1}, when 
$n$ is sufficiently large, we have 
\#\label{eq:final_bound2}
\| f_{n,\lambda}-f^*\|_{\cH}   \leq 2  (\lambda /\kappa  )^{(\beta-1) /2 } \cdot \| f^*\|_{\mathcal{H}^{\beta}}+  32  C_{K,1} / \lambda  \cdot \sqrt{ \log n /  n}  
\#
with probability at least $1 - n^{-2}$.
Moreover, to balance  
	the two terms on the right-hand side of \eqref{eq:final_bound2},  
	we
	set the regularization parameter $\lambda$ to be 
	\#\label{eq:set_reg_param}
	\lambda  =  \cO\Bigl [ C_{K,1}^{2 / ( \beta+1)}  \cdot \kappa^{(\beta-1) / (\beta+1)}   \cdot \| f^* \|_{\cH ^\beta} ^{ - 2 /  (\beta+1) } \cdot  (\log n / n )^{1/ (\beta+1)}    \Bigr ]
	 \# 
	 in \eqref{eq:final_bound2}, which further  implies that 
	\#\label{eq:large_prob_bound}
	\| f_{n,\lambda}-f^*\|_{\cH} = \cO\Bigl [  C_{K,1}^{(\beta-1) / (\beta+1) } \cdot \kappa ^{-(\beta-1) / (\beta+1) }  \cdot \| f^* \|_{\cH^\beta} ^{ 2 /  (\beta+1) }  \cdot  ( \log n / n  )^{  ( \beta-1 )/ (2 \beta+ 2)}    \Bigr ]
	\#  
	with probability at least $1 - n^{-2}$.
	Here $\cO(\cdot)$ in \eqref{eq:large_prob_bound} hides absolute constants that does not depend on $\kappa$,   $C_{K,1}$, or  $\| f^* \|_{\cH^\beta}$. Recall that we define $\alpha_{\beta} =  (\beta-1) / (\beta+1)$. 
	Then, 
	\eqref{eq:large_prob_bound} can be equivalently written as 
	\#\label{eq:highp_bound}
	\| f_{n,\lambda}-f^*\|_{\cH} = \cO\Bigl [  C_{K,1}^{\alpha_\beta  }   \cdot \kappa ^{- \alpha_\beta   } \cdot \| f^* \|_{\cH^\beta} ^{ 1- \alpha_\beta }  \cdot  ( \log n / n  )^{ \alpha_\beta /2 }    \Bigr ].
	\#

Finally, in the rest of the proof, we utilize \eqref{eq:highp_bound} to obtain an upper bound on the gradient error $  \nabla f_{n,\lambda} - \nabla f^*  $.  Notice that the reproducing property of $\cH$ implies that 
\#\label{eq:grad}
 \partial_j   f(x)  = \frac{ \partial  f(x) }{\partial x_j} = \frac{ \partial \la K(x, \cdot ), f \ra_{\cH } }{\partial x_j }  = \biggl \la \frac{ \partial  K(x, \cdot )}{ \partial x_j },  f \biggr  \ra_{\cH } = \la \partial _j K(x, \cdot ), f \ra _{\cH}
\#
holds for all $ f\in \cH$, $x\in \cX$, and $   j\in \{1, \ldots, d\}$. 
By   Cauchy-Schwarz inequality, \eqref{eq:grad} implies that  
$|  \partial _j f(x) | \leq \|f \|_{\cH} \cdot \| \partial _j K(x, \cdot ) \|_{\cH }$. Hence, by Assumption \ref{ass4}, we have 
$ 
|  \partial _j f(x) |  \leq C_{K,2} \cdot \| f\|_{\cH},
$
which implies that 
\#\label{eq:grad_f_upper}
\bigl \| \nabla  f_{n, \lambda }(x)  - \nabla f^* (x) \bigr  \|_2 \leq C_{K,2} \cdot \sqrt{d} \cdot \| f_{n,\lambda} - f^* \|_{\cH}, \qquad \forall x\in \cX.
\#
Note that the right-hand side of \eqref{eq:grad_f_upper} does not depend on $x$. Integrating both sides of \eqref{eq:grad_f_upper} with respect to measure $\PP$, we obtain that 
\#\label{eq:integral_error_grad}
\int_{\cX} \bigl \| \nabla  f_{n, \lambda }(x)  - \nabla f^* (x) \bigr  \|_2^2  ~\ud \PP(x)  \leq C_{K,2}^2 \cdot d \cdot \| f_{n,\lambda} - f^* \|_{\cH}^2 . 
\#
Plugging \eqref{eq:highp_bound}  into \eqref{eq:integral_error_grad}, we complete the proof of Theorem \ref{thm:stat}.

Furthermore, we can similarly obtain in-expectation bounds utilizing \eqref{eq:final_expectation}. 
Specifically, 
combining \eqref{eq:bias_final} and  \eqref{eq:final_expectation}, 
we obtain 
\$
 \EE \big( \| f_{n,\lambda}-f ^* \|_{\cH}^2  \bigl 
 )&   \leq 2 \EE \big( \| f_{n,\lambda}-f_{\lambda}\|_{\cH}^2  \bigl 
)  + 2 \| f_{\lambda} - f^* \|_{\cH} ^2 \notag \\
&\leq 4  (\lambda /\kappa  )^{\beta-1} \cdot \| f^*\|_{\mathcal{H}^{\beta}} ^2  + 768\cdot   C_{K,1}^2   /( \lambda^2 \cdot n). 
\$
Setting  the regularization parameter $\lambda $ as in \eqref{eq:set_reg_param} yields that 
$$
\EE \big( \| f_{n,\lambda}-f ^* \|_{\cH}^2  \bigl 
) =   \cO\Bigl [  C_{K,1}^{2\alpha_\beta  }   \cdot \kappa ^{-2 \alpha_\beta   } \cdot \| f^* \|_{\cH^\beta} ^{ 2-2 \alpha_\beta }  \cdot  ( \log n / n  )^{ \alpha_\beta  }    \Bigr ],
$$
where $\alpha _{\beta} = (\beta-1) / (\beta + 1)$. 
Furthermore, taking the expectation on both ends of \eqref{eq:integral_error_grad} with respect to the randomness of $\{ X_i\}_{i\in [n]}$, we have 
\$
\EE\bigg [ \int_{\cX} \bigl \| \nabla  f_{n, \lambda }(x)  - \nabla f^* (x) \bigr  \|_2^2  ~\ud \PP(x)\bigg ] &  \leq C_{K, 2}^2 \cdot d\cdot \EE \big( \| f_{n,\lambda}-f ^* \|_{\cH}^2  \bigl 
) \notag \\
& =   \cO \bigl [ C_{K,2}^2 \cdot d \cdot C_{K,1}^{2 \alpha_\beta  } \cdot \kappa ^{- 2 \alpha_\beta   } \cdot \| f^* \|_{\cH^\beta} ^{2-2 \alpha_\beta }   \cdot  ( \log n / n  )^{ \alpha_\beta }   \bigr ] .
\$
Therefore, we have established in-expectation bounds that are similar to those in \eqref{eq:error_bound_stat}, which concludes the proof.
\end{proof}

\subsection{Proof of Theorem \ref{thm:comp}}  \label{sec:proof:convergence} 

\begin{proof} 
	Our proof is based on the proof of Theorem \ref{thm_conv_sgd_with_err} for analyzing stochastic gradient descent   on  Riemannian manifold. 
	
	To begin with, to utilize Proposition \ref{prop:push_par}, we first characterize the Lipschitz continuity of $\nabla \tilde f^*_k$, which is implied by an upper bound of $\nabla^2 \tilde f^*_k$. 
	Here $\tilde f_k^*$ is obtained  by solving the statistical estimation problem in  \eqref{eq:new_stat}.
	Note that $ \partial^2_{ij} \tilde  f_k^* (x) =  \la\partial^2_{ij} K(x, \cdot ), \tilde  f_k^*(\cdot ) \ra_\cH$.
	 Thus, by Assumption \ref{ass4}, it suffices to upper bound $\| \tilde  f_k^* \|_{\cH}$.  
	By triangle inequality, we have 
	\begin{align}
	\label{eq:lip1}
	\| \tilde f_k^* \|_\cH \le \| \tilde f^*_k - f^*_k \|_\cH + \| f^*_k \|_{\cH},
	\end{align}
	where $f_k^*$ is the solution to the variational problem associated with $F(\tilde p_k)$. 	
	For the first term of \eqref{eq:lip1}, by Theorem \ref{thm:stat}, we have  
	\begin{align}
	\label{eq:lip2}
	\| \tilde f^*_k - f^*_k \|_\cH = \cO\Bigl [  C_{K,1}^{\alpha_\beta  }   \cdot \kappa ^{- \alpha_\beta   } \cdot \| f^*_k \|_{\cH^\beta}^{ 1- \alpha_\beta }  \cdot  ( \log N / N  )^{ \alpha_\beta /2 }    \Bigr ],
	\end{align}
	with probability at least $1-N^{-2}$, where $\alpha_\beta = (\beta - 1) / (\beta +1 )$. For the second term of \eqref{eq:lip1},
	recall that we let $\nu$ denote the Lebesgue measure on $\cX$ and the integral operator $\cT$ introduced in \S\ref{bg:rkhs} has eigenvalues $\{ \mu_i\}_{i \geq 1}$. 
Let $\mu_{\max} = \max_{i\geq 1} \mu_i$. 
	Since $\beta \in (1,2)$, by direct computation, 
	 we have 
	\begin{align}
	\label{eq:lip3}
	\| f^*_k \|_{\cH}^2 = \sum_{i=1}^\infty \mu_i^{-1} \cdot  \la f, \psi_i \ra_\nu^2 \le \mu_{\max}^{\beta - 1} \cdot \sum_{i=1}^\infty \mu_i^{-\beta} \cdot  \la f, \psi_i \ra_\nu^2 = \mu_{\max}^{\beta - 1} \cdot \| f^*_k \|_{\cH^\beta}^2 \leq \mu_{\max}^{\beta - 1} \cdot R^2 ,
	\end{align}
	where the last inequality follows from Assumption  
	\ref{assume:f_star}. 
	Plugging \eqref{eq:lip2} and \eqref{eq:lip3} into \eqref{eq:lip1},  by Assumption  
	\ref{assume:f_star}, with probability at least $1 - N^{-2}$, we obtain that  
	\begin{align}
	\label{eq:lip4}
	\| \tilde f^*_k\|_\cH = \cO\Bigl [  C_{K,1}^{\alpha_\beta  }   \cdot \kappa ^{- \alpha_\beta   } \cdot R^{ 1- \alpha_\beta }  \cdot  ( \log N / N  )^{ \alpha_\beta /2 }    \Bigr ] + \mu_{\max}^{(\beta - 1)/2} \cdot R. 
	\end{align}
 Now we let  $N$ to be sufficiently large such  that 
	\$
	\cO\Bigl [  C_{K,1}^{\alpha_\beta  }   \cdot \kappa ^{- \alpha_\beta   } \cdot R^{ 1- \alpha_\beta }  \cdot  ( \log N / N  )^{ \alpha_\beta /2 }    \Bigr ] \le \mu_{\max}^{(\beta - 1)/2} \cdot R
	\$ 
	and define an event $\cE_k$ as 
	\begin{align}
	\label{eq:event}
	\cE_k = \biggl\{\| \tilde f^*_k - f^*_k\|_\cH \le \cO\Bigl [  C_{K,1}^{\alpha_\beta  }   \cdot \kappa ^{- \alpha_\beta   } \cdot R^{ 1- \alpha_\beta }  \cdot  ( \log N / N  )^{ \alpha_\beta /2 }    \Bigr ] \biggr\}, \qquad \forall k \geq 0. 
	\end{align}
By \eqref{eq:lip4} and \eqref{eq:event}, 
we have  $\| \tilde f^*_k\|_\cH \le 2\mu_{\max}^{(\beta - 1)/2} \cdot R$ conditioning on event $\cE_k$, which holds with 
	  probability at least $1 - N^{-2}$.
	Moreover, under  Assumption \ref{ass4},  we have 
	\begin{align}
	\label{eq:lip5}
	\| \nabla^2\tilde f^*_k(x) \|_{\fro} & = \biggl [  \sum_{i,j \in [d]}   \bigl|  \partial^2_{ij} \tilde f^*_k(x)  \bigr|^2  \biggr ] ^{1/2 }  = \biggl [ \sum_{i,j \in [d]} \bigl | \bigl \la \partial^2 _{i,j} K(x, \cdot ), \tilde f_k^* \bigr \ra _{\cH } \bigr | ^2   \biggr ] ^{1/2} \\
	& \le   \biggl [ \sum_{i,j \in [d]}  \bigl \| \partial^2 _{i,j} K(x, \cdot ) \bigr \|_{\cH} ^2       \biggr ]  ^{1/2} \cdot \| \tilde f_k^* \|_{\cH} \leq   d  \cdot C_{K, 3} \cdot \|\tilde f_k^* \|_{\cH} \le 2d \cdot  C_{K, 3} \cdot \mu_{\max}^{(\beta - 1)/2} \cdot R. \notag 
	\end{align}
	Here $\| \cdot \|_{\fro}$ denotes the matrix  Frobenius norm, the first inequality follows from the Cauchy-Schwarz inequality, and the second inequality holds under  Assumption \ref{ass4}.
	Thus, 
	conditioning on event $\cE_k$, 
	 $\nabla \tilde f^*_k$ is $H_0 $-Lipschitz continuous,  where we define  
	$H_0 = 2d  \cdot C_{K, 3} \cdot \mu_{\max}^{(\beta - 1)/2} \cdot R.$
	
	Furthermore, recall that the total number of iterations is denoted by  $K $, which is smaller than $N$. To simplify the notation, we define event $\cE$ as $\cE = \cap _{k = 0}^K \cE_k$. 
	By union bound, we have $\PP(\cE) \geq 1 - (K+1) \cdot N^{-2} \geq 1 - N^{-1}$. 
	Moreover, conditioning on event $\cE$, $\nabla \tilde f_k^* $ is $H_0$-Lipschitz continuous for all $k \in \{0, \ldots , K\}$.	
 Notice that the stepsize $\alpha $ is set such that $\alpha < 1/ H_0$. 
	By  Proposition \ref{prop:push_par}, for any $k \leq K$, 
	we   can equivalently write \eqref{eq:final_iters} as 
	\#\label{eq:comp11}
	\tilde p_{k+1} = \expm_{\tilde p_k}  \bigl \{  - \alpha _k  \cdot [ \grad F(\tilde p_k) +   \delta_k ] \bigr \} , \qquad \delta_k = -\dvg \bigl  [ \tilde p_k \cdot ( \nabla \tilde f_k ^* - \nabla f_k^*)  \bigr ] .
	\# 
	Note that $\delta_k \in \cT_{\tilde p_k} \cP_2(\cX)$ is a tangent vector at point $\tilde p_k$. 
	Since $F$ is $L$-smooth under Assumption \ref{assume:objective}, by \eqref{eq:F_smooth} and \eqref{eq:comp11}, we have that
	\begin{align} \label{eq:comp12}
	F(\tilde  p_{k+1}) &  \leq   F(\tilde p_k) - \alpha_k \cdot \bigl\la \grad F(\tilde p_k) ,  \grad F (\tilde p_k) \bigr \ra_{\tilde p_k} - \alpha_k \cdot \bigl \la  \grad F(\tilde p_k) ,    \delta_k \bigr\ra_{\tilde p_k} \notag \\
	&\qquad +  \alpha_k^2 \cdot L/ 2 \cdot  \bigl \la  \grad F(\tilde p_k) +   \delta_k,  \grad F(\tilde p_k) +    \delta_k \bigr\ra_{\tilde p_k}  \notag \\
	& = F(\tilde p_k) - ( \alpha_k  - \alpha_k^2 \cdot L /2  ) \cdot \bigl\la \grad F(\tilde p_k) ,  \grad F (\tilde p_k) \bigr \ra_{\tilde p_k} \notag \\
	& \qquad + (   \alpha_k  + \alpha_k^2  \cdot L ) \cdot  \bigl | \bigl \la  \grad F(\tilde p_k) ,    \delta_k \bigr\ra_{\tilde p_k} \bigr |  +   \alpha_k^2  \cdot L /2 \cdot  \bigl \la  \delta _k,  \delta_k \bigr\ra_{\tilde p_k}  , 
	\end{align}
	where $\la \cdot , \cdot \ra_{\tilde p_k}$ is the 
	Riemannian metric on $\cT_{ \tilde p_k} \cP_2 (\cX)$.   
	By the inequality $2 ab \leq a^2 + b^2$,   we have  
	\begin{align}\label{eq:comp13}
	2 \cdot  \bigl | \bigl \la \grad F(\tilde p_k), \delta_k \big\ra _{\tilde p_k} \bigr |  & \leq  \bigl \la  \grad F(\tilde p_k),  \grad F(\tilde p_k) \bigr\ra_{\tilde p _k} + \bigl \la \delta_k, \delta_k\bigr\ra_{\tilde p_k}. 
	\end{align}
	Thus, by \eqref{eq:comp12} and  \eqref{eq:comp13}, we have that
	\#\label{eq:comp14} 
	F( \tilde p_{k+1}) &  \leq   F(\tilde p_k) -\alpha_k \cdot ( 1- 2\alpha_k L ) / 2 \cdot    \bigl\la \grad F(\tilde p_k) ,  \grad F (\tilde p_k) \bigr \ra_{\tilde p_k} \notag \\
	& \qquad   +  \alpha_k\cdot  ( 1+2  \alpha_k L ) /2 \cdot  \bigl \la \delta_k, \delta_k\bigr\ra_{\tilde p_k}.   
	\#
	Notice that under Assumption \ref{assume:objective}, 
	 $F$ is $\mu$-gradient dominated.  
	 To simplify the notation, let $F^*$ denote the optimal value $\inf_{p \in \cP_2(\cX) } F(p)$. 
	By \eqref{eq:F_grad_dom} and \eqref{eq:comp14} we have that
	\#\label{eq:comp15} 
	F(\tilde p_{k+1})  \leq  F(\tilde p_k) - \mu \cdot \alpha_k \cdot ( 1- 2\alpha_k L )/2  \cdot   [ F (\tilde p_k) - F^*    ] +\alpha_k \cdot  ( 1+2  \alpha_k L ) /2 \cdot \bigl\la \delta_k, \delta_k\bigr\ra_{\tilde p_k}.
	\#
	As we have introduced in \eqref{eq:def_varepsilon_k}, 
	$ \la \delta_k, \delta_k \ra_{\tilde p_k}$ is equal to 
	\#\label{eq:define_var_epsilon}
	\varepsilon_k = \int_{\cX} \big \| \nabla  \tilde f_k^* (x)  - \nabla f_k^* (x) \big \|_2^2 \cdot \tilde p_k (x) ~\ud x. 
	\#  
	Thus, combining \eqref{eq:define_var_epsilon}, we can equivalently write \eqref{eq:comp15} as   
	\# \label{eq:comp16}
	F(  \tilde p_{k+1} ) - F^*  &  \leq  [ 1-  \mu \cdot \alpha_k \cdot ( 1- 2\alpha_k L ) /2 ]  \cdot  [ F (\tilde p_k) - F^*    ]    + \alpha_k ( 1+2  \alpha_k L  ) /2 \cdot  \varepsilon_k.
	\#
	which 
	characterizes  the performance of a single step of variational transport in term of the objective value. 
	Furthermore, recall that we set $\alpha_k = \alpha$ for all $k \geq 0$, where $\alpha \cdot 2 L < 1 $. Thus, we have 
	\#\label{eq:comp18}
	1-  \mu \cdot \alpha_k \cdot ( 1- 2\alpha_k L) /2 \leq 1 - \alpha \cdot \mu /2 \in (0,1), \qquad ( 1+2  \alpha_k L ) /2 <  1.
	\#
	In the sequel,  we define $\rho =1 - \alpha \cdot \mu /2 $.
	Combining \eqref{eq:comp16} and  \eqref{eq:comp18}, it holds that 
	\# \label{eq:comp19}
	F(   \tilde p_{k+1} ) -  F^*      \leq  \rho  \cdot [ F (\tilde p_k) - F^*    ] + \alpha \cdot \varepsilon_k.
	\#
	Multiplying $\rho^{-(k+1)}$ to both sides of \eqref{eq:comp19}, 
	we have 
	\#\label{eq:comp20}
	\rho^{-(k+1)} \cdot  \bigl [ F(   \tilde p_{k+1} ) -  F^*     \bigr ]  \leq \rho^{-k} \cdot \ [ F (\tilde p_k) - F^*    ]   + \rho^{-(k+1)} \cdot   \alpha \cdot \varepsilon_k .
	\#
	Thus, $\{ \rho^{-k} \cdot   [ F(   \tilde p_{k} ) -  F^*       ]  \}_{k \geq 0   }^{K+1}$  forms a telescoping sequence. 
	Summing both sides of \eqref{eq:comp20}, for any $k \leq K$,
	we have 
	\#\label{eq:comp21}
	F(   \tilde p_{k} ) -  F^*   \leq \rho^k \cdot   [ F(   \tilde p_{0} ) -  F^*   ] +  \sum_{\ell = 0}^{k-1}  \rho^{k-1-\ell}  \cdot \alpha\cdot  \varepsilon_\ell .
	\#
	Meanwhile, 
	conditioning on  event $\cE$, 
	for each $\ell \leq K$, 
	by Assumption \ref{assume:f_star} and Theorem \ref{thm:stat}  we have  
	\begin{align}
	\label{eq:comp22}
	\varepsilon_\ell & = \int_{\cX} \big \| \nabla  \tilde f_\ell^* (x)  - \nabla f_\ell^* (x) \bigr \|_2^2 \cdot \tilde p_\ell (x) ~\ud x 
	 = \int_{\cX} \bigg[  \sum_{j \in [d] } \big|  \partial _j \tilde f_{\ell}^* (x) - \partial _j  f_{\ell}^* (x) \bigr  |^2 \biggr ]  \cdot \tilde p_{\ell} (x) ~\ud x \notag \\
	 & 
	 = \int_{\cX} \bigg[  \sum_{j \in [d] } \big| \bigl \la  \partial _j K(x, \cdot)  , \tilde f_{\ell}^* -     f_{\ell}^*  \bigr \ra_{\cH } \bigr  |^2 \biggr ]  \cdot \tilde p_{\ell} (x) ~\ud x  \leq \bigg[  \int_{\cX}   \sum_{j \in [d] }  \bigl  \|  \partial _j K(x, \cdot)  \bigr\|_{\cH }   ^2   \cdot \tilde p_{\ell} (x) ~\ud x \bigg]  \cdot \| \tilde f_{\ell}^* -     f_{\ell}^*  \|_{\cH } ^2 \notag \\
	 &
	  = \cO \Bigl [ C_{K,2}^2 \cdot d \cdot C_{K,1}^{2 \alpha_\beta  } \cdot \kappa ^{- 2 \alpha_\beta   }\cdot R^{2-2 \alpha_\beta }   \cdot  ( \log N / N  )^{ \alpha_\beta }   \Bigr ],
	\end{align}
	where the inequality in \eqref{eq:comp22} follows from Cauchy-Schwarz inequality and $\cO(\cdot)$ omits absolute constants.  
	We define the following error term
	\begin{align*}
	{\rm Err} =  \cO \Bigl [ \alpha \cdot  C_{K,2}^2 \cdot d \cdot C_{K,1}^{2 \alpha_\beta  } \cdot \kappa ^{- 2 \alpha_\beta   }\cdot R^{2-2 \alpha_\beta }   \cdot  ( \log N / N  )^{ \alpha_\beta }   \Bigr ].
	\end{align*}
	Thus, by \eqref{eq:comp21} and \eqref{eq:comp22}, we have that
	\begin{align*}
	F(   \tilde p_{k} ) -  \inf_{p \in \cP_2(\cX) } F(p)  \leq \rho^k \cdot \bigl [ F(   \tilde p_{0} ) -  \inf_{p \in \cP_2(\cX) } F(p)    \bigr ] +  (1-\rho)^{-1} \cdot {\rm Err}, \qquad \forall k \in [K], 
	\end{align*}
	with probability at least $ 1 - N^{-1}$, which completes the proof of Theorem \ref{thm:comp}.	
\end{proof}

\section{Proofs of the Auxiliary Results}

In this section, we prove the auxiliary results introduced in \S\ref{sec:algo}.

\subsection{Proof of Proposition \ref{prop:func_grad}} \label{proof:prop:func_grad}

\begin{proof}
	We first recall the definition of   directional derivative. For any $p \in \cP_2(\cX)$ and any $s \in \cT_p \cP_2(\cX)$, let $\gamma \colon [0, 1 ] \rightarrow \RR$ be a curve satisfying $\gamma (0) = p$ and $\gamma'(0) = s$.  
	Then, the directional derivative of $F$ satisfies      
	\#\label{eq:func_grad1}
	\bigl \la \grad F(p), s \bigr \ra _p =   \frac{\ud}{\ud t} F \bigl [  \gamma(t) \bigr ] \bigggiven _{t= 0} . 
	\#
	By the definition of functional gradient with respect to the $\ell_2$-structure, the directional derivative at $p$ in the direction of $s$ can be written as 
	\#\label{eq:func_grad2} 
	\frac{\ud}{\ud t} F \bigl [  \gamma(t) \bigr ] \bigggiven _{t= 0} = \int _{\cX} \frac{\delta F}{\delta p} (x)  \cdot s (x)~ \ud x. 
	\#
	Let $u \colon \cX\rightarrow \RR$ be the unique  solution to  the following elliptic equation,
	\#\label{eq:func_grad3}
	-\dvg \bigl [  p(x)\cdot (\nabla u) (x)  \bigr ]   & = s(x), \qquad \forall x \in \cX, 
	\#  
	where $\dvg$ is the divergence operator on $\cX$. Thus, combining \eqref{eq:func_grad2} and \eqref{eq:func_grad3}, and applying the integration by parts, we have 
	\#\label{eq:func_grad4}
	\frac{\ud}{\ud t} F \bigl [  \gamma(t) \bigr ] \bigggiven _{t= 0} & = - \int _{\cX} \frac{\delta F}{\delta p} (x)  \cdot  \dvg \bigl [  p(x)\cdot (\nabla u) (x)  \bigr ]  \ud x  \\
	& = - \int _{\cX} \biggl \{    \dvg \biggl [  p(x)\cdot \frac{\delta F}{\delta p} (x)  \cdot (\nabla u) (x)  \biggr ]  -   \biggl \la  \nabla \biggl (\frac{\delta F}{\delta p} \biggr )   (x) , (\nabla u) (x) \biggr \ra  \cdot p(x)\biggr\}    \ud x,\notag
	\#
	where the second equation follows from the fact that  $\dvg(f\cdot v)  = \la \nabla f, v \ra + f \cdot \dvg(v)$ holds  for any scalar function $f$ and any vector-valued function   $v$.  
	Here   $\la \cdot, \cdot \ra$ is the inner product structure on $\cX$. 
	Note that $\cX$ is a compact Riemannian manifold without a boundary. 
	By the divergence theorem \citep{rudin1976principles}, 
	it holds that the first term on   the right-hand side of \eqref{eq:func_grad4} is equal to zero. 
	Thus, combining \eqref{eq:func_grad1} and \eqref{eq:func_grad4} we obtain that 
	\#\label{eq:func_grad5}
	\bigl \la \grad F(p), s \bigr \ra _p = \int_{\cX} 
	\biggl \la  \nabla \biggl (\frac{\delta F}{\delta p} \biggr )   (x) , (\nabla u) (x) \biggr \ra  \cdot p(x)  ~  \ud x. 
	\#
	
	Meanwhile, let $v \colon \cX \rightarrow \RR$ be the solution to   the elliptic equation 
	\$
	-\dvg \bigl [  p(x)\cdot (\nabla v) (x)  \bigr ]   & =   [\grad F(p)] (x),  \qquad \forall x \in \cX. 
	\$
	Since both $\grad F(p)$ and $s$ belong  to the tangent space $\cT_p \cP_2(\cX)$, their 
	inner product 
	is given by the 
	Riemannian metric on $(\cP_2(\cX), W_2)$, i.e., 
	\#\label{eq:func_grad6}
	\bigl \la \grad F(p), s \bigr \ra _p  = \int_{\cX}   \bigl \la \nabla u(x) , \nabla v (x) \bigr \ra \cdot p(x) ~ \ud x. 
	\# 
	Combining \eqref{eq:func_grad5} and \eqref{eq:func_grad6}, we obtain for all  $ s \in \cT_p \cP_2(\cX)$ that 
	\#\label{eq:func_grad7}
	\int_{\cX} 
	\biggl \la  \nabla \biggl (\frac{\delta F}{\delta p} \biggr )   (x) , (\nabla u) (x) \biggr \ra  \cdot p(x)  ~  \ud x 	= \bigl \la \grad F(p), s \bigr \ra _p   = \int_{\cX}   \Bigl \la( \nabla u)(x) , (\nabla v) (x) \Bigr \ra \cdot p(x) ~ \ud x.
	\#
	Since $s \in \cT_p\cP_2(\cX)$ is arbitrary, \eqref{eq:func_grad7} implies that 
	$
	\nabla (\delta F / \delta p ) = \nabla v.
	$
	Therefore, we have 
	\$
	\grad F(p) = -\dvg\bpa{p \nabla v} & = - \dvg  \bigl [ p \cdot \nabla (\delta F/ \delta p) \bigr ],
	\$
	which establishes \eqref{eq:riem_grad}. 
	
	It remains to obtain $\delta F/ \delta p$ for $F$ defined in \eqref{eq:var_func}. 
	By the definition of the functional derivative with respect to the   $\ell_2$-structure, for any square-integrable function $\phi \colon \cX \rightarrow \RR$, we have 
	\#
	\label{eq:func_differential}
	\int_{\cX} \frac{\delta F } {\delta p} (x) \cdot \phi(x) ~\ud x = \lim_{\epsilon \rightarrow 0} \frac{1}{\epsilon} \cdot \bigl [ F(p + \epsilon \cdot \phi) - F(p) \bigr ] .
	\#
	For the notational simplicity, we let  $f_{\epsilon} ^* $  denote  the  optimal dual variable of  the  optimization problem
	\#\label{eq:dual_problem22}
	\sup_{f \in \mathcal{F}}  \biggl \{ \int _{\cX} f(x) \cdot  [p(x) + \epsilon \cdot \phi (x) ]   ~\ud x - F^*(f) \biggr\} 
	\# 
	for any sufficiently small $\epsilon > 0$. Then, it holds that 
	\#\label{eq:compute_grad1}
	&  \bigl [ F(p + \epsilon \cdot \phi) - F(p) \bigr ] \notag \\
	& \qquad  \geq  \biggl [  \int_{\cX} f^*_p (x) \cdot  [p(x) + \epsilon \cdot \phi (x) ]   ~\ud x - F^*(f_p^* )  \biggr ] - \biggl [  \int_{\cX} f^*_p (x) \cdot   p(x)    ~\ud x - F^*(f_p^* )  \biggr ] \notag \\
	& \qquad = \epsilon \cdot \int_{\cX } f_p^* (x) \cdot \phi(x) ~\ud x, 
	\#
	where the inequality follows from the suboptimality  of $f_p^*$ for \eqref{eq:dual_problem22}. 
	Similarly, we obtain an upper bound of $ [ F(p + \epsilon \cdot \phi) - F(p)  ] $ by
	\#\label{eq:compute_grad2}
	&  \bigl [ F(p + \epsilon \cdot \phi) - F(p) \bigr ] \notag \\
	& \qquad  \leq  \biggl [  \int_{\cX} f^*_\epsilon (x) \cdot  [p(x) + \epsilon \cdot \phi (x) ]   ~\ud x - F^*(f_\epsilon^* )  \biggr ] - \biggl [  \int_{\cX} f^*_\epsilon (x) \cdot   p(x)    ~\ud x - F^*(f_\epsilon^* )  \biggr ] \notag \\
	& \qquad  =  \epsilon \cdot \int_{\cX } f_\epsilon^* (x) \cdot \phi(x) ~\ud x.
	\# 
	Combining \eqref{eq:compute_grad1} and \eqref{eq:compute_grad2}, we obtain that $F(p + \epsilon \cdot \phi) - F(p) $ converges to zero as $\epsilon$ tends to zero.  Since  $F^*$ is strongly convex, there exists a constant $\lambda > 0$ such that,  for any two  measurable functions $f_1$ and $f_2$, we have 
	\$
	\int_{\cX} \biggl [  \frac{\delta F^*} {\delta f_1}(x) - \frac{\delta  F^*} {\delta f_2}(x) \biggr ]  \cdot [f_1(x) - f_2(x) ] ~ \ud x \geq  \lambda \cdot  \int_{\cX} [f_1(x) - f_2(x) ] ^2 ~\ud x.
	\$
	Besides, since $f_p^*$ and $f_{\epsilon}^*$ are maximizers of the  optimization problems in \eqref{eq:var_func} and \eqref{eq:dual_problem22}, respectively, we have 
	$
	\delta F^*/ \delta f_p^*  = p $ and 
	$  \delta F^*/ \delta f_\epsilon^*   =p + \epsilon \cdot \phi.
	$
	Thus, combining the strong convexity of $F^*$ and the Cauchy-Schwarz inequality, we obtain 
	\$
	\lambda \cdot \int _{\cX} (f_{\epsilon}^* - f_p^*)^2 \ud x  \leq \epsilon \cdot \int_{\cX} \phi \cdot (f_{\epsilon}^* - f_p^*) ~\ud x \leq \epsilon \cdot \biggl ( \int_{\cX}  \phi^2 ~\ud x  \biggr )  ^{1/2} \cdot \biggl [ \int_{\cX}   (f_{\epsilon}^* - f_p^*)^2 ~\ud x  \biggr ]^{1/2},
	\$
	which implies that $\| f_{\epsilon}^* - f_p^* \|_{\ell_2}  \leq \epsilon/\lambda \cdot \| \phi \|_{\ell_2}$ and thus $f_{\epsilon}^*$ converges to $f_p^*$ as $\epsilon$ goes to zero.  Finally, combining \eqref{eq:func_differential},  \eqref{eq:compute_grad1}, and \eqref{eq:compute_grad2}, we have for all $\phi$ that 
	\$
		\int_{\cX} \frac{\delta F } {\delta p} (x) \cdot \phi(x) ~\ud x = \lim_{\epsilon\rightarrow 0}   \frac{1}{\epsilon} \cdot \bigl [ F(p + \epsilon \cdot \phi) - F(p) \bigr ] = \int_{\cX} f_p^* (x) \cdot \phi(x) ~ \ud x
	\$
	which implies that $\delta F / \delta p = f_p^*$.
	
	Finally, by \eqref{eq:riem_grad}, for $F$ defined in \eqref{eq:var_func}, we have $\grad F(p) = - \dvg ( p \cdot \nabla f_p^*)$. Therefore, we complete the proof of this proposition.
\end{proof}

\subsection{Proof of Proposition \ref{prop:push_par}} \label{proof:prop:push_par}

\begin{proof}
	We define a curve $\gamma\colon [0, 1/ H) \rightarrow \cP_2(\cX)$ by letting $ \gamma(t) =[ \expm_{\cX} (t \cdot \nabla u) ]_{\sharp} p $  for all $t\in [0, 1/H)$.  By the definition of exponential mapping, to establish \eqref{eq:push_eq} of Proposition \ref{prop:push_par}, it suffices to show (i)  $\gamma(0) = p$, (ii) $\gamma'(0) = s$, and (iii) $\gamma$ is a geodesic on $\cP_2(\cX)$ for $t \in [0, 1/H)$. 
	First note  that $\gamma (0) = \id  _{\sharp} p = p$, where $\id  \colon \cX \rightarrow \cX$ is the  identity mapping on $\cX$.   
	In addition, the   following lemma  proves that $\gamma(t)$ is a geodesic on $\cP_2(\cX)$. 
	
	\begin{lemma}
		\label{lem_push_particle_gives_geodesic}
		Let $\cX$ be $\RR^d$ or a  closed convex  subset of $\RR^d$ with a periodic boundary condition. 
		Let $u: \cX\rightarrow \cX $ be a  twice continuously differentiable function on $\cX$  with gradient $\nabla u \colon \cX \rightarrow \cX$ being  $H$-Lipschitz continuous. Then,  for any $p \in \cP_2(\cX)$,  a curve $\gamma \colon [ 0, 1/ H ) \rightarrow \cP_2(\cX)$ defined by 
		$ 
		\gamma(t) =  [ \expm_{\cX} (t \cdot  \nabla u)  ] _{\sharp}p  
		$ 
		is a geodesic on $\cP_2(\cX)$. 
	\end{lemma}
	
	\begin{proof}
		We consider that $\cX $ is $ \RR^d$ or a subset of $\RR^d $ with periodic boundary condition separately.  We first consider the former case, where $\gamma(t)$ can be written as 	\#\label{eq:exp_map_form}
		[ \expm_{\cX} (t \cdot \nabla u ) ]_{\sharp} p = (\id + t \cdot \nabla u )_{\sharp} p, \qquad \forall t \in [0, 1/ H ). 
		\# 
 	To show that $\gamma$ is a geodesic, we  first show that $\gamma $   is indeed a curve in $\cP_2(\cX)$, i.e., $\gamma(t) \in \cP_2(\cX)$ for all $   t \in [0, 1/H)$. To this end, we define $\varphi_t(x) = \| x \|_2^2 /2 + t \cdot u(x)$ for all $t\in [0, 1/ H )$. By definition, we have $   \nabla  \varphi_t = \id + t \cdot \nabla u$, which implies that  $ \gamma (t) = [ \nabla \varphi_t ] _{\sharp} p$.  Since $\nabla u $ is $H$-Lipschitz continuous, for all $t\in [0, 1/ H)$, $ \varphi_t$ is strongly convex. 
		In addition, since $\varphi_t$ is also   twice continuously differentiable, the Jacobian of $\nabla \varphi_t$, i.e., $\nabla^2 \varphi_t$ is continuous and  positive definite, which implies that
		$\nabla \varphi_t  $ is an invertible mapping from $\RR^d$ to $\RR^d$.  Hence,   $[ \nabla \varphi_t]_{\sharp} p$ is absolutely continuous with respect to the Lebesgue measure and has a finite second moment. This shows that $\gamma(t) \in \cP_2(\cX) $  for all $t\in [0, 1/ H) $.  
		
		It remains to show that $\gamma$ is a geodesic. Note that we have $\gamma (t) = [\nabla \varphi_t ] _{\sharp} p $ for all $t \in [0,1/ H)$. The Brenier's Theorem (see, for example, Theorem 2.12 in \cite{villani2003topics}) shows that, for any fixed $t \in (0, 1 / H)$, there exists a unique  optimal transport plan between $p$ and $\gamma(t)$ that is induced by  the gradient of a convex function $\tilde \varphi_t $.  Specifically,  it holds that $ \gamma (t) = [ \nabla \tilde \varphi_t ] _{\sharp} p$. 
		Due to the uniqueness, we conclude that  $\nabla \varphi_t$ is the optimal transportation plan between $p$ and $\gamma(t)$. 
		Finally, 
		we fix any $\overline{t} \in [0, 1/H)$ and show that $\gamma$ is a geodesic when restricted to $ [0, \overline{t} ] $. For any $t\in [0, \overline t ]$, we have 
		\$
		\nabla \varphi_{t  }  = \id +   t  \cdot  \nabla u = [ 1 -  (t / \overline{t} ) ] \cdot \id + (t / \overline{t} )  \cdot \nabla  \varphi_{\overline{t}} ,
		\$ 
		where $ \varphi_{\overline{t}} $ is a strongly convex function.
		Thus, we can write $\gamma (t) $ as $\{ [ 1 -  (t / \overline{t} ) ] \cdot \id + (t / \overline{t} )  \cdot \nabla  \varphi_{\overline{t}} \} _{\sharp} p$.
		By direct computation, for any $0 \leq t_1 < t_2 \leq \overline {t}$,  we have 
		\#\label{eq:w2_geodesic}
		& W_2  [ \gamma(t_1) , \gamma(t_2) ]    \notag \\
		& \qquad = \biggl [ \int _{\cX}   \Bigl \|   \bigl \{  [ 1 -  (t _1 / \overline{t} )  ] x  + (t_1  / \overline{t} ) \cdot \nabla \varphi_{\overline{t}} (x) \bigl \}   - \bigl \{  [ 1 -  (t _2/ \overline{t} )  ] x  + (t_2  / \overline{t} ) \cdot \nabla \varphi_{\overline{t}} (x)  \bigr \} \Bigr \| _2 ^2 \cdot    p(  x) ~\ud x \biggr ]^{1/2}  \notag \\
		& \qquad  = (t_2 -t_1) / \overline{t}\cdot \biggl [ \int _{\RR^d}   \bigl \|     x  - \nabla \varphi_{\overline{t}} (x)   \bigr \| _2 ^2   \cdot   p(  x) ~\ud x  \biggr ]^{1/2}  = (t_2 -t_1) / \overline{t}  \cdot W_2\bigl[ p , \gamma(\overline {t} ) \bigr] .
		\#
		Thus, $\{ \gamma(t)\}_{t \in [0, \overline{t}]}$ is a reparametrized geodesic.  Since $\overline{t}$ is arbitrarily chosen within $[0, 1/H)$, it then follows that $\gamma(t)$ is a geodesic for $0 \leq t < 1/H$.

		It remains to prove this lemma for the case where $\cX$ is a closed convex subset of $\RR^d$ with a  periodic boundary condition.  In this case, any $x \in \cX$ can be identified with an equivalence class  of $\RR^d$. Each probability measure $p\in \cP(\cX)$  is unique identified with a periodic measure $\tilde p \in \cP(\RR^d)$ such that $\tilde p$ coincides with $p$ on $\cX$. 
		We  call    $\tilde p$  the periodic extension of $p$. 
		Since $p$ is absolutely continuous with respect to the Lebesgue measure and has a positive density, so is $\tilde p$.  
		Moreover, $u \colon \cX \rightarrow \RR$ can also be extended as a periodic function on $\RR^d$, and $\varphi_t (x) =  \| x \|_2^2  /2 + t\cdot u(x)  $ is a strongly convex, twice continuously differentiable, and  periodic function on $\RR^d$. 
		Thus, $(\id + t\cdot \nabla u)_{\sharp} \tilde p $ is  the periodic extension of $[ \expm_{\cX} (t \cdot \nabla u ) ]_{\sharp} p$  \citep{carlen2003constrained}. These two measures coincide on $\cX$, i.e., 
		\#\label{eq:periodic_extension}
		[ \expm_{\cX} (t \cdot \nabla u ) ]_{\sharp} p  = (\id + t\cdot \nabla u)_{\sharp} \tilde p\big\vert_{\cX}, 
		\#
		where $\cdot \vert_{\cX}$ denotes the restriction to $\cX$.
		Note that   $[\nabla \varphi_t]_{\sharp} \tilde p $ is absolutely continuous with respect to the Lebesgue measure and has a  finite second-order moment. Thus, restricting $[\nabla \varphi_t]_{\sharp} \tilde p $ to $\cX$, \eqref{eq:periodic_extension} implies that  $\gamma (t) \in  \cP_2(\cX)$ for all $t\in [0, 1/ H)$, i.e., $\gamma$ is a curve on $\cP_2(\cX)$.

		To show that $\gamma$ is a geodesic,  we utilize the generalization of Brenier's theorem to probability distributions over a Riemannian manifold \citep{mccann2001polar,gigli2011inverse}. For any $t \in [0, 1/ H)$, since $\gamma(t) \in \cP_2(\cX)$, there exists a unique optimal transportation plan  $\Upsilon\colon \cX \rightarrow \cX$ between $p$ and $\gamma(t)$  such that  $\gamma(t) = \Upsilon_{\sharp} p$. Specifically,    $\Upsilon$  takes the form  of $\Upsilon(x) = \expm_{x} [ -\nabla \psi(x) ] $ for some $\psi \colon \cX \rightarrow \RR$ such that $\| x \|_2^2 /2 - \psi(x)$ is convex. Thus, due to the uniqueness and the fact that $\nabla u$ is $H$-Lipschitz, $\expm_{\cX} ( - t\cdot \nabla u )$ is the optimal transportation plan between $p$ and $\gamma(t)$. 
		
		Now we fix any $\overline t \in [0, 1/H)$ and show that $\gamma$ is a geodesic for $t\in [0, \overline t]$. For any $0\leq t_1 < t_2 \leq \overline t$,   following the derivations in \eqref{eq:w2_geodesic}  and  \eqref{eq:periodic_extension},  we have 
		\$ 
		& W_2  [ \gamma(t_1) , \gamma(t_2) ]   = W_2 \bigl \{ [ \expm_{\cX} (t_1  \cdot \nabla u ) ]_{\sharp} p   , [ \expm_{\cX} (t_2  \cdot \nabla u ) ]_{\sharp} p    \bigr \}    \notag \\
		&  \qquad   =   W_2 \Bigl \{(\id + t_1 \cdot \nabla u)_{\sharp} \tilde p\big\vert_{\cX}  , (\id + t_2 \cdot \nabla u)_{\sharp} \tilde p\big\vert_{\cX}   \Bigr \} = (t_2 - t_1) / \overline t \cdot  W_2 \Bigl \{  \tilde p\big\vert_{\cX}  , (\id + \overline t \cdot \nabla u)_{\sharp} \tilde p\big\vert_{\cX}   \Bigr \}  \notag \\
		& \qquad  = (t_2 - t_1) / \overline t \cdot  W_2 [ p, \gamma(\overline t) ], 
		\$
		where the second and the last equality follows from \eqref{eq:periodic_extension}. Thus, we obtain that $\{ \gamma (t) \}_{t\in [0, \overline t]}$ is a geodesic up to reparametrization, which concludes the proof of Lemma \ref{lem_push_particle_gives_geodesic}.
	\end{proof}

	To conclude the proof of Proposition \ref{prop:push_par}, it remains to show that $\gamma'(0) = s$. 
	Similar to the proof of the above lemma, we  distinguish   the two cases  where  $\cX$ is $\RR^d$ and $\cX $ is a closed convex subset of $\RR^d$ with a periodic boundary condition.  
	
	For the former case,   to simplify the notation, we write $T_{t} = \id + t  \cdot \nabla u $, which is invertible for  $t \in [0, 1/ H)$. 
	By the definition of the pushforward mapping, we have 
	\$
	\gamma(t) (x) & =  \bigl [ (T_t) _{\sharp} p  \bigr ] (x) = p \bigl [ T_t^{-1} (x) \bigr ]  \cdot \biggl  | \frac{\ud }{\ud x} \bigl [ T_t^{-1} ( x) \bigr ]  \biggr  |,
	\$
	where the second equality follows from the change-of-variable formula and $ | \frac{\ud }{\ud x}  [ T_t^{-1} ( x)  ]   |$ is the determinant of the Jacobian.  
	When $t$ is sufficiently small, for any $x \in \cX $, the Taylor  expansion in $t$ yields that 
	$
	T_t^{-1} (x) = x - t \cdot ( \nabla u)  (x) + o(t) ,
	$
	which further implies that  
	\# \label{eq:differentiate} 
	\gamma(t) (x) & = p\bigl [ x - t \cdot ( \nabla u)  (x) + o(t)  \bigr ]  \cdot  \bigl | I_d - t \cdot  (\nabla^2  u) (x)  + o(t) \bigr |  \notag \\
	& =  \big[ p(x) - t \cdot \bigl \la 
	( \nabla u)  (x), ( \nabla p)(x)   \bigr \ra + o(t) \bigr ] \cdot  \bigl \{ 1  - t \cdot  \trace [ (\nabla^2  u) (x) ]   + o(t) \bigr \} \notag \\
	& =  p(x) - t \cdot \bigl \la 
	( \nabla u)  (x), ( \nabla p)(x)   \bigr \ra - t\cdot   \trace [ (\nabla^2  u) (x) ]  \cdot p(x) + o(t),
	\#
	where the second equality follows from the Taylor expansion of $p [ x - t \cdot ( \nabla u)  (x) + o(t)  ]$. 
	Since 
	\$ \bigl  \la 
	 \nabla u  (x), \nabla p(x)    \bigr \ra + \trace \bigl[ \nabla^2  u (x) \bigr] \cdot p(x)  = \dvg \bigl [p(x) \cdot  \nabla u (x)  \bigr ],
	\$  by \eqref{eq:differentiate} we obtain that 
	\$
	\gamma(t) (x)  = p(x) - t \cdot  \dvg \bigl [p(x) \cdot  \nabla u (x)  \bigr ] + o(t) , 
	\$
	which implies that $\gamma '(0) =  - \dvg (p \cdot \nabla u)   = s$. 
	
	It remains to show  $\gamma '(0) =  - \dvg (p \cdot \nabla u)   = s$ when $\cX$ is a closed compact subset of $\RR^d$ with periodic boundary condition.  As shown in the proof of Lemma \ref{lem_push_particle_gives_geodesic}, $p$ can be periodically extended to a measure $\tilde p$  on $\RR^d$ and that such an extension is unique. 
	The solution $u$ to  the elliptic equation $- \dvg(   p \cdot \nabla u) =s$ can also also be viewed as a periodic function on $\RR^d$. 
	Thus, $(\id + t\cdot \nabla u)_{\sharp} \tilde p $ is  the periodic extension of $[ \expm_{\cX} (t \cdot \nabla u ) ]_{\sharp} p$  and  \eqref{eq:periodic_extension} holds. 
	Note that  
	$\tilde \gamma (t) = (\id + t\cdot \nabla u)_{\sharp} \tilde p$ satisfies that $\tilde \gamma(0) = \tilde p$ and $\tilde \gamma' (0) = s$. 
	Therefore,  restricting $\tilde \gamma(t)$  to $\cX$, we conclude that $\gamma'(t) = s$, which completes the proof. 
\end{proof}

\end{document}